\documentclass[preprint, 11pt, authoyear]{article}
\usepackage{jmlr2e}
\usepackage{environ} 

\usepackage{afterpage}
\usepackage{mathrsfs}
\usepackage{bm}
\usepackage{mathtools}
\usepackage{verbatim}
\usepackage{hhline}

\usepackage[table,xcdraw,dvipsnames]{xcolor}

\usepackage{refcount}

\newcommand{\myfootnotetext}[1]{\footnotetext{#1\label{fn:text}%
        \edef\fnmark{\getpagerefnumber{fn:mark}}%
        \edef\fntext{\getpagerefnumber{fn:text}}%
        \ifx\fnmark\fntext\else\ClassWarning{}{footnote mark and text on different pages!}\fi}}

\usepackage{comment}
\includecomment{AdditionalTxt}

\usepackage{appendix}
\usepackage{hyperref}
\usepackage{amsmath}
\usepackage{amssymb}
\usepackage{graphicx}
\usepackage{amsfonts}
\usepackage{latexsym}
\usepackage{pdfsync}
\usepackage{subcaption}
\usepackage{color}
\usepackage{tabularx}
\usepackage[english]{babel}
\usepackage{array}
\usepackage{algorithm}
\usepackage{algorithmic}
\usepackage{esint}

\usepackage[bottom]{footmisc}

\usepackage{enumerate}

\usepackage{bbm}% For Images
\usepackage{color, graphicx}              % to include figures
\usepackage{caption}
\usepackage{epstopdf}
\usepackage{tabularx}

\newtheorem{assumption}[theorem]{Assumption}

\numberwithin{equation}{section}
\usepackage{hyperref}
\usepackage{multirow}

\usepackage{tikz}
\DeclareFontFamily{OT1}{pzc}{}
\DeclareFontShape{OT1}{pzc}{m}{it}{<-> s * [1.10] pzcmi7t}{}
\DeclareMathAlphabet{\mathpzc}{OT1}{pzc}{m}{it}

\numberwithin{equation}{section}

\newcommand{\rhoT}{\rho_T}

\newcommand{\brhoL}{\bm{\rho}_T^L}
\newcommand{\brhoT}{\bm{\rho}_T}
\newcommand{\brho}{\bm{\rho}}
\newcommand{\brhoEAL}{\bm{\rho}_T^{EA,L}}
\newcommand{\brhoxiL}{\bm{\rho}_T^{\xi, L}}
\newcommand{\brhoEA}{\bm{\rho}_T^{EA}}
\newcommand{\brhoxi}{\bm{\rho}_T^{\xi}}
\newcommand{\bmu}{\bm{\mu^{\bY}}}

\newcommand{\forcex}{\mathbf{F}^{\bx}}

\newcommand{\vertiii}[1]{{\left\vert\kern-0.25ex\left\vert\kern-0.25ex #1 
    \right\vert\kern-0.25ex\right\vert\kern-0.25ex}}

\newcommand{\mbf}[1]{\boldsymbol{#1}}

\newcommand{\dbinnerp}[1]{\langle\hspace{-1mm}\langle{#1}\rangle\hspace{-1mm}\rangle}

\newcommand{\norm}[1]{\left\| #1 \right\|}

\newcommand{\real}{\mathbb{R}}

\newcommand{\br}{\mbf{r}}

\newcommand{\bv}{\mbf{v}}

\newcommand{\bx}{\mbf{x}}
\newcommand{\bX}{\mbf{X}}

\newcommand{\bXdot}{\dot{\mbf{X}}}

\renewcommand{\top}{T}

\newcommand{\topwd}{\mbf{s}}
\newcommand{\topwdT}{\mathcal{V}} %The total features.
\newcommand{\topwde}{\topwd^E}
\newcommand{\topwdem}{\topwd^{E, (m)}}
\newcommand{\topwda}{\topwd^A}
\newcommand{\topwdam}{\topwd^{A, (m)}}
\newcommand{\topwdxi}{\topwd^{\xi}}
\newcommand{\topwdxim}{\topwd^{\xi, (m)}}

\newcommand{\FeatureHypSpaceE}{\mathbf{S}^E}
\newcommand{\FeatureHypSpaceA}{\mathbf{S}^A}
\newcommand{\FeatureHypSpacexi}{\mathbf{S}^{\xi}}
\newcommand{\PairHypSpace}{\mathbf{R}}
\newcommand{\bY}{\mbf{Y}} % new
\newcommand{\by}{\mbf{y}}
\newcommand{\bz}{\mbf{z}}

\newcommand{\mE}{\mathcal{E}}
\newcommand{\bmE}{\mbf{\mathcal{E}}}

\newcommand{\mK}{\mathcal{K}}

\newcommand{\mS}{\mathcal{S}}

\newcommand{\mY}{\mathcal{Y}}
\newcommand{\bigO}{\mathcal{O}}
\newcommand{\R}{\real}

\renewcommand{\dim}{d}
\newcommand{\numcl}{K}
\newcommand{\idxcl}{k}

\newcommand{\cl}{C}
\newcommand{\clof}{\mathpzc{k}}

\newcommand{\bL}{\mbf{L}}
\newcommand{\bZ}{\mbf{Z}}
\newcommand{\bXi}{\mbf{\Xi}}

\newcommand{\rhsf}{\mathbf{f}}

\newcommand{\rhsfvnc}{\rhsf^{\text{nc}, \dot\bx}}
\newcommand{\rhsfxinc}{\rhsf^{\text{nc}, \xi}}

\newcommand{\intkernele}{\intkernel^{E}}
\newcommand{\intkernela}{\intkernel^{A}}
\newcommand{\intkernelxi}{\intkernel^{\xi}}
\newcommand{\bintkernela}{\bintkernel^{A}}
\newcommand{\bintkernelxi}{\bintkernel^{\xi}}
\newcommand{\bintkernele}{\bintkernel^{E}}
\newcommand{\hypspacee}{\mathcal{H}^{E}}
\newcommand{\hypspacea}{\mathcal{H}^{A}}
\newcommand{\hypspacexi}{\mathcal{H}^{\xi}}

\newcommand{\intkernelvare}{\varphi^{E}}
\newcommand{\intkernelvara}{\varphi^{A}}
\newcommand{\intkernelvarxi}{\varphi^{\xi}}

\newcommand{\bintkernelvarxi}{\bintkernelvar^{\xi}}
\newcommand{\bV}{\mbf{V}}

\newcommand{\lintkernele}{\lintkernel^{E}}
\newcommand{\lintkernela}{\lintkernel^{A}}
\newcommand{\lintkernelxi}{\lintkernel^{\xi}}

\newcommand{\blintkernelxi}{\blintkernel^{\xi}}

\newcommand{\bvpEA}{\bm{\varphi}^{EA}}
\newcommand{\bvpE}{\bm{\varphi}^{E}}
\newcommand{\bvpA}{\bm{\varphi}^{A}}
\newcommand{\bvpxi}{\bm{\varphi}^{\xi}}
\newcommand{\bvxi}{\bm{\varphi}^{\xi}}
\newcommand{\bvphEA}{\widehat{\bm{\varphi}}^{EA}}

\newcommand{\bvphxi}{\widehat{\bm{\varphi}}^{\xi}}
\newcommand{\bpEA}{\bm{\phi}^{EA}}
\newcommand{\bpE}{\bm{\phi}^{E}}
\newcommand{\bpA}{\bm{\phi}^{A}}
\newcommand{\bpxi}{\bm{\phi}^{\xi}}
\newcommand{\bphEA}{\widehat{\bm{\phi}}^{EA}}

\newcommand{\bphxi}{\widehat{\bm{\phi}}^{\xi}}

\newcommand{\coords}{\bX,\bV,\bXi}
\newcommand{\estcoords}{\widehat{\bX},\widehat{\bV},\widehat{\bXi}}
\newcommand{\coordsT}{\bX(t),\bV(t),\bXi(t)}

\newcommand{\coordsTL}{\bX(t_l),\bV(t_l),\bXi(t_l)}
\newcommand{\coordsTLM}{\bX^{(m)}(t_l),\bV^{(m)}(t_l),\bXi^{(m)}(t_l)}

\newcommand{\sx}{\topwde}
\newcommand{\sxdot}{\topwda}

\newcommand{\phijhat}{\widehat{\phi}_{jj'}^{EA}}

\newcommand{\sxi}{\topwde_{ii'}}
\newcommand{\sxidot}{\topwda_{ii'}}   %Used to be \bm 

\newcommand{\spacevariables}{\bm{\mathcal{V}}}

\newcommand{\Ltwo}{L^2}
\newcommand{\LtwoB}{\bm{L}^2}  %L2 bold

\makeatletter
\NewEnviron{widerequation}{%
  \begin{equation*}
  \sbox\z@{\let\label\@gobble$\displaystyle\BODY$}
  \makebox[\textwidth]{%
    \begin{minipage}{\dimexpr\wd\z@+3em}
    \vspace{-\baselineskip}
    \begin{equation}
    \BODY
    \end{equation}
    \end{minipage}%
  }
  \end{equation*}
}

\newcommand{\wildcard}{{\{E,A,\xi\}}}

\newcommand{\combf}[2]{#1 \oplus #2}
\newcommand{\combfbvphEbvphA}{\bvphEA}
\newcommand{\combfbphEbphA}{\bphEA}
\newcommand{\combfbpEbpA}{\bpEA}
\newcommand{\combfbvpEbvpA}{\bvpEA}

\newcommand{\basise}{\psi^{\bx}}

\newcommand{\basisxi}{\psi^{\xi}}

\newcommand{\force}{F}
\newcommand{\forcev}{\force^{\dot\bx}}
\newcommand{\forcexi}{\force^{\xi}}
\newcommand{\intkernel}{\phi}
\newcommand{\lintkernel}{\widehat{\intkernel}}
\newcommand{\bintkernel}{{\bm{\phi}}}
\newcommand{\blintkernel}{{\widehat{\bm{\phi}}}}
\newcommand{\intkernelvar}{\varphi}
\newcommand{\bintkernelvar}{{\bm{\varphi}}}
\newcommand{\rhsfo}{\mathbf{f}}
\newcommand{\hypspace}{\mathcal{H}}
\newcommand{\bhypspace}{\mbf{\mathcal{H}}}

\newcommand{\bhypspaceEA}{\mbf{\mathcal{H}}^{EA}}
\newcommand{\bhypspacexi}{\mbf{\mathcal{H}}^{\xi}}

\newcommand{\E}{\mathbb{E}}

\newcommand{\argmin}[1]{\underset{#1}{\operatorname{arg}\operatorname{min}}\;}

\newcommand{\mand}{\quad \text{and} \quad}

\DeclareMathAlphabet{\mathpzc}{OT1}{pzc}{m}{it}

\jmlrheading{}{}{}{}{}{Learning governing laws in second-order agent systems}
\ShortHeadings{Learning governing laws in second-order agent systems}{Miller, Tang, Zhong, Maggioni}
\firstpageno{1}

\begin{document}

\title{Learning Theory for Inferring Interaction Kernels in Second-Order Interacting Agent Systems} %from data

\author{\name Jason Miller$^{2,3}$ \email jason.miller@jhu.edu \\ \AND \name Sui Tang$^{1,4}$ \email suitang@math.ucsb.edu \\
\AND \name Ming Zhong$^{2,3}$  \email mzhong5@jhu.edu \\
\AND 
\name Mauro Maggioni$^{1,2,3}$ \email mauromaggionijhu@icloud.com	  \\ 
\\
       \addr $^1$Departments of Mathematics, $^2$Applied Mathematics and Statistics, \\
       $^3$ Johns Hopkins University, 3400 N. Charles Street, Baltimore, MD 21218, USA\\
 $^4$ University of California, Santa Barbara, 552 University Rd, Isla Vista, CA 93117, USA
}

\maketitle

\begin{abstract}%   <- trailing '%' for backward compatibility of .sty file
Modeling the complex interactions of systems of particles or agents is a fundamental scientific and mathematical problem that is studied in diverse fields, ranging from physics and biology, to economics and machine learning. In this work, we describe a very general second-order, heterogeneous, multivariable, interacting agent model, with an environment, that encompasses a wide variety of known systems. We describe an inference framework that uses nonparametric regression and approximation theory based techniques to efficiently derive estimators of the interaction kernels which drive these dynamical systems. We develop a complete learning theory which establishes strong consistency and optimal nonparametric min-max rates of convergence for the estimators, as well as provably accurate predicted trajectories. The estimators exploit the structure of the equations in order to overcome the curse of dimensionality and we describe a fundamental coercivity condition on the inverse problem which ensures that the kernels can be learned and relates to the minimal singular value of the learning matrix. The numerical algorithm presented to build the estimators is parallelizable, performs well on high-dimensional problems, and is demonstrated on complex dynamical systems. 
\end{abstract}
	
\ \\

\tableofcontents

\ \\

\section{Introduction} \label{sec:intro}
Physical, biological, and social systems across all scales of complexity and size can often be described as dynamical systems written in terms of interacting agents (e.g. particles, cells, humans, planets, ...). Rich theories have been developed  to explain the collective behavior of these interacting agents across many fields including astronomy, particle physics, economics, social science, and biology. Examples include predator-prey systems, molecular dynamics, coupled harmonic oscillators, flocking birds or milling fish, human social interactions, and celestial mechanics, to name a few. In order to encompass many of these examples, we will describe a very general second-order, heterogeneous (the agents can be of different types), interacting (the acceleration of an agent is a function of properties of the other agents) agent system that includes external forces, masses of the agents, multivariable interaction kernels, and an additional environment variable that is a dynamical property of the agent (for example, a firefly having its luminescence varying in time). We propose a learning approach that combines machine learning and dynamical systems in order to provide highly accurate dynamical models of the observation data from these systems. 

The model and learning framework presented in sections \ref{sec:ModelDesc}-\ref{sec:LT} includes a very large number of relevant systems and allows for their modeling. 
Clustering of opinions \cite{Krause2000,CKFL2005, BHT2009, MT2014} is a simple first-order case that exhibits clustering. Flocking of birds \cite{CS2007,CM2008,CD2011} provides a simple example of a second-order system that exhibits an emergent shared velocity of all agents. Milling of fish \cite{Chuang2007,Abaid2010,Albi2014,Chuang2016} is another second-order model and presents both a $2$ and $3$-dimensional milling pattern over long time and introduces a non-collective force from the environment. 
A model of oscillators (fireflies) that sync and swarm together, and have their dynamics governed by their positions and a phase variable $\xi$, was studied by \cite{Strogatz2000,OKeeffe2017,OKeeffe2018,Okeeffe2019}.
There are also models that include both energy and alignment interaction kernels, a particular case of this is the anticipation dynamics model from \cite{shu2019anticipation}, which we study in this work.

One can also consider a collection of celestial bodies interacting via the gravitational potential, which was initially studied in \cite{Zhong20} and is further studied in the upcoming \cite{GravLetter20}. All of these models fit into our framework and we have presented detailed studies of these, and others in this work as well as \cite{Zhong20, Tang2019, lu2019nonparametric}. 
These dynamics exhibit a wide range of emergent behaviors, and 
as shown in \cite{Vicsek_model,TKIHLC2013,CM2008,GC2004,
CHDOB2007,BDT2017,MT2014}, the behaviors can be studied when the governing equations are known. However, if the equations are not known and the data consists of only trajectories, we still wish to develop a model that can both make accurate predictions of the trajectories and discover a dynamical form that accurately reflects their emergent properties. To achieve this, we present a theoretically optimal learning algorithm that is accurate, captures emergent behavior for large time, and, by exploiting the structure of the collective dynamical system, avoids the curse of dimensionality. 

Applying machine learning to the sciences has experienced tremendous growth in recent years, a small selection of general applications related to the ideas in this work include: learning PDEs (\cite{Bar-Sinai,Schaeffer6634, Kevri_2}), modeling dynamical systems (\cite{Kevrik_1, Kevri_3, Kevri_4}), governing equations (\cite{Champion, Kevri_5}), biology (\cite{Chiel993}), fluid mechanics (\cite{Raissi1026, Han219}), many-body problems in quantum systems (\cite{Carleo602}), mean-field games (\cite{Ruthotto}), meteorology (\cite{Ham2019}), and dynamical systems (\cite{Costa1501, Brunton3932, YairE, Bongard9943}). These, and the references therein, give a flavor of the diverse range of applications.
A vast literature exists in the context of learning dynamical systems. In the case of a general nonlinear dynamical system, symbolic regression has been developed to learn the underlying form of the equations from data, see \cite{bongard2007automated, schmidt2009distilling}. Sparse regression techniques which use an extremely large collection of functions, often containing most major mathematical functions, are fit to the data with a sparsity condition that only allows a few terms to appear in the final model. Detailed study and development of these approaches can be found for SINDy in (\cite{BPK2016, RBPK2017, BKP2017}), a LASSO-type penalty (\cite{han2015robust, kang2019ident}), and sparse Bayesian regression (\cite{zhang2018robust}). Other approaches consider multiscale methods, statistical mechanics, or force-based models, see \cite{BCCCCGLOPPVZ2008,BCGMSVW2012}. 
Deep learning has also been applied to learn dynamical systems, for ODEs see \cite{raissi2018multistep, rudy2019deep} and for PDEs see \cite{raissi2018deep, raissi2018hidden, long2017pde}, as well as the references therein.

The majority of the earliest work in inferring interaction kernels in systems of the type \eqref{eq:intro:second_order}, \eqref{eq:2ndOrder} occurred in the Physics literature, going back to the works of Newton. From the viewpoint of purely data-driven analysis of the equations, requiring limited or no physical reasoning, foundational work was \cite{LLEK2010,KTIHC2011}. In these works, the interaction kernels are assumed to be in the span of a known family of functions and parameters are estimated. 
 In statistics, the problem of parameter estimation in dynamical systems from observations is classical, e.g. \cite{varah1982spline,brunel2008parameter,
liang2008parameter,cao2011robust,ramsay2007parameter}. The question of identifiability of the parameter emerges, see e.g. \cite{dattner2015optimal,miao2011identifiability}. Our work is closely related to this viewpoint but our parameter is now infinite-dimensional, with identifiability discusses in section \ref{sec:coercivity}.

Our learning approach is based on exploiting the structure of collective dynamical systems and nonparametric estimation techniques (see \cite{cucker2002mathematical, Tsybakov:2008, Gyorfi06, devore2006approximation, binev2005universal}).
We focus here on second-order models and the form of the equations, generalizing the first order models (see discussion in Appendix \ref{sec:1stOrder}), derived from Newton's second law: for $i = 1, \ldots, N$
\begin{align}\label{eq:intro:second_order}
\hspace{-2cm} m_i\ddot\bx_i(t) = \forcev(\bx_i, \dot\bx_i) + \frac{1}{N} \sum_{\substack{i' = 1,\\ i' \neq i}}^N &\intkernel^E(\norm{\bx_{i'}(t) - \bx_i(t)})(\bx_{i'}(t) - \bx_i(t)) \nonumber \\ &+ \intkernel^A(\norm{\bx_{i'}(t) - \bx_i(t)})(\dot{\bx}_{i'}(t) - \dot{\bx}_i(t)) .
\end{align}
Here, $m_i$ is the mass of the $i^{th}$ agent, $\bx_i$ is its position, $\forcev$ is a non-collective force, and $\intkernel^E: \R^+ \rightarrow \R$, $\intkernel^A: \R^+ \rightarrow \R$ are known as the interaction kernels.  A significant amount of research on modeling collective dynamics is concerned with inducing desired collective behaviors (flocking, clustering, milling, etc.) from relatively simple, local and non-local interaction kernels, often from known or specific parametric families of relatively simple functions. Here we consider a non-parametric, inverse-problem-based approach to infer the interaction kernels from observations of trajectory data, especially within short-time periods.  In \cite{BFHM17}, a convergence study of learning unknown interaction kernels from observation of first-order models of homogeneous agents was done for increasing $N$, the number of agents. 
The estimation problem with $N$ fixed, but the number of trajectories $M$ varying, for first-order and second-order models of heterogeneous agents was numerically studied in \cite{lu2019nonparametric} and learning theory on these first-order models was developed in \cite{Tang2019, li2019identifiability}.  Further extension of the model and algorithm to more complicated second-order systems, with particular emphasis on emergent collective behaviors, was discussed in \cite{Zhong20}.  A big data application to real celestial motion ephemerides is developed and discussed in \cite{GravLetter20}.  In this work, we provide a rigorous learning theory covering the models presented in \cite{Zhong20}, as well as the second-order models introduced in \cite{lu2019nonparametric}.  We consider generalizations of the models in \cite{Zhong20}, to include models with higher-dimensional interaction kernels, that do not depend only on pairwise distances. Compared to the theories studied in \cite{Tang2019, li2019identifiability}, our theory focuses on second-order models with interaction kernels of the form $\intkernele(r)r + \intkernela(r)\dot{r}$ (with $r$ and $\dot r$ representing norms of differences of positions and, respectively, velocities of pairs of agents); additionally, we discuss the identifiability and separability of $\intkernele$ and $\intkernela$ from the sum.

The overall objective of the algorithm can be stated as: given trajectory data generated from an interacting agent system, we wish to learn the underlying interaction kernels, from which we will understand its long-term and emergent behavior, and ultimately build a highly accurate approximate model that faithfully captures the dynamics. We make minimal assumptions on the form of the interaction kernels and the various forces involved, and in some cases the assumptions are made for purely theoretical reasons and the algorithm can still perform well when they do not hold for a given system.
We offer a learning approach to address these collective systems by first discovering the governing equations from the observational data, and then using the estimated equations for large-time predictions. 

The approach in this paper builds on many of these ideas and uses observation data coming from collective dynamical systems of the form \eqref{eq:2ndOrder}, to learn the underlying interaction functions. This variational approach was initially developed in \cite{BFHM17, lu2019nonparametric} and further studied and extended in \cite{Tang2019,Zhong20}.
Our analysis of this system blends differential equations, inverse problems and nonparametric regression, and (statistical) learning theory. A central insight is that we exploit the form of \eqref{eq:2ndOrder} to move the inference task to just the unknown functions ($\intkernel^E,\intkernel^A,\intkernel^{\xi}$) allowing us to avoid the curse of dimensionality incurred if we were to directly perform regression against the high-dimensional phase space and trajectory data of the system, which provides independence of observations where each observation is a different trajectory.

To use the trajectory data to derive estimators, we consider appropriate hypothesis spaces in which to build our estimators, measures adapted to the dynamics, norms, and other performance metrics, and ultimately an inverse problem built from these tools. 
Once we have obtained this estimated interaction kernel, we want to study its properties as a function of the amount of trajectory data we receive, which is the $M$ trajectories sampled from different initial conditions from the same underlying system. Here we study properties of the error functional, establish the uniqueness of its minimizers, and use the probability measures to define a dynamics-adapted norm to measure the error of our estimators over the hypothesis spaces. In comparing the estimators to the true interaction kernels, we first establish concentration estimates over the hypothesis space.

Our first main result is the strong consistency of our learned estimators asymptotically. For the relevant definitions see \ref{sec:learn_framework} and for the full theorem see section \ref{sec:consistency}, which for the model \eqref{eq:intro:second_order} yields,
\begin{align}
\lim_{M\rightarrow \infty}\|\bphEA - \bpEA\|_{\LtwoB(\brhoEAL)} =0 \text{ with probability one.}
\end{align}
where $\brhoEAL$ is a dynamics-adapted measure on pairwise distances -- and we use a weighted $\LtwoB$ space detailed in section \ref{sec:LT}, particularly \eqref{meas:vectorized}. 
Perhaps most importantly, we give a rate of convergence in terms of the $M$ trajectories. We achieve the minimax rate of convergence for any number of variables $|\spacevariables|$ in the interaction kernels. See section \ref{sec:rateofconv} for the full theorem, (see section \ref{sec:learn_framework} for relevant definitions) which is given by:
\begin{align}
 \mathbb{E}_{\bmu}\Big[\| \bphEA 
 -\bpEA \|^2_{\LtwoB(\brhoEAL)} \Big] \leq C \left(\frac{\log M}{M}\right)^{\frac{2s}{2s+|\spacevariables|}}\,. 
 \end{align}
In the case of model \eqref{eq:intro:second_order}, $|\spacevariables|=1$, as in the results for first-order systems \cite{BFHM17,lu2019nonparametric,Tang2019}.

This means that our estimators converge at the same rate in $M$ as the best possible estimator (up to a logarithmic factor) one could construct when the initial conditions are randomly sampled from some underlying initial condition distribution denoted $\bmu$ throughout this work, see (section \ref{sec:main_results}). 

To solve the inverse problem, we give a detailed discussion of an essential link between these three aspects, the notion of coercivity of the system - detailed in section \ref{sec:coercivity}. Coercivity plays a key role in the approximation properties, the ability to approximate the interaction kernel and the learning theory. We also present numerical examples, see the detailed numerical study in \cite{Zhong20}, which help to explain why the particular norms we define are the right choice, as well as show excellent performance on complex dynamical systems in section \ref{sec:numerics}. 

Our paper is structured as follows. The first part of the paper describes the model, learning framework, inference problem, and the basic tools needed for the learning theory. These ideas are all explained in detail in sections \ref{sec:ModelDesc}-\ref{sec:LT}. If one wishes to quickly jump to the theoretical sections, and then refer back to the definitions as needed, we have provided table \ref{tab:2ndOrder_def},\ref{tab:2ndOrderlearning_def} which explains the model equations and outlines the definitions and concepts needed for the learning theory and general theoretical results, respectively.  
The theoretical part of the paper (sections \ref{sec:coercivity}-\ref{sec:trajectory}) discusses fundamental questions of identifiability and solvability of the inverse problem, consistency, and rate of convergence of the estimators, and the ability to control trajectory error of the evolved trajectories using our estimators. Some key highlights of our theoretical contributions are described in \ref{sec:theory_overview}, with full details in the corresponding sections. 
Lastly, we consider applications in section \ref{sec:numerics}, as well as have many additional proofs and details in appendices \ref{s:trajectoryerrorproof}-\ref{sec:algorithm_numeric}.

\subsection{Comparison with existing work}
Our learning approach discovers the governing structure of a particular subset of dynamical systems of the form,
\[
\dot\bY(t) = \mbf{F}_{\bpEA, \bpxi}(\bY(t)), \quad t \in [0, T].
\]
from observations $\{\bY(t_l), \dot\bY(t_l)\}_{l = 1}^L$, by implicitly inferring the right hand side, $\mbf{F}_{\bpEA, \bpxi}$. The main difficulties in establishing an effective theory of learning $\mbf{F}_{\bpEA, \bpxi}$ are the \textit{curse of dimensionality} caused by the size of $\bY$, which is often $N(2d + 1)$, where $N$ is the number of agents, $d$ the dimension of physical space; and the \textit{dependence} of the observation data, for example $\bY(t_{l + 1})$ depends on $\bY(t_{l})$.

There are many techniques which can be used to tackle the high-dimension of the data set: sparsity assumptions, dimension reduction, reduced-order modeling, and machine learning techniques trained using gradient-based optimization. The dependent nature of the data prevents traditional regression-based approaches, see the discussion in \cite{Tang2019}, but many of the approaches above successfully address this. Our work, however, exploits the interacting-agent structure of collective dynamical systems, which is driven by a collection of two-body interactions where each interaction depends only on pairwise data between the states of agents, as in \eqref{eq:intro:second_order}. With this structure in mind, we are able to reduce the ambient dimension of the data $N(2d + 1)$ to the dimension of the variables in the interaction kernels, which is independent of $N$. We also naturally incorporate the dependence in the data in an appropriate manner by considering trajectories generated from different initial conditions.

Our theoretical results focus on the joint learning of $\bintkernele, \bintkernela$ that takes into account their natural weighted direct sum structure that is described in the following sections, which is different from the learning theory on single $\bintkernele$'s considered in \cite{lu2019nonparametric, Tang2019}.  The current theoretical framework is not able to conclusively show that $\bintkernele$ and $\bintkernela$ can be learned separately; however we demonstrate in various numerical experiments that by learning $\bintkernele$ and $\bintkernela$ jointly, we still achieve strong performance. Finally, we note that the first-order theory developed in \cite{Tang2019} is a special case of our second-order theory, see details in appendix \ref{sec:1stOrder}.

\section{Model description} \label{sec:ModelDesc}
%%Test notation: $\bm{\mathbb{S}}$
In order to motivate the choice of second-order models considered in this paper, we begin our discussion with a simple second-order model derived from the classical mechanics point of view.  Let us consider a closed system of $N$ homogeneous agents (or particles) equipped with a certain type of Lagrangian energy, namely $L(t)$ for the whole system, given as follows,
\[
L(t) = \frac{1}{2}m_i\norm{\dot\bx_i(t)}^2 - \frac{1}{N}\sum_{1 \le i < i' \le N}U(\norm{\bx_{i'}(t) - \bx_i(t)}), \quad i = 1, \ldots, N.
\]
Here $U$ is a potential energy depending on pairwise distance.  From the Lagrange equation, $\frac{d}{dt}\Big(\partial_{\dot\bx_i}L\Big) = \partial_{\bx_i}L$, we obtain a simple second-order collective dynamics model
\begin{equation}\label{eq:2nd_simple}
m_i\ddot\bx_i(t) = \frac{1}{N}\sum_{i' = 1, i' \neq i}^N\intkernele(\norm{\bx_{i'}(t) - \bx_i(t)})(\bx_{i'}(t) - \bx_i(t)), \quad i = 1, \ldots, N.
\end{equation}
Here, $\intkernele(r) = \frac{U'(r)}{r}$ represents an energy-based interaction between agents.  For example, taking $U(r) = \frac{NGm_{i'}m_i}{r}$, it becomes the celebrated model for Newton's universal gravity.  
In order to incorporate more complicated behaviors into the model equation, we consider alignment-based interactions (to align velocities, so that short-range repulsion, mid-range alignment, and long arrange attraction are all present), auxiliary variables describing internal states of agents (emotion, excitation, phases, etc.), and non-collective forces (interaction with the environment).  We also consider a system of $N$ heterogeneous agents, such that the agents belong to $\numcl$ disjoint types $\{\cl_{\idxcl}\}_{\idxcl = 1}^{\numcl}$, with $N_{\idxcl}$ being the number of agents of type $\idxcl$. They interact according to the system of ODEs 
\begin{equation}
\begin{dcases}
m_i\ddot\bx_i(t) &= \forcev(\bx_i(t), \dot\bx_i(t), \xi_i(t)) + \sum_{i'=1}^N \frac{1}{N_{\clof_{i'}}} \Big[\intkernele_{\clof_i \clof_{i'}}(r_{ii'}(t), \topwde_{i i'}(t))(\bx_{i'}(t) - \bx_i(t)) \\
&\qquad\qquad\qquad\qquad\qquad\qquad\qquad + \intkernela_{\clof_i \clof_{i'}}(r_{ii'}(t), \topwda_{i i'}(t))(\dot\bx_{i'}(t) - \dot\bx_i(t))\Big]  \\
\dot\xi_i(t) &= \forcexi(\bx_i(t), \dot\bx_i(t), \xi_i(t)) + \sum_{i'=1}^N \frac{1}{N_{\clof_{i'}}} \intkernelxi_{\clof_i \clof_{i'}}(r_{ii'}(t), \topwdxi_{i i'}(t))(\xi_{i'}(t) - \xi_i(t))
\end{dcases}
\label{eq:2ndOrder}
\end{equation}
for $i = 1, \ldots, N$, where $\clof_i, \clof_{i'} \in \{1,\cdots, K\}$ are the indices of the agent type of the agents $i$ and $i'$ respectively; the interaction kernels, $\intkernele_{\idxcl \idxcl'}, \intkernela_{\idxcl\idxcl'}, \intkernelxi_{\idxcl\idxcl'}$, are in general different for interacting agents of different types, and they not only depend on pairwise distance $r_{ii'}(t)$, but also on  other features, given by $\topwde_{i i'}(t), \topwda_{i i'}(t), \topwdxi_{i i'}(t)$.  For example, the interactions between birds can also depend on the field of vision, not just the distance between pairs of birds. Note that we will often suppress the explicit dependence on time $t$ when it is clear from context. 
The unknowns, for which we will construct estimators, in these equations, are the functions $\intkernele_{\clof_i \clof_{i'}}$, $\intkernela_{\clof_i \clof_{i'}}$ and $\intkernelxi_{\clof_i \clof_{i'}}$; everything else is assumed given.

Table \ref{tab:2ndOrder_def} gives a detailed explanation for the definition of the variables used in \eqref{eq:2ndOrder}.
We note that in what follows the notation, $\wildcard$, attached to any expression means that there is one of those maps, functions, etc. for each element in the set $\wildcard$. It is a convenient way to avoid excessive repetition of similar definitions. 
\begin{table}[H]
\centering
\small{
\small{\begin{tabular}{ c | l }
\hline
Variable                    & Definition \\
\hline\hline
$i,i'$                     & index of agent, from $1,\ldots, N$ \\
\hline
$m_i$                       & mass of agent $i$ \\
\hline
$\bx_i(t),\dot{\bx}_i(t),\ddot{\bx}_i(t)\in \R^d$          & position/velocity/acceleration vector of agent $i$ at time $t$ \\
\hline 
$\xi_i, \dot{\xi}_i$					& auxiliary variable, and its derivative  \\
\hline
$\|\cdot\|$ & Euclidean norm in $\mathbb{R}^d$ \\
\hline
$\numcl$                    & number of types of agents \\
\hline
$\idxcl, \idxcl'$            & index of type of agents, from $1,\ldots, \numcl$  \\
\hline
$N_{\idxcl}$                & number of agents in type $\idxcl$ \\
\hline
$\clof_i$                   & type index of agent $i$ \\
\hline
$\cl_{\idxcl}$              &  set of indices for type $k$ agent, subset of $\{1, \ldots, N\}$   \\
\hline
$\intkernel_{kk'} $ (or $\intkernel_{kk''}$) & interaction kernel: influence of agent $k'$ (or $k''$) on agent $k$ \\
%how the agents in type $k$ influence agents in type $k'$ (or $k''$) \\
\hline
$\forcev, \forcexi$         & non-collective forces on $\ddot\bx_i$ and $\dot\xi_i$ respectively \\
\hline
$\intkernele, \intkernela, \intkernelxi$ & energy, alignment, and environment-based interaction kernels respectively \\
\hline 
\iffalse
$M$							& number of trajectories \\
\hline 
$L$							& number of observations in time \\
\hline						
$n$							& number of basis functions used \\
\fi
%\hline $\topwdT $      & \shortstack{Features used across any of the $(k,k')$ pairs over E,A, and $\xi$ \\ $\topwdT(\bx, \dot{\bx}, \xi, \bx ', \dot{\bx}', \xi '): \R^{4d + 2} \rightarrow \mathbb{R}^{p}$} \\
\hline $\topwdT $      & Features, $\topwdT(\bx, \dot{\bx}, \xi, \bx ', \dot{\bx}', \xi '): \R^{4d + 2} \rightarrow \mathbb{R}^{p}$ \\

\hline
$\topwd_{(\idxcl, \idxcl')}^{\wildcard}$            & Feature map, $\pi_{\idxcl\idxcl'}^{\wildcard}\circ \topwdT(\bx, \dot\bx, \xi, \bx', \dot\bx', \xi'): \R^{4d + 2} \rightarrow \R^{p_{(\idxcl,\idxcl')}^{\wildcard}}$ \\

\hline
$\topwd_{ii'}^{\wildcard}$            & Feature evaluation, 
$\topwd_{(\clof_i,\clof_{i'})}^{\wildcard}(\bx_i, \dot\bx_i, \xi_i, \bx_{i'}, \dot\bx_{i'}, \xi_{i'}) \in \R^{p_{\clof_{i}\clof_{i'}}^{\wildcard}}$ \\

\hline
\end{tabular}}  
}
\caption{Notation for the variables in \eqref{eq:2ndOrder}.}
\label{tab:2ndOrder_def} 
\end{table}
The specific instances of the feature map $\topwdT$ together with corresponding projections $\pi_{kk'}^{\wildcard}$ include a variety of systems that have found a wide range of applications in physics, biology, ecology, and social science; see the examples in the chart below. We assume that the function $\mathcal{V}$ is Lipschitz and known and so are all the $\pi_{kk'}^{\wildcard}$'s. The Lipschitz assumption is necessary for us to control the trajectory error and ensure the well-posedness of the system. The function $\mathcal{V}$ is a uniform way to collect all of the different variables (functions of the inputs) used across any of the $(k,k')$ pairs over all of the $E,A,\xi$ functions in the system. This uniformity is helpful when discussing the rate of convergence, among other places. Examples of where this generality matters emerge naturally, say when one has a different number of variables across interaction kernels for different pairs $(k,k')$, or when the energy and alignment kernels depend on $r$ and then additional but distinct other variables. 
From this uniform set of variables, we then project (which of course implies that the feature maps are all Lipschitz) to arrive at the relevant function $\topwd_{(\idxcl, \idxcl')}^{\wildcard}$ for each pair and each of the elements of the wildcard. Lastly, we can then evaluate this map at the specific pair of agents $(i,i')$, that leads to the feature evaluation, $\topwd_{ii'}^{\wildcard}$ which is the expression used in the model equation \eqref{eq:2ndOrder}. 

The models encompassed by the form \eqref{eq:2ndOrder} are quite diverse. For a concrete example, please see section \ref{sec:example_AD}.
We summarize many examples in table \ref{tab:2ndOrder_Examples} with a shaded cell indicating that the model has that characteristic, and an empty cell indicates that the model does not have this characteristic. A numeric value indicates this is the number of unique variables, $\spacevariables, \spacevariables_{\xi}$ used within the EA or $\xi$ portions of the system. The number of these unique variables specifies the dimension in the minimax convergence rate, see section \ref{sec:rateofconv}.

\begin{table}[]
\begin{tabular}{|l||l|l|l|l|l|l|l|l|l|l|l|l|}
\hline
      & \multicolumn{12}{l|}{Properties} \\ \hline 
Model &$\phi^E$&$\phi^A$&$m_i$&$\forcev$&$\phi^{\xi}$&$\forcexi$ &$K>1$&$\topwde$&$\topwda$&$\topwdxi$&$|\spacevariables|$&$|\spacevariables_{\xi}|$\\ \hline
Anticipation Dynamics  &    \cellcolor[HTML]{343434}    &   \cellcolor[HTML]{343434}     &  \cellcolor[HTML]{343434}   &         &            &           &     &    \cellcolor[HTML]{343434}     &     \cellcolor[HTML]{343434}    &          &          2       &  \\ \hline
Celestial Mechanics     &   \cellcolor[HTML]{343434}      &        &  \cellcolor[HTML]{343434}   &         &            &           & \cellcolor[HTML]{343434}     &         &         &          &         1       &  \\ \hline
Cucker-Smale      &        &   \cellcolor[HTML]{343434}      &  \cellcolor[HTML]{343434}   &         &            &           &     &         &         &          &       1          &  \\ \hline
Fish Milling 2D      &    \cellcolor[HTML]{343434}   &        & \cellcolor[HTML]{343434}   &   \cellcolor[HTML]{343434}       &            &           &     &         &         &          &        1         &  \\ \hline
Fish Milling 3D &    \cellcolor[HTML]{343434}     &        &  \cellcolor[HTML]{343434}   &    \cellcolor[HTML]{343434}      &            &           &     &         &         &          &          1       &  \\ \hline
Flocking w. Ext. Poten.  &    \cellcolor[HTML]{343434}     & \cellcolor[HTML]{343434}       & \cellcolor[HTML]{343434}   &    \cellcolor[HTML]{343434}      &            &           &     &         &         &          &          1       &  \\ \hline
Phototaxis     &        &    \cellcolor[HTML]{343434}     & \cellcolor[HTML]{343434}   &  \cellcolor[HTML]{343434}        &           \cellcolor[HTML]{343434}  &       \cellcolor[HTML]{343434}     &      &         &         &          &       1          &  1 \\ \hline
Predator-Swarm ($2^{nd}$ Order)    &    \cellcolor[HTML]{343434}     &        &  \cellcolor[HTML]{343434}   &   \cellcolor[HTML]{343434}       &            &           &    \cellcolor[HTML]{343434}  &         &         &          &      1           &  \\ \hhline{|=||=|=|=|=|=|=|=|=|=|=|=|=|}
Lennard-Jones     &   \cellcolor[HTML]{343434}      &        &     &         &            &    \cellcolor[HTML]{343434}        &     &         &         &          &          1       &  \\ \hline      
 Opinion Dynamics  &  \cellcolor[HTML]{343434}      &        &    &         &            &           &     &         &         &          &         1        &  \\ \hline
 Predator-Swarm ($1^{st}$ Order)    &    \cellcolor[HTML]{343434}     &        &     &          &            &           &    \cellcolor[HTML]{343434}  &         &         &          &      1           &  \\ \hline
Synchronized Oscillator    &        &    \cellcolor[HTML]{343434}     &     &   \cellcolor[HTML]{343434}       &    \cellcolor[HTML]{343434}         &   \cellcolor[HTML]{343434}         &     &     \cellcolor[HTML]{343434}     &         &      \cellcolor[HTML]{343434}     &                2 & 2 \\ \hline
\end{tabular}
\caption{Summary of the models studied in this work and in \cite{lu2019nonparametric,Tang2019,Zhong20,GravLetter20}} 
\label{tab:2ndOrder_Examples}
\end{table}
Our second-order model equations cover the first-order models considered in \cite{lu2019nonparametric, Tang2019, Zhong20} as special cases, see Appendix \ref{sec:1stOrder}, which is why we choose second-order models as the main focus of this work.  Furthermore, the dynamical characteristics produced by second-order models are much richer and can model more complicated collective motions and emergent behavior of the agents. Note that our second-order model in \eqref{eq:2ndOrder}, even when written as a first-order system in more variables, is a strict generalization of the previous first-order analysis. 

\subsection*{Rate of convergence notation}
For the system \eqref{eq:2ndOrder}, depending on the number of variables, we will have a different rate of convergence as the dimension of the function(s) we are learning changes. In order to present a unified theorem, we adopt the following notation. Let $\spacevariables$ denote the set of features in the range of $\topwdT$ that are arguments of $\bintkernel^E$ or $\bintkernel^A$ across each of the $(k,k')$ pairs. 
For many collective dynamical systems $\spacevariables = \{r\}$. In the case where the system has both $\bintkernele$ and $\bintkernela$ but $\spacevariables = \{r\}$, it is easy to see from Theorem \ref{main:convrate} below that we still only pay the $1$-dimensional rate. 
Analogously, we define $\spacevariables_{\xi}$ to be the set of all features in the range of $\topwdT$ that are arguments  involved in the $\bintkernelxi$ part of \eqref{eq:2ndOrder} across all of the $(k,k')$ pairs. This notation is used to frame the convergence rate theorem on the $\xi$ variable as well. 

\section{Preliminaries and notation} \label{sec:notation}
We vectorize the models given in \eqref{eq:2ndOrder} in order to give them a more compact description.  Letting $\bv_i(t) = \dot\bx_i(t)$, we take the following notations
\[
\bX(t) = \begin{bmatrix} \bx_1(t) \\ \vdots \\ \bx_N(t) \end{bmatrix} \in \mathbb{R}^{Nd}, \quad \bV(t) := \begin{bmatrix} \bv_1(t) \\ \vdots \\ \bv_N(t) \end{bmatrix} \in \mathbb{R}^{Nd}, \quad \bXi(t) := \begin{bmatrix} \xi_1(t) \\ \vdots \\ \xi_N(t) \end{bmatrix} \in \mathbb{R}^N.
\]
We introduce a weighted norm to measure the system variables, denoted $\norm{\cdot}_{\mS(\cdot)}$ and given by, 
\begin{equation}\label{eq:sys_var_norm}
\norm{\bZ}_{\mS}^2 := \sum_{i = 1}^N\frac{1}{N_{\clof_i}}\norm{\bz_i}^2,
\end{equation}
for $\bZ = \begin{bmatrix} \bz_1^\top, & \ldots, & \bz_N^\top \end{bmatrix}^\top$ with each $\bz_i \in \R^{d}$ or $\R$. Here $\norm{\cdot}$ is the same norm used in the construction of pairwise distance data for the interaction  kernels. In the subsequent equations we {drop the explicit dependence on} $t$ for simplicity.  The weight factor, $\frac{1}{N_{\clof_i}}$, is introduced so that agents of different types are considered equally and we learn well even in the case that the number of agents in the classes is highly non-uniform.
With thsee vectorized notations, the model in \eqref{eq:2ndOrder} becomes,
\[
\begin{dcases}
\vec{m} \circ \ddot\bX &= \rhsfvnc(\bX, \bV, \bXi) + \rhsf^{\bm{\phi}^E}(\bX, \bV, \bXi) + \rhsf^{\bm{\phi}^A}(\bX, \bV, \bXi) \\
\dot\bXi &= \rhsfxinc(\bX, \bV, \bXi) + \rhsf^{\bintkernelxi}(\bX, \bV, \bXi).
\end{dcases}
\]
Here $\vec{m} = \begin{bmatrix} m_1, & \ldots, & m_N \end{bmatrix}^\top \in \mathbb{R}^N$, $\circ$ is the Hadamard product, and we use boldface fonts to denote the vectorized form of our estimators (with some once-for-all-fixed ordering of the pairs $(k,k')_{k,k'=1,\dots,K}$):
\begin{equation} \label{not:vectorized}
\bm{\phi}^E = \{\intkernele_{\idxcl \idxcl'}\}_{\idxcl, \idxcl' = 1}^{\numcl}, \quad \bm{\phi}^A = \{\intkernela_{\idxcl \idxcl'}\}_{\idxcl, \idxcl' = 1}^{\numcl}, \quad \bm{\phi}^{\xi} = \{\intkernelxi_{\idxcl \idxcl'}\}_{\idxcl, \idxcl' = 1}^{\numcl}.
\end{equation}
We also use the shorthand, 
\begin{equation} \label{not:shorthand}
\bpEA = \bpE \oplus \bpA, 
\end{equation}
to denote the element of the direct sum of the function spaces containing $\bpE, \bpA$. This notation will be used throughout on the energy and alignment ($EA$ for short) portion of the system in order to simplify the notation. 

We denote by $\rhsfvnc$, the vectorized notation for the non-collective force defined as follows,
$\rhsfvnc(\bX, \bV, \bXi) := \begin{bmatrix} \forcev(\bx_1, \dot\bx_1, \xi_1), \hdots, \forcev(\bx_N, \dot\bx_N, \xi_N)  \end{bmatrix}^T \in \mathbb{R}^{Nd}$,
and
$$\rhsf^{\bm{\phi}^E} := \bigg[
\sum_{i'=1}^N \frac{1}{N_{\clof_{i'}}} \intkernele_{\clof_1 \clof_{i'}}(r_{1i'}, \topwde_{1i'})(\bx_{i'} - \bx_1),
\hdots,
\sum_{i'=1}^N \frac{1}{N_{\clof_{i'}}} \intkernele_{\clof_N \clof_{i'}}(r_{Ni'}, \topwde_{Ni'})(\bx_{i'} - \bx_N) 
\bigg]^T \in \mathbb{R}^{Nd}.
$$
We omit the analogous definitions for $\rhsf^{\bm{\phi}^A}$ and $\rhsf^{\bm{\phi}^{\xi}}$.

\subsection{Trajectory Performance Measurement}  \label{subsec:trajperf}
We will also consider another measurement to assess the learning performance of the estimated kernels in terms of trajectory error. We compare the observed trajectories to the estimated trajectories evolved from the same initial conditions but with the estimated interaction kernels.   Let $\bY(t) = [ \bX^\top(t),  \bV^\top(t),  \bXi^\top(t) ]^\top$ be the trajectory from dynamics generated by the true/unknown interaction kernels with initial condition, $\bY(0)$; and $\hat\bY(t) = [ \hat\bX^\top(t), \hat\bV^\top(t), \hat\bXi^\top(t)]^\top$ be the trajectory from dynamics generated by the estimated interaction kernels learned from observation of $\{\bY(t_l)\}_{l = 1}^L$ with the same initial condition, $\bY(0)$ (i.e., $\hat\bY(0) = \bY(0)$).  We define a norm for the difference between $\bY$ and $\hat\bY$ at time $t_l$:

\begin{equation}\label{eq:y_norm}
\norm{\bY(t_l) - \hat\bY(t_l)}_{\mY}^2 = \norm{\bX(t_l) - \hat\bX(t_l)}_{\mS}^2 + \norm{\bV(t_l) - \hat\bV(t_l)}_{\mS}^2 + \norm{\bXi(t_l) - \hat\bXi(t_l)}_{\mS}^2\,,
\end{equation}
and a corresponding norm on the trajectory $\bY_{[0, T]} = \{\bY(t_l)\}_{l = 1}^L$ ($0 = t_1 < \dots < t_L = T$):
\begin{equation}\label{eq:traj_norm}
\norm{\bY_{[0, T]} - \hat\bY_{[0, T]}}_{\text{traj}} = \max_{l = 1, \ldots, L} \norm{\bY(t_l) - \hat\bY(t_l)}_{\mY}\,.
\end{equation}
We also consider a relative version, invariant under changes of units of measure:
\[
\norm{\bY_{[0, T]} - \hat\bY_{[0, T]}}_{\text{traj}^*} = \frac{\norm{\bY_{[0, T]} - \hat\bY_{[0, T]}}_{\text{traj}}}{\norm{\bY_{[0, T]}}_{\text{traj}}}\,.
\]
Lastly, we report errors between $\bX_{[0, T]}$ and $\hat\bX_{[0, T]}$,
\[
\norm{\bX_{[0, T]} - \hat\bX_{[0, T]}}_{\mS^*} = \frac{\max_{l = 1,\ldots,L}\{\norm{\bX(t_l) - \hat\bX(t_l)}_{\mS}\}}{\max_{l = 1, \ldots,L}\{\norm{\bX(t_l)}_{\mS}\}}.
\]
Similar re-scaled norms are used for the difference between $\bV_{[0, T]}$ and $\hat\bV_{[0, T]}$, and for the difference between $\bXi_{[0, T]}$ and $\hat\bXi_{[0, T]}$.

\subsection{Function spaces} \label{sec:functionspaces}
We begin by describing some basic ideas about measures and function spaces. 
Consider a compact or precompact set $\mathcal{U} \subset \mathbb{R}^{p}$ for some $p$, then we define the infinity norm as,
$$\|h\|_{\infty}=\operatorname{ess} \sup _{x \in \mathcal{U}}|h(x)|.$$
Further define, $L^{\infty}(\mathcal{U})$ as the space of real valued functions defined on $\mathcal{U}$ with finite $\infty$-norm. 
A key function space we need to consider is, 
$$C_{c}^{k, \alpha}(\mathcal{\mathcal{U}}) \text { for } k \in \mathbb{N}, 0<\alpha \leq 1,$$ defined as the space of compactly supported, $k$ times continuously differentiable functions with a $k$-th derivative that is H\"older continuous of order $\alpha$. 
We can then consider vectorizations of these spaces as 
$$\boldsymbol{L}^{\infty}(\mathcal{U}):=\bigoplus_{k, k^{\prime}=1,1}^{K, K} L^{\infty}(\mathcal{U}),$$ which has the vectorized infinity norm given by 
$$\|\boldsymbol{f}\|_{\infty}=\max _{k, k^{'}}\left\|f_{k k^{\prime}}\right\|_{\infty}, \forall \boldsymbol{f} \in \boldsymbol{L}^{\infty}(\mathcal{U}).$$
Similarly, we consider direct sums of measures, with corresponding vectorized function spaces. This is done explicitly in section \ref{sec:PerfMeasures} where we define a weighted $L^2$ space (under a particular measure with an associated norm that induces a weighting) and consider vectorized versions of it.   

In order to develop a theoretical foundation, and in line with the literature, we make assumptions on the class of functions that can arise as interaction kernels in the model \eqref{eq:2ndOrder}. As the agents get farther and farther apart, they eventually should have no influence on each other. This is an approximation to the vanishing, or rapidly decreasing to zero, nature of pairwise interaction as distance increases that is observed in many physical models. We thus assume a maximum interaction radius for the interaction kernels which represents the maximal distance at which one agent can influence another. Similar assumptions will be made on the feature maps.  

More precisely, for each pair $(k,k')$ we consider the following spaces,
\begin{align}
\mathcal{K}_{\idxcl \idxcl'}^\wildcard := L^{\infty}([R_{\idxcl \idxcl'}^{\min}, R_{\idxcl \idxcl'}^{\max}] \times \mathbb{S}^{\wildcard}_{\idxcl \idxcl'})\,,
\end{align}
for $\idxcl, \idxcl' = 1, \ldots, \numcl$ where recall that the $\wildcard$ notation means that there is an admissible space (or more generally an expression) for each element of the set $\wildcard$. 
Here, $R_{\idxcl\idxcl'}^{\min},R_{\idxcl\idxcl'}^{\max}$ are the minimum or maximum, respectively, possible interaction radii for agents in $\cl_{\idxcl'}$ influencing agents in $\cl_{\idxcl}$.
Similarly, $\mathbb{S}^{E}_{\idxcl \idxcl'}, \mathbb{S}^{A}_{\idxcl \idxcl'}, \mathbb{S}^{\xi}_{\idxcl \idxcl'}$ are compact sets in $\R^{p_{kk'}^E}, \R^{p_{kk'}^A}, \R^{p_{kk'}^{\xi}}$ which contain the ranges of the feature maps, $\topwde_{kk'}, \topwda_{kk'}$ and $\topwdxi_{kk'}$.  

We can also define the vectorizations of these spaces, which we will look at subsets of when we define the admissible spaces below, given by
$$
\bm{\mathcal{K}}^{\wildcard} :=\bigoplus_{k,k'=1,1}^{K,K} \mathcal{K}_{\idxcl \idxcl'}^\wildcard.
$$
In order to provide uniform bounds, we introduce the following sets:
\begin{equation} \label{eq:compactsets}
\FeatureHypSpaceE := \prod_{k,k'}\mathbb{S}_{kk'}^E \qquad 
\FeatureHypSpaceA := \prod_{k,k'}\mathbb{S}_{kk'}^A \qquad
\FeatureHypSpacexi := \prod_{k,k'}\mathbb{S}_{kk'}^{\xi}
\end{equation}

Next, we introduce notation to bound the interaction radii on the pairwise distances and pairwise velocities. 
\begin{remark}
Here we note an important distinction. There is the range of the norms of pairwise interactions generated by the dynamics,  and the underlying support of the interaction kernels themselves. These two notions of interaction radius are distinct, we will comment on the estimation and subtleties of both in the numerical algorithm section. 
\end{remark}
We let 
\begin{align} \label{eq:Rdef}
\PairHypSpace &:= \prod_{k,k'}[R_{kk'}^{\min},R_{kk'}^{\max}]\\
R &:= \max_{k,k'}R_{kk'}^{\max}.
\end{align}
Notice that a uniform support for all interaction kernels on the pairwise distance variable $r$ is $[0,R]$. 

We denote the distribution of the initial conditions by $\bmu$. This measure is unknown and is the source of randomness in our system. It reflects that we will observe trajectories which start at different initial conditions, but that evolve from the same dynamical system, which allows for learnability. For the numerical experiments we will choose our initial conditions to be sampled uniformly over a system dependent range. 

Due to the form of \eqref{eq:2ndOrder}, and the norms defined below, we consider the following dynamics induced ranges of the variables. Note that the first supremum is taken over the initial conditions, each of which generate different solutions $\bx_i(t)$ which are used in the second supremum. 
\begin{align} 
R_{\dot{x}} &:= \sup_{\bY(0) \sim \bmu} \sup_{t \in [0,T]} \max_{i,i'} \|\dot{\bx}_i(t) - \dot{\bx}_{i'}(t) \| \label{eq:Rdyndef} \\
R_{\xi}   &:= \sup_{\bY(0) \sim \bmu} \sup_{t \in [0,T]} \max_{i,i'} \|\xi_i(t) - \xi_{i'}(t) \|
\label{eq:Rdyndef:xi}
\end{align} 
We assume in this work that both of these quantities $R_{\dot{x}}, R_{\xi}$ are finite. This will be easily satisfied if the measures $\bm{\mu^{\dot{x}}}, \bm{\mu}^{\xi}$ (specifying the distribution of the initial conditions on the velocities and the environment variable) are compactly supported. This follows by the assumptions on the interaction kernels below and that we only consider finite final time $T$.
 
In order for the second-order systems given by \eqref{eq:2ndOrder} to be well-posed, we assume that the interaction kernels lie in \textit{admissible sets}. For each of the kernels, let $\mathcal{U}_{kk'}:=[R_{\idxcl \idxcl'}^{\min}, R_{\idxcl \idxcl'}^{\max}] \times \mathbb{S}^{\wildcard}_{\idxcl \idxcl'}$ and define

\begin{align} \label{eq:E-Admissibility}
\boldsymbol{\mathcal{K}}_{S_{\wildcard}}^{\wildcard}&:=\Big\{\left(\intkernel^{\wildcard}_{\idxcl \idxcl'}\right)_{\idxcl,\idxcl' = 1,1}^{\numcl,\numcl}: \text{{ for all }} k,k'=1,\dots,K\,, \\ 
&\qquad\intkernel^{\wildcard}_{\idxcl \idxcl'} \in C_c^{0,1}\left(\mathcal{U}_{kk'}\right),
 \norm{\intkernel^{\wildcard}_{\idxcl \idxcl'}}_{\infty} + \operatorname{Lip}\left[\intkernel^{\wildcard}_{\idxcl \idxcl'}\right] \leq S_{\wildcard}\Big\}\,. \nonumber
\end{align}
When estimating the $EA$ part of the system, we will consider the direct sum admissible space, for $S_{EA}\geq \max\{S_E,S_A\}$,
\begin{equation} \label{eq:dsum-admissibility}
\boldsymbol{\mathcal{K}}_{S_{EA}}^{EA} := \boldsymbol{\mathcal{K}}_{S_{E}}^{E} \oplus \boldsymbol{\mathcal{K}}_{S_{A}}^{A} 
\end{equation}

The admissibility assumptions allow us to establish properties such as existence and uniqueness of solutions to \eqref{eq:2ndOrder} as well as to have control on the trajectory errors in finite time $[0, T]$. It further allows us to show regularity and absolute continuity with respect to Lebesgue measure of the appropriate performance measures defined in section \ref{sec:PerfMeasures}.  

In the learning approach, we will consider hypothesis spaces that we will search in order to estimate the various interaction kernels.  
The hypothesis spaces corresponding to $\{\phi_{kk'}^{\wildcard}\}$ are denoted as $\{\hypspace_{kk'}^{\wildcard}\}$ and we vectorize them as,
\begin{align} \label{not:hypspaces}
\bhypspace^{\wildcard} := \bigoplus_{k,k'=1,1}^{K,K}\hypspace_{kk'}^{\wildcard}.
\end{align}
Analogous to our simplified notation for $\bpEA, \bvpEA$ described in \eqref{not:shorthand}, we define the direct sum of the hypothesis spaces as, 
\begin{equation} \label{not:bhypspaceEA}
\bhypspaceEA := \bhypspace^E\oplus \bhypspace^A
\end{equation} 
We will consider specific hypothesis spaces during the learning theory and numerical algorithm sections.

\section{Inference problem and learning approach}\label{sec:learn_framework}
In this section, we first introduce the problem of inferring the interaction kernels from observations of trajectory data and give a brief review and generalization of the learning approach proposed in the works \cite{lu2019nonparametric} and \cite{Zhong20}. 

\subsection{Problem setting}
 Our observation data is given by $\{\bY^{(m)}(t_l), \smash{\dot\bY^{(m)}(t_l)}\}_{m=1,l= 1}^{M,L}$ for $0 = t_1 < t_2 < \cdots < t_L = T$. Here $\dot\bY=[\by_1^\top,\cdots, \by_N^\top]$ and $ \by_i = \begin{bmatrix}\bx_i^\top(t), \dot\bx_i^\top(t), \xi_i(t)\end{bmatrix}^\top$.  For simplicity, we only consider equidistant observation points: $t_{l} - t_{l - 1} = h$ for $l = 2, \ldots, L$.  The proposed algorithm with slight modifications also works for non-equispaced points. The $M$ sets of discrete trajectory data are generated by the system (2.1) with the unknown set of interaction kernels, i.e. $\bm{\phi}^E, \bm{\phi}^A, \bm{\phi}^{\xi}$, whose  initial conditions $\bY^{(m)}(0)$ are drawn i.i.d from $\bmu$, a probability measure defined on the space $\R^{N(2d+2)}$. The goal is to infer the unknown interaction kernels directly from data.

\subsection{Loss functionals}  \label{sec:lossfunctionals}
Given observations, the references \cite{lu2019nonparametric, Tang2019, Zhong20} proposed an empirical error functional, recalling the notational convention \eqref{not:shorthand},
\begin{eqnarray}
\mbf{\mathcal{E}}_{M}^{EA}(\combfbvpEbvpA)&=&\frac{1}{LM}\sum_{l=1, m=1}^{L,M}
\Big\Vert\ddot{\bX}^{(m)}(t_l) -  \rhsfvnc(\coordsTLM) \label{eq:Error_Func_Empirical:EA} \\ 
&-& \rhsf^{\bvpE}(\coordsTLM) - \rhsf^{\bvpA}(\coordsTLM)\Big\Vert^2_{\mathcal{S}},  \nonumber\\
\mbf{\mathcal{E}}_{M}^{\xi}(\bm{\varphi}^{\xi})&=& \frac{1}{LM} \sum_{l=1, m=1}^{L,M} \Big\Vert\dot{\bm{\Xi}}(t_l) - \rhsfxinc(\coordsTLM) - \label{eq:Error_Func_Empirical:xi} \\
& &\rhsf^{\bm{\varphi}^{\xi}}(\coordsTLM)\Big\Vert^2_{\mathcal{S}}\,.\nonumber
\end{eqnarray}
The estimators of interaction kernels are defined as the minimizers of  the error functionals $\mbf{\mathcal{E}}_{M}^{EA}$ and $\mbf{\mathcal{E}}_{M}^{\xi}$ over $\bhypspaceEA$ and $\bhypspacexi$ respectively, i.e.
\begin{equation} \label{not:minimizers}
\bphEA = \argmin{\bvpEA \in \bhypspaceEA}\mbf{\mathcal{E}}_{M}^{EA}(\bvpEA)\,,
\qquad
\blintkernelxi = \argmin{\bintkernelvarxi \in \bhypspacexi}\mbf{\mathcal{E}}_{M}^{\xi}(\bintkernelvarxi).
\end{equation}
For the learning theory, we will consider the following error functionals.
On the energy and alignment portion, we consider,
\begin{align}
\mbf{\mathcal{E}}_{\infty}^{EA}(\combfbvpEbvpA)&:=\mathbb{E}_{\bmu}\frac{1}{L}\sum_{l=1}^{L}\Big[
\Big\Vert\ddot{\bX}(t_l) -  \rhsfvnc(\coordsTL) \nonumber \\ 
&- \rhsf^{\bvpE}(\coordsTL) - \rhsf^{\bvpA}(\coordsTL)\Big\Vert^2_{\mathcal{S}}
\Big].  \label{eq:Error_Func:EA}
\end{align}
Similarly, on the $\xi$ portion, we consider, 
\begin{equation}
\mbf{\mathcal{E}}_{\infty}^{\xi}(\bvpxi)= \mathbb{E}_{\bmu}\frac{1}{L} \sum_{l=1}^L \Big[
\Big\Vert\dot{\bm{\Xi}}(t_l) - \rhsfxinc(\coordsTL) - \rhsf^{\bvpxi}(\coordsTL)\Big\Vert^2_{\mathcal{S}}
\Big]. \label{eq:Error_Func:xi}
\end{equation}
We can relate these error functionals to the natural empirical error functionals introduced at the start of the section as follows. By the Strong Law of Large Numbers we have that, 
\[
\mbf{\mathcal{E}}_{\infty}^{EA}(\combfbvpEbvpA)= \lim_{M\to \infty} \mbf{\mathcal{E}}_{M}^{EA}(\combfbvpEbvpA) \quad \text{and} \quad
\mbf{\mathcal{E}}_{\infty}^{\xi}(\bm{\varphi}^{\xi}) =\lim_{M\to \infty} \mbf{\mathcal{E}}_{M}^{\xi}(\bm{\varphi}^{\xi}).  
\]
The relationship between the theoretical and empirical error functionals will play a key role in the learning theory.

\subsection{Overview of theoretical contributions} \label{sec:theory_overview}
The papers \cite{Zhong20, lu2019nonparametric,Tang2019}  have applied this learning approach to a variety of systems and the extensive numerical simulations demonstrate the effectiveness of the approach. However, theoretical guarantees of the proposed approach for second order systems had not been fully developed and will be the main focus of this paper. Our theory contains the first-order theory in \cite{Tang2019} as a special case, as discussed in Appendix \ref{sec:1stOrder}.  We focus on the regime where $L$ is fixed but $M \rightarrow \infty$. We provide a learning theory that answers the fundamental questions: 
\begin{itemize}
\item \textbf{Quantitative description of estimator errors}. We will introduce measures to describe how close the estimators are to the true interaction kernels, that lead to novel dynamics-adapted norms. See section \ref{sec:LT}. 
\item \textbf{Identifiability of kernels}. We will establish the existence and uniqueness of the estimators as well as relate the solvability of our inverse problem to a fundamental coercivity property.  See section \ref{sec:coercivity}.
\item  \textbf{Consistency and optimal convergence rate of the  estimators}. We will prove theorems on strong consistency and optimal minimax rates of convergence of the estimators, which exploit the separability of the learning on the energy and alignment from the learning on the environment variable. See section \ref{sec:main_results}.
\item \textbf{Trajectory Prediction} We prove a theorem that describes the performance of the estimated dynamics using the estimated kernels compared to the true dynamics. Our result demonstrates how the expected supremum error (over the entire time interval) of our trajectories is controlled by the norm of the difference between the true and estimated kernels, further justifying our choice of norms and estimation procedure. See section \ref{sec:trajectory}.
\end{itemize}

\subsection{Hypothesis Space and Algorithm}  \label{sec:hypspaceandalg} 
First, we choose finite-dimensional subspaces of $\hypspacee_{\idxcl\idxcl'}$, i.e., $\tilde{\hypspace}_{\idxcl\idxcl'}^E \subset \hypspacee_{\idxcl\idxcl'}$, whose basis functions are piece-wise polynomials of varying degrees (other type of basis functions are also possible, e.g., clamped B-splines as shown in \cite{lu2019nonparametric}); similarly for $\tilde{\hypspace}_{\idxcl\idxcl'}^A \subset \hypspacea$.  Hence, each test function, $\intkernelvare_{\idxcl \idxcl'}, \intkernelvara_{\idxcl \idxcl'}$, can be expressed in terms of the linear combination of the basis functions as follows
\begin{align*}
\intkernelvare_{\idxcl\idxcl'}(r, \topwde) &= \sum_{\eta_{\idxcl\idxcl'}^{E} = 1}^{n_{\idxcl\idxcl'}^{E}} \alpha_{\idxcl, \idxcl', \eta_{\idxcl\idxcl'}^{E}}^{E}\basise_{\idxcl, \idxcl', \eta_{\idxcl\idxcl'}^{E}}(r, \topwde), \\
\intkernelvara_{\idxcl\idxcl'}(r, \dot{r}, \topwda) &= \sum_{\eta_{\idxcl\idxcl'}^{A} = 1}^{n_{\idxcl\idxcl'}^{A}} \alpha_{\idxcl, \idxcl', \eta_{\idxcl\idxcl'}^{A}}^{A}\basise_{\idxcl, \idxcl', \eta_{\idxcl\idxcl'}^{A}}(r, \dot{r}, \topwda),.
\end{align*} 
Substituting this linear combination back into \eqref{eq:Error_Func_Empirical:EA}, we obtain a system of linear equations, 
\[
A^{EA}\vec{\alpha}^{EA} = \vec{b}^{EA}.
\]
Here, $\vec{\alpha}^{EA} \in \R^{n^{EA}} = \begin{bmatrix} (\vec{\alpha}^E)^\top & (\vec{\alpha}^A)^\top \end{bmatrix}^\top$ with $\vec{\alpha}^E$ and $\vec{\alpha}^A$ being the collection of $\alpha_{\idxcl, \idxcl', \eta_{\idxcl\idxcl'}^{E}}^{E}$ or $\alpha_{\idxcl, \idxcl', \eta_{\idxcl\idxcl'}^{A}}^{A}$ respectively.  Moreover, $A^{EA} \in \R^{n^{EA} \times n^{EA}}$ and $\vec{b}^{EA} \in \R^{n^{EA}}$.  See Sec. \ref{sec:algorithm_numeric} for full details. 

%\subsubsection*{Computational complexity}
The overhead memory storage needed is $MLN(d(5 + n^{EA} + n^{\xi}) + 3)$, with $MLN(4d + 2)$ needed for trajectory data, $MLNd(n^{EA} + n^{\xi}$ (here $n^{EA} = n^E + n^A$, the sum of the number of basis functions on $E$ and $A$) for learning matrices, and $MLN(d + 1)$ for right hand side vectors.  Hence if $M \gg \bigO(1)$, we can consider parallelization in $m$ in order to reduce the overhead memory, ending up with $M_{\text{per core}}\big(LN(d(5 + n^{EA} + n^{\xi}) + 3)\big)$ with $M_{\text{per core}} = \frac{M}{n_{\text{cores}}}$.  The final storage of $A$ and $\vec{b}$ only needs $n^{EA}(n^{EA} + 1) + n^{\xi}(n^{\xi} + 1)$.  

The computational cost of for solving the system $A^{EA}\vec{\alpha}^{EA} = \vec{b}^{EA}$ is $MLN^2 + MLd(n^{EA})^2 + (n^{EA})^2\log(n^{EA})$, with $MLN^2$ for computing pairwise data, $MLd(n^{EA})^2$ for constructing the learning matrix and right hand side vector, and $(n^{EA})^2\log(n^{EA})$ for solving the linear system.  In the case of choosing the optimal number of basis functions, i.e., $n^{EA}_{\text{opt}} = M^{\frac{|V|}{2s + |V|}}$, we end up with the total computational cost, of the order $M^{1 + \frac{2|V|}{2s + |V|}}$, which is slightly super linear in $M$ , but less than quadratic in the common situation when $2s+|V|\ge 2|V|$.

Similar analysis on solving $A^{\xi}\vec{\alpha}^{\xi} = \vec{b}^{\xi}$ also shows that the computational cost is slightly super linear in $M$.

\section{Learning theory}\label{sec:LT}
Given estimators, how to best measure the estimation error? This is the first question to address in order to have a full understanding of the performance. In earlier works, a set of probability measures that are adapted to the dynamical system and learning setting are introduced to describe how close  the estimators  are to the true interaction kernels.  We will generalize these ideas and measures to our learning problem. 

\subsection{Probability measures} \label{sec:PerfMeasures}
The variables we are working with have a natural distribution on the space of pairwise distances and features. Together with the fact that our functions have as arguments the variables ($r,\dot{r},\topwde,\topwda,\topwdxi$) which are functions of the state space, we are led to consider probability measures which account for the distribution of the data, while respecting the interaction structure of the system. For further intuition into these measures, see \cite{Tang2019,lu2019nonparametric,BFHM17}.
For each interacting pair $(k,k')$, we introduce the following  probability measures,

\begin{equation}\label{meas:fullEA}
\begin{dcases}
\rho_T^{EA, \idxcl, \idxcl'}(r, \topwde, \dot{r}, \topwda) 
&:= \mathbb{E}_{\bY_0 \sim \bmu}\frac{1}{TN_{\idxcl \idxcl'}}\int_{t = 0}^T\sum_{\substack{i \in \cl_{\idxcl}, i' \in \cl_{\idxcl'} \\ i \neq i'}}\delta_{i i', t}(r, \topwde, \dot{r}, \topwda)\, dt \\
\rho_T^{EA, L, \idxcl,\idxcl'}(r, \topwde, \dot{r}, \topwda) &:= \mathbb{E}_{\bY_0 \sim \bmu}\frac{1}{LN_{\idxcl\idxcl'}}\sum_{l = 1}^L\sum_{\substack{i \in \cl_{\idxcl}, i' \in \cl_{\idxcl'} \\ i \neq i'}}\delta_{i i', t_l}(r, \topwde, \dot{r}, \topwda) \\
\rho_T^{EA, L, M, \idxcl, \idxcl'}(r, \topwde, \dot{r}, \topwda) 
&:= \frac{1}{MLN_{\idxcl \idxcl'}}\sum_{l, m = 1}^{L, M}\sum_{\substack{i \in \cl_{\idxcl}, i' \in \cl_{\idxcl'} \\ i \neq i'}}\delta_{i i', t_l, m}(r, \topwde, \dot{r}, \topwda)
\end{dcases}
\end{equation}
where $N_{kk'} = N_kN_{k'}$ for $k\neq k'$ and $N_{k k^{\prime}}={N_{k} \choose 2}$ for $k=k'$, and the dirac measures are defined as
\[
\begin{dcases}
\delta_{i i', t}(r, \topwde, \dot{r}, \topwda) &:= \delta_{r_{i i'}(t), \topwde_{i i'}(t), \dot{r}_{i i'}(t), \topwda_{i i'}(t)}(r, \topwde, \dot{r}, \topwda) \\
\delta_{i i', t, m}(r, \topwde, \dot{r}, \topwda) &:= \delta_{r_{i i'}^{(m)}(t), \topwdem_{i i'}(t), \dot{r}_{i i'}^{(m)}(t), \topwdam_{i i'}(t)}(r, \topwde, \dot{r}, \topwda). 
\end{dcases}
\] We use a superscript $(m)$ to denote that the variable is calculated from the data from that $m^{\text{th}}$ trajectory. For example, $\topwdem_{i i'}(t)$ denotes the feature maps between agents $i$ and $i'$ at time $t$ along the $m^{\text{th}}$ trajectory. 

The measure $\rho_T^{EA, L, \idxcl, \idxcl'}$ is  the discrete counterpart of $\rho_T^{EA, \idxcl, \idxcl'}$ at the observation time instances. In practice, we can use $\rho_T^{EA, L, M, \idxcl, \idxcl'}$ to approximate $\rho_T^{EA, L, \idxcl, \idxcl'}$ since it can be computed from observational data and will converge to $\rho_T^{EA, L, \idxcl, \idxcl'}$ as $M \rightarrow \infty$. 

We also consider the marginal distributions 
\begin{equation}\label{meas:rhoE}
\rho_T^{E, \idxcl, \idxcl'}(r, \topwde)       := \int_{\dot{r}}\int_{\topwda} \rho_T^{EA, \idxcl, \idxcl'} \, d\topwda \, d\dot{r}\quad,\quad
\rho_T^{A, \idxcl, \idxcl'}(r, \dot{r}, \topwda)       := \int_{\topwde} \rho_T^{EA, \idxcl, \idxcl'} \, d\topwde
\end{equation}
and $\rho_T^{E, L, \idxcl, \idxcl'}(r, \topwde) $, $\rho_T^{E, L, M, \idxcl, \idxcl'}(r, \topwde)$, $\rho_T^{A, L, \idxcl, \idxcl'}(r, \dot{r}, \topwda)$,  $\rho_T^{A, L, M, \idxcl, \idxcl'}(r, \dot{r}, \topwda)$ defined analogously as above.
The empirical measures, $\rho_T^{E, L, M, \idxcl, \idxcl'}, \rho_T^{A, L, M, \idxcl ,\idxcl'}$, are the ones used in the actual algorithm to quantify the learning performances of the estimators $\lintkernele_{\idxcl \idxcl'}$ and $\lintkernela_{\idxcl \idxcl'}$ respectively.  They are also crucial in discussing the separability of $\lintkernele_{\idxcl \idxcl'}$ and $\lintkernela_{\idxcl \idxcl'}$.

For ease of notation, we introduce the following measures to handle the heterogeneity of the system, and which are used to describe error over all of the pairs $(k,k')$. 
\begin{equation} \label{meas:vectorized}
\brhoEAL=\bigoplus_{k, k^{\prime}=1,1}^{K, K} \rho_{T}^{EA,L, k k^{\prime}}, \quad
\brhoEA=\bigoplus_{k, k^{\prime}=1,1}^{K, K} \rho_{T}^{EA,k k^{\prime}}, \quad \LtwoB\left(\brhoEAL\right)=\bigoplus_{k, k^{\prime}=1,1}^{K, K} \Ltwo\left(\rho_{T}^{EA,L, k k^{\prime}}\right)
\end{equation} 

Similar definitions apply for measures related to learning the $\xi$-based interaction kernels, see Appendix \ref{sec:pm_xi}.
We discuss some key properties of the measures in Appendix \ref{sec:app:analytical}.

\subsection{Learning performance} 
We now discuss the performance measures for the estimated interaction kernels. We have already treated the trajectory estimation error in section \ref{subsec:trajperf}. 
We use weighted $L^2$-norms (with mild abuse of notation, we omit the weight from the notation) based on the dynamics-adapted measures introduced above, and with analogous definitions when discrete over $L$:
\begin{eqnarray}
\norm{\lintkernele_{\idxcl \idxcl'} - \intkernele_{\idxcl \idxcl'}}_{L^2(\rho_T^{E, \idxcl, \idxcl'})}^2 &:=& \int_{(r,\topwde)}(\lintkernele_{\idxcl \idxcl'}(r, \topwde) - \intkernele_{\idxcl \idxcl'}(r, \topwde))^2r^2 \, d\rho_T^{E, \idxcl, \idxcl'} \nonumber\\
\norm{\lintkernela_{\idxcl \idxcl'} - \intkernela_{\idxcl \idxcl'}}_{L^2(\rho_T^{A, \idxcl, \idxcl'})}^2 &:=& \int_{(r,\dot{r},\topwda)}(\lintkernele_{\idxcl \idxcl'}(r, \dot{r}, \topwda) - \intkernele_{\idxcl \idxcl'}(r, \dot{r}, \topwda))^2\dot{r}^2 \, d\rho_T^{A, \idxcl, \idxcl'} \nonumber \\
\norm{\lintkernel_{\idxcl \idxcl'}^{EA}  - \intkernel_{\idxcl \idxcl'}^{EA}}_{L^2(\rho_T^{EA, \idxcl, \idxcl'})}^2
&:=& \int_{r,\topwde,\dot{r},\topwda} \Big[(\lintkernele_{\idxcl \idxcl'}(r, \topwde) - \intkernele_{\idxcl \idxcl'}(r, \topwde))r \nonumber \\
&  & \qquad + (\lintkernele_{\idxcl \idxcl'}(r, \dot{r}, \topwda) - \intkernele_{\idxcl \idxcl'}(r, \dot{r}, \topwda))\dot{r}\Big]^2 \, d\rho_T^{EA, \idxcl, \idxcl'}\,.\label{e:L2NormDefs}
\end{eqnarray}

Our learning theory focuses on minimizing the difference between $\lintkernele_{\idxcl \idxcl'} \oplus \lintkernela_{\idxcl \idxcl'}$ and $\intkernele_{\idxcl \idxcl'} \oplus \intkernela_{\idxcl \idxcl'}$ in the joint norm given by \eqref{e:L2NormDefs}.  As long as the joint norm is small, our estimators  produce faithful approximations of the right hand side function of the original system and  trajectories. However, it does not necessarily imply that both $\lintkernele_{\idxcl \idxcl'} - \intkernele_{\idxcl \idxcl'}$'s and $\lintkernela_{\idxcl \idxcl'} - \intkernela_{\idxcl \idxcl'}$'s are small in their corresponding energy- and alignment-based norms, since  the joint norm is a weaker norm. It would be interesting to study if there is any equivalence between these two norms, but the problem appears to be quite delicate.  The theoretical investigation is still ongoing.

Now, we have all the tools needed to establish a theoretical framework: dynamics induced probability measures, performance measurements in appropriate norms, and loss functionals. These will allow us to discuss the convergence properties of our estimators.  Full details of the numerical algorithm are given in Appendix \ref{sec:algorithm_numeric}.
\subsection*{Notational summary}
A summary of the learning theory notation introduced in sections \ref{sec:notation}, \ref{sec:learn_framework}, and the notation above, is given below in table \ref{tab:2ndOrderlearning_def}. 
\begin{table}[h] 
\centering
\small{
\small{\begin{tabular}{ c | c | c}
\hline
Notation                    & Definition & Ref \\
\hline\hline
$M$   & number of trajectories & Sec. \ref{sec:intro} \\
\hline
$L$ & number of times in $[0,T]$ for each trajectory & Sec. \ref{sec:ModelDesc} \\
\hline
$\bY(t)$ & full state space vector containing $\bX(t),\bV(t),\bXi(t)$& Sec. \ref{sec:ModelDesc}  \\
\hline
$\wildcard$  & wildcard, means the notation applies for all $3$ variables &  Sec. \ref{sec:ModelDesc} \\
\hline 
 $\| \bX \|_{\mathcal{S}}$   &    $\sum_{i = 1}^N\frac{1}{{N_{\clof_i}}}\norm{\bx_i}^2$ &  \eqref{eq:sys_var_norm}\\
\hline 
$\bmu$ & distribution on the initial conditions $\bY(0)$ & Sec. \ref{sec:functionspaces}\\
\hline
$\bintkernel^{\wildcard}=(\intkernel_{kk'}^{\wildcard})_{k,k'}$ &  vectorized true $E,A,\xi$ interaction kernels & \eqref{not:vectorized} \\ 

\hline
$\bintkernelvar^{\wildcard}=(\intkernelvar_{kk'}^{\wildcard})_{k,k'}$ &  $\bintkernelvar^{\wildcard} \in \bhypspace^{\wildcard}$ with $\intkernelvar_{kk'}^{\wildcard} \in \mathcal{H}_{kk'}^{\wildcard}$ & \eqref{not:vectorized} \\
 \hline
$EA$ & shorthand denoting energy and alignment part of system & \eqref{not:shorthand} \\
\hline 
$\bpEA$  & represents the joint function $\bintkernele\oplus\bintkernela\in \bhypspaceEA$ & \eqref{not:shorthand} \\
 \hline 
$ \norm{\bY}_{\mY}^2$ & $\norm{\bX}_{\mS}^2 + \norm{\bV}_{\mS}^2 + \norm{\bXi}_{\mS}^2.$ & \eqref{eq:y_norm}\\
\hline
$\textbf{S}^{\wildcard}$ &  $\prod_{k,k'}\mathbb{S}_{kk'}^{\wildcard}$ & \eqref{eq:compactsets}\\ 
\hline 
$\PairHypSpace$  & $ \prod_{k,k'}[R_{kk'}^{\min},R_{kk'}^{\max}]$ & \eqref{eq:Rdef} \\
\hline
$R$ &  $\max_{k,k'}R_{kk'}^{\max}$ & \eqref{eq:Rdef}\\
\hline 
$\boldsymbol{\mathcal{K}}_{S_{\wildcard}}^{\wildcard}$ & admissible spaces for the $E,A,\xi$ kernels & \eqref{eq:E-Admissibility} \\
\hline
$\hypspace_{kk'}^{\wildcard}$ &  the hypothesis spaces for $\intkernel_{kk'}^{\wildcard}$ & \eqref{not:hypspaces} \\
\hline
$\bhypspace^{\wildcard}=\oplus_{kk'} \hypspace_{kk'}^{\wildcard}$ &  the hypothesis spaces for $\bintkernel^{\wildcard}$ & \eqref{not:hypspaces} \\
\hline
$\bhypspaceEA$  & direct sum of hypothesis spaces $\bhypspace^E\oplus \bhypspace^A$ & \eqref{not:bhypspaceEA}\\
\hline
$\mbf{\mathcal{E}}_{M}^{EA}(\cdot),\mbf{\mathcal{E}}_{M}^{\xi}(\cdot)$ &  empirical $EA$ error functional, $\xi$ error functional & Sec. \ref{sec:lossfunctionals} \\
\hline
$\widehat\bintkernel_{M}^{EA}:=\widehat\bintkernel_{L,M,\bhypspaceEA}^{EA}$ & $\mathrm{argmin}_{\bvpEA \in \bhypspaceEA} \bm{\mathcal{E}}_{M}^{EA}(\bvpEA)$ & \eqref{not:minimizers} \\
\hline 
$\widehat\bintkernel_{M}^{\xi}:=\widehat\bintkernel_{L,M,\bhypspacexi}^{\xi}$ & $\mathrm{argmin}_{\bvxi \in \bhypspacexi} \bm{\mathcal{E}}_{M}^{\xi}(\bvxi)$ & \eqref{not:minimizers} \\
\hline
$\{\psi_{kk',p}^{\wildcard}\}_{p=1}^{n_{kk'}^{\wildcard}}$ & basis for  $\hypspace_{kk'}^{\wildcard}$ & Sec. \ref{sec:hypspaceandalg}\\
\hline
$\brhoEA$, $\brhoxi$  & measure for $EA$, $\xi$ with continuous time, infinite trajectories & \eqref{meas:vectorized}, \ref{sec:pm_xi}  \\
\hline 
$\brhoEAL$, $\brhoxiL$ & measure for $EA$, $\xi$ discrete in time, infinite trajectories & \eqref{meas:vectorized}, \ref{sec:pm_xi} \\
\hline 
$\LtwoB\left(\brhoEAL\right)$ & $\bigoplus_{k, k^{\prime}=1,1}^{K, K} \Ltwo\left(\rho_{T}^{EA,L, k k^{\prime}}\right)$ & \eqref{meas:vectorized} \\
\hline 
$c_{\bhypspaceEA}$, $c_{\bhypspacexi}$ & coercivity constant on the $\bhypspaceEA$, $\bhypspacexi$ hypothesis spaces &  Def.\ref{def_coercivity_2ndorder}  \\
\hline
$\bhypspace_M^{EA}$,$\bhypspace_M^{\xi}$ & hypothesis spaces on $EA,\xi$ depending on $M$ & Sec. \ref{sec:main_results}\\
\hline
$A^{EA}$, $A^{\xi}$  & Learning matrices for the inverse problem & Sec. \ref{sec:numerics} \\
\hline
$\mathcal{N}(\bhypspace,\delta)$ &  $\delta$-covering number, under the $\infty$-norm, of a set $\bhypspace$ &  \cite{Wellner}\\
\hline 
\end{tabular}}  
}
\caption{Notation used throughout the paper}
\label{tab:2ndOrderlearning_def} 
\end{table}

\section{Identifiability of kernels from data} \label{sec:coercivity}
In this section we introduce a technical condition on the dynamical system that relates to the solvability of the inverse problem and plays a key role in the learning theory. We establish theorems in two directions:
\begin{enumerate}
\item Showing how identifiability of the kernels can be derived from the coercivity condition by relating the coercivity constant to the singular values of the learning matrices associated to our inverse problem, for both finitely and infinitely many trajectories.
\item Establishing the coercivity condition for a wide class of dynamical systems of the form \eqref{eq:2ndOrder}, under mild assumptions on the distribution $\bmu$ of the initial conditions. From our numerical experiments, we expect that coercivity holds even more generally. 
\end{enumerate}

We will make the following assumptions on the hypothesis spaces used in the learning approach for the remainder of the paper:
\begin{assumption}\label{compactnesscondition}
 $\bhypspaceEA$ is a compact convex subset of $\boldsymbol{\mathcal{K}}_{S_{EA}}^{EA}:= \boldsymbol{\mathcal{K}}_{S_E}^E \oplus \boldsymbol{\mathcal{K}}_{S_A}^A$ (see \ref{eq:dsum-admissibility}) which implies that the infinity norm of all elements in $\bhypspaceEA$ is bounded above by a constant $S_{EA}\geq \max\{{S_E, S_A}\}$.
\end{assumption}
\begin{assumption}\label{compactnessconditionxi}
 $\bhypspacexi$ is a compact convex subset of $\boldsymbol{\mathcal{K}}_{S_{\xi}}^{\xi}$ (see \ref{eq:E-Admissibility}) and is bounded above by $S_{0}\geq S_{\xi}$.
\end{assumption}
It is easy to see that $\bhypspaceEA$ can be naturally embedded as a compact subset of $\LtwoB(\brhoEAL)$ and that $\bhypspacexi$ can be naturally embedded as a compact subset of $\LtwoB(\bm{\rho}_T^{\xi,L})$ (recall these measures are defined in section \ref{sec:PerfMeasures}). Assumptions \ref{compactnesscondition}, \ref{compactnessconditionxi} ensure the existence of minimizers to the loss functionals $\mbf{\mE}_{M}^{EA}, \mbf{\mE}_{M}^{\xi}$ defined in \eqref{eq:Error_Func_Empirical:EA} and \eqref{eq:Error_Func_Empirical:xi}, which will be proven in Appendix \ref{sec:app:learntechnical}. 

In order to ensure learnability we introduce a coercivity condition, with terminology coming from the Lax-Milgram theorem. 
In the second-order case we will have two coercivity conditions, one for the energy and alignment, and the other for the $\xi$ variable. These conditions serve the same purpose in both cases. They first ensure that the minimizers to the error functionals are unique, and second that when the expected error functional is small, then the distance from the estimator to the true kernels is small in the appropriate $\bm\rho_T$ norm. 

Due to their connection to the error functional and the learnability of the kernels, coercivity plays an important role in the theorems of Section \ref{sec:main_results}.

\begin{definition}[Coercivity condition] \label{def_coercivity_2ndorder}
For the dynamical system \eqref{eq:2ndOrder} observed at time instants $0=t_1<t_2<\dots <t_L=T$ and with initial condition distributed $\bmu$ on $\R^{(2d+1)N}$, it satisfies the coercivity condition on the hypothesis space $\bhypspaceEA$ with constant $c_{\bhypspaceEA}$ if 
\begin{align}\label{2ndorder:gencoer}
 c_{\bhypspaceEA}  := \inf_{ \combfbvpEbvpA \in \bhypspaceEA \backslash \{\mbf 0  \}} \,\frac{ \frac{1}{L}\sum_{l=1}^{L}\E_{\bmu} \bigg[ \big\|  \rhsfo_{\combfbvpEbvpA}(\coordsTL) \big\|_{\mathcal{S}}^2\bigg] }{\|\combfbvpEbvpA\|_{\LtwoB(\brhoEAL)}^2} >0 .
\end{align} 
An analogous definition holds for continuous observations on the time interval $[0,T]$, by replacing the sum over observations at discrete times with an integral over $[0,T]$.
Similarly, the system satisfies the coercivity condition on the hypothesis space $\bhypspacexi$ with constant $c_{\bhypspacexi}$ if  
\begin{align}\label{2ndorder:gencoer:xi}
 c_{\bhypspacexi}  := \inf_{ \bvpxi \in \bhypspacexi \backslash \{\mbf 0  \}} \,\frac{ \frac{1}{L}\sum_{l=1}^{L}\E_{\bmu} \bigg[ \big\|  \rhsfo_{\bvpxi}(\coordsTL) \big\|_{\mathcal{S}}^2\bigg] }{\| \bvpxi \|_{\LtwoB(\bm{\rho}_T^{\xi,L})}^2} >0.
\end{align} 
An analogous definition holds for continuous observations on the time interval $[0,T]$, by replacing the sum over observations at discrete times with an integral over $[0,T]$.
\end{definition}

 In the following, we prove the coercivity condition on  general compact sets of $\LtwoB([0,R],\brhoEAL)$ under suitable hypotheses. Our result is independent of $N$, which implies that the finite sample bounds of Theorem \ref{t:2ndordersystem:thm_optRate} can be dimension free -- in that the coercivity constant has no dependence on $N$. This result implies that coercivity may be a fundamental property of the dynamical system, including in the mean field regime ($N\to\infty$). 
 
\subsection{Identifiability from coercivity} \label{sec:Conditioning}

By choosing the hypothesis space to be compact and convex, we are able to show that the error functional has a unique minimizer. However, many possible bases exist that could potentially yield good performance (in terms of the error functional and $\LtwoB$ error to the true kernel). 
We want to choose a basis such that the regression matrix, $A^{EA}$    
defined in Appendix \ref{sec:algorithm_numeric}, is well-conditioned -- which will ensure that the inverse problem can be solved and thus an estimator can be learned that will have good performance (in terms of the error functional and $\LtwoB(\brhoEAL)$ error to the true kernel).
In the proposition below we establish two results in this direction. The key for both results is that the basis is chosen to be orthonormal in $\LtwoB(\brhoEAL)$, versus the naive choice of basis in the underlying direct sum of $\bm{L}^{\infty}$ spaces that the interaction kernels live in. The first result is theoretical and shows that, under appropriate assumptions on the basis, the minimal singular value of the expected regression matrix (denoted $A_{\infty}^{EA}$) equals the coercivity constant. While the second result is critical for the practical implementation as it lower bounds the minimal singular value of the regression matrix by the coercivity constant with high probability. In both cases, the numerical performance is affected by the size of the coercivity constant of $\bhypspaceEA$ and if the hypothesis space is well-chosen, then the coercivity constant will be sufficiently large and the regression matrix will be well-conditioned. 
 
To ease the notation, we introduce the bilinear functional $\dbinnerp {\cdot, \cdot}$ on $\bhypspaceEA \times \bhypspaceEA$, defined by
\begin{align}
\label{eq:bilinearFn}
 \dbinnerp {\bvpEA_1, \bvpEA_2}:= \nonumber 
 \\ & \hspace{-3cm} \frac{1}{L}\sum_{l,i=1}^{L,N}\frac{1}{N_{\clof_{i}}} \mathbb{E}_{\bmu}\bigg[\bigg\langle\sum_{i'=1}^{N} \frac{1}{N_{\clof_{i'}}} \Big( \intkernelvar_{1,\clof_i\clof_{i'}}^E(r_{ii'}(t),\topwde_{ii'})\mbf{r}_{ii'}(t) + 
\intkernelvar_{1,\clof_i\clof_{i'}}^A(r_{ii'}(t),\topwda_{ii'})\mbf{\dot{r}}_{ii'}(t) \Big)
, \nonumber \\
& \hspace{-2cm} \sum_{i'=1}^{N}\frac{1}{N_{\clof_{i'}}}\Big( \intkernelvar_{2,\clof_i\clof_{i'}}^E(r_{ii'}(t),\topwde_{ii'})\mbf{r}_{ii'}(t) + 
\intkernelvar_{2,\clof_i\clof_{i'}}^A(r_{ii'}(t),\topwda_{ii'})\mbf{\dot{r}}_{ii'}(t) \Big) \bigg\rangle \bigg] 
\end{align} for any 
$\bvpEA_1=(\combf{\intkernelvar_{1,kk'}^E}{\intkernelvar_{1,kk'}^A})_{k,k'=1,1}^{K,K} \in \bhypspaceEA$, and 
$\bvpEA_2=(\combf{\intkernelvar_{2,kk'}^E}{\intkernelvar_{2,kk'}^A})_{k,k'=1,1}^{K,K} \in \bhypspaceEA$. 
For every pair $(k,k')$ let $(\combf{\psi_{kk',i}^E}{\psi_{kk',i}^A})_{i=1}^{n_{kk'}}$ 
be a basis of 
$$\hypspace_{{kk'}}^{EA} \subset L^\infty([0,R]\times \mathbb{S}_{kk'}^E) \oplus L^{\infty}([0,R]\times \mathbb{S}_{kk'}^A) $$ 
satisfying the orthonomality and boundedness conditions
\begin{align}\label{onb}
\langle \psi_{kk',p}^{EA},\psi_{kk',p'}^{EA} \rangle_{\Ltwo{(\rho_T^{EA,L,kk'})}}&=\delta_{p,p'}\quad,\quad
\|\psi_{kk',p}^{EA}\|_{\infty} &\leq S_{EA}. 
\end{align} 
We note that multivariable basis functions arise naturally in this setting due to the model. A tensor product basis of splines or piecewise polynomials can be used, as one explicit example. The $n_{kk'}$ notation allows multivariable functions, different choices for the number of basis functions across pairs $(k,k')$, and a different number of basis functions within a pair with respect to the underlying coordinates of the tensor product.

By convention, we use the lexicographic ordering to order within pairs $(k,k')$ (with order $r,\topwde, \topwda$), and then across pairs (with the lexicographic ordering on pairs of integers). 
Set $\mathbf{n}=\sum_{k,k'}n_{kk'}=dim(\bhypspaceEA)$; then for any function $\bintkernelvar^{EA} \in \bhypspaceEA$, we can write $$\bintkernelvar^{EA}=\sum_{p=1}^{\mathbf{n}} a_{p}\mbf{\psi}_{p}^{EA}.$$ 
Under the setting above, we have the following relationship between the coercivity constant and the minimal singular value of the empirical and expected learning matrix:

\begin{proposition}
\label{2ndordersystem:coercivityconstant}
Consider the matrices $$A_{\infty}^{EA}=\big(\dbinnerp{\mbf{\psi}_{p}^{EA},\mbf{\psi}_{p'}^{EA}}\big)_{p, p'} \in \mathbb{R}^{\mathbf{n} \times \mathbf{n}}\,,\quad A_{\infty}^{\xi}=\big(\dbinnerp{\mbf{\psi}_{p}^{\xi},\mbf{\psi}_{p'}^{\xi}}\big)_{p, p'} \in \mathbb{R}^{\mathbf{n}_{\xi} \times \mathbf{n}_{\xi}}\,,$$ and choose the hypothesis spaces as $\bhypspaceEA = \mathrm{span}\{ \combf{\mbf{\psi}_{p}^E}{\mbf{\psi}_{p}^A}\}_{p=1}^{\mathbf{n}}$ and $\bhypspacexi = \mathrm{span}\{ \mbf{\psi}_{p}^{\xi}\}_{p=1}^{\mathbf{n}_{\xi}}$. Then the coercivity constants for $\bhypspaceEA, \bhypspacexi$ are the smallest singular value of $A_{\infty}^{EA}$, $A_{\infty}^{\xi}$, respectively:
\begin{align}
\sigma_{\min}(A_{\infty}^{EA}) =c_{\bhypspaceEA}\, \qquad  \sigma_{\min}(A_{\infty}^{\xi}) =c_{\bhypspacexi}
\label{2ndorder:gencoer:result}
\end{align}
with $c_{\bhypspaceEA}$ defined in \eqref{2ndorder:gencoer} and $c_{\bhypspacexi}$ defined in \eqref{2ndorder:gencoer:xi}. 
Additionally, for large $M$, the smallest singular value of $A^{EA}$ satisfies the inequality
\begin{align*}
\sigma_{\min}(A^{EA}) \geq  0.8 c_{\bhypspaceEA}
\end{align*} 
with probability at least $1-2\mathbf{n}\exp\bigg(-\frac{c_{\bhypspaceEA}^2M}{100\mathbf{n}^2c_1^2+\frac{20}3\cdot c_1\cdot c_{\bhypspaceEA}\cdot\mathbf{n}} \bigg)$ with $c_1=2K^4\max\{R,R_{\dot{x}}\}^2S_{EA}^2+2$. 
Similarly, for large $M$, the smallest singular value of $A_{M}^{\xi}$ satisfies the inequality
\begin{align*}
\sigma_{\min}(A_{M}^{\xi}) \geq  0.8 c_{\bhypspacexi}
\end{align*} 
with probability at least $1-2\mathbf{n}_{\xi}\exp\bigg(-\frac{c_{\bhypspacexi}^2M}{100n^2c_1^2+\frac{20}{3}c_2\cdot c_{\bhypspacexi}\cdot n} \bigg)$ with $c_2=2K^4R_{\xi}^2S_{\xi}^2+2$. 
\label{p:sigmaminAL} 
Therefore, with high probability, a system and its associated hypothesis space satisfying the coercivity condition and $M$ sufficiently large, the inverse problem will be solvable with condition number controlled by the coercivity constant. 
\end{proposition}

\begin{proof}
We prove the result in the $EA$ case, the proof of the results about the $\xi$ part of the system is analogous. The orthonormality of the component functions given in \eqref{onb}, implies that
$\langle \mbf{\psi}_{p}^{EA}, \mbf{\psi}_{p'}^{EA}\rangle_{\LtwoB(\brhoEAL)}=\delta_{pp'}$.
Expand $\bvpEA \in \bhypspaceEA$ in this basis as $\bvpEA = \sum_{p=1}^{\mathbf{n}} a_p\mbf{\psi}_{p}^{EA}.$ Let the vector $v = (a_1, \ldots, a_{\mathbf{n}}) \in \mathbb{R}^{\mathbf{n}}$, and notice that
\begin{align*}
\frac{1}{L}\sum_{l=1}^{L}\E_{\bmu} \bigg[ \big\|  \rhsfo_{\combfbvpEbvpA}(\coordsTL) \big\|_{\mathcal{S}}^2\bigg]  =\dbinnerp {\sum_{p=1}^{\mathbf{n}} a_{p}\mbf{\psi}_{p}^{EA}, \sum_{p=1}^{\mathbf{n}} a_{p} \mbf{\psi}_{p}^{EA}} \\
= v^T A_{\infty}^{EA}v \geq \sigma_{\min}(A_{\infty}^{EA})\|v\|^2 
= \sigma_{\min}(A_{\infty}^{EA}) \big\| \bvpEA\|_{\LtwoB(\brhoEAL)}^2
\end{align*}
This lower bound is achieved by the singular vector corresponding to the singular value $\sigma_{\min}(A_{\infty}^{EA})$, so that by definition \eqref{2ndorder:gencoer} we have that $\sigma_{\min}(A_{\infty}^{EA}) =c_{\bhypspaceEA}$. \\

For the second statement, we consider the learning matrix $A^{EA}$ (defined in section \ref{sec:algorithm_numeric}), which we can also write as $A_M^{EA}$ to emphasize the dependence on $M$ as needed built, from the $M$ observed trajectories. By construction, for each $m$, $A_{\infty}^{EA}=\mathbb{E}_{\bmu}[A^{EA,(m)}]$ and $\lim_{M\rightarrow\infty} A_{M}^{EA}=A_{\infty}^{EA}$ by the Strong Law of Large Numbers. 
Next we will derive some important properties of the learning matrix that will allow us to apply the matrix Bernstein inequality (see \cite{tropp2015introduction}, Theorem 6.1.1, Corollary 6.1.2). Note that we will use the notation from this reference.
First we note an elementary matrix analysis result (see \cite{bhatia97} Problem III.6.13); for any two square matrices $A,B$, 
$$
\max_j |\sigma_j(A) - \sigma_j(B)| \leq \|A-B \|.
$$
All norms in this proof are the spectral norm, unless otherwise specified.
Thus if we get a concentration inequality of the form $\mathbb{P}_{\bmu}\{\|A_{\infty}^{EA} - A_{M}^{EA}  \| \geq t \}$ we will get the desired result relating the minimal singular values of $A_{M}^{EA}$ to $c_{\bhypspaceEA}$. First, notice that $\mathbb{E}_{\bmu}[A_{M}^{EA}]= A_{\infty}^{EA}$. 
Additionally, using the definition of the regression matrix, and our assumptions on the kernels and the dynamics, we can bound every entry by $c_1 = 2K^4\max\{R, R_{\dot{x}} \}^2S_{EA}^2 + 2$. This immediately implies the bound 
$$
\|A_{M}^{EA} - A_{\infty}^{EA}\| \leq 2\mathbf{n}c_1.
$$ 
Next, we upper bound the matrix variance statistic (in our case $Z = \sum_{k=1}^M S_k$ where $S_k = \frac{1}{M}A^{EA,(k)}$), defined as 
$$
v(Z) = \max\{ \| \mathbb{E}[(Z - \mathbb{E}Z)(Z - \mathbb{E}Z)^* \|, \| \mathbb{E}[(Z - \mathbb{E}Z)^{*}(Z - \mathbb{E}Z) \| \}.
$$
Using a similar analysis to bound each entry of the matrices, we can arrive at the result that $v(Z) \leq 2\mathbf{n}^2c_1^2$. Now, we apply the matrix Bernstein inequality to see that 
$$
\mathbb{P}_{\bmu}\Big\{\|A_{M}^{EA} - A_{\infty}^{EA} \| \geq t  \Big\} \leq 2\mathbf{n}
\exp \bigg( -\frac{c_{\bhypspaceEA}^2M}{100\mathbf{n}^2c_1^2 + \frac{20c_1c_{\bhypspaceEA}}{3}\mathbf{n} }  \bigg).
$$
Note that the $M$ in the numerator comes because $A_{M}^{EA}$ has a factor of $\frac{1}{M}$ on it.
Lastly, choose $t = \frac{c_{\bhypspaceEA}}{5}$, which together with the results above yield the desired inequality.  
\end{proof}

From Proposition \ref{p:sigmaminAL} we see that, for each hypothesis space $\hypspace_{kk'}$, it is important to choose a basis that is well-conditioned in $\LtwoB(\brhoEAL), \LtwoB(\brhoxiL)$,  instead of in the corresponding $\bm{L}^{\infty}$ spaces.  If not, the learning matrices, defined in Appendix \ref{sec:algorithm_numeric}, $A^{EA}, A^{\xi}$ may be ill-conditioned or even singular. This would lead to fundamental numerical challenges in solving for the kernels. 
%This issue can deteriorate in practice when the unknown $\smash{\rhoL}$ is replaced by the empirical measure $\smash{\brho_{T}^{L, M}}$. 
In order to mitigate these issues, one can use piecewise polynomials on a partition of the support of the empirical measure and/or use the pseudo-inverse with an adaptive tolerance.

 \subsection{Discussions on the coercivity condition}

The coercivity condition is key to the identifiability of the kernels from data. It is determined by the distribution of the solution to the agent system and introduces constraints on the hypothesis space. For the second-order system, it is therefore related to the distribution $\bmu$ of the initial conditions, the true interaction kernels, and the non-collective force. The coercivity condition has been studied for first-order systems in \cite{lu2019nonparametric,Tang2019, li2019identifiability}. Below, we give a brief review.

For homogeneous systems, \cite{lu2019nonparametric,Tang2019} showed that the coercivity condition holds true on any compact subset of the corresponding $L^2$ space for the case of $L=1$. This result has been generalized to cover heterogeneous systems in \cite{Tang2019} and a few examples of the stochastic homogeneous system including linear systems and nonlinear systems with stationary distributions for general $L$ in \cite{li2019identifiability}.

In this paper, we shall employ a similar idea as for first-order systems and extend the result to second-order systems. One key in the proof is  to show the positiveness of integral operators that arise in the expectation in Eq. \eqref{2ndorder:gencoer}. We focus on a representative model of second-order homogeneous systems,
\begin{equation}
\begin{dcases}
m_i\ddot\bx_i &= \forcev(\bx_i, \dot\bx_i, \xi_i) + \sum_{i'=1}^N \frac{1}{N} (\intkernele(\norm{\bx_{i'} - \bx_i})(\bx_{i'} - \bx_i) + \intkernela(\norm{\bx_{i'} - \bx_i})(\dot\bx_{i'} - \dot\bx_i) \\
\dot\xi_i &= \forcexi(\bx_i, \dot\bx_i, \xi_i) + \sum_{i'=1}^N \frac{1}{N} \intkernelxi(\norm{\bx_{i'} - \bx_i})(\xi_{i'} - \xi_i)
\end{dcases}\,
\end{equation}
which includes the first-order systems considered in  \cite{BFHM17,lu2019nonparametric,Tang2019} as special cases and various second-order system examples in \cite{lu2019nonparametric, Zhong20} as specific applications. We shall prove the coercivity condition holds true for the case $L=1$:  
\begin{theorem}\label{2ndordersingle:coercivity}
Consider the system \eqref{eq:2ndOrderCoercivity} at time $t_1=0$ with the initial distribution $\mu_0^{\bY}=\begin{bmatrix}\mu_0^{\bX}\\ \mu_0^{\dot\bX} \\ \mu_0^{\bXi} \end{bmatrix}$, where ${ \mu}_0^{\bX}$  is exchangeable Gaussian with $\mathrm{cov}(\bx_i(t_1))-\mathrm{cov}(\bx_i(t_1),\bx_j(t_1))=\lambda I_d$\,  for a constant $\lambda>0$,  ${ \mu}_0^{\dot\bX}, { \mu}_0^{\bXi}$  are exchangeable with finite second moment, and they are independent of ${ \mu}_0^{\bX}$.   Then 

\begin{align*}
\mathbb{E}_{\mu_0^{\bY}} \|\rhsfo_{\combf{\varphi^E}{\varphi^A}} (\bX(0), \bV(0))\|^2_{\mathcal{S}} &\geq c_{1,N, \hypspace^{EA}}\vertiii{\combf{\varphi^E}{\varphi^A}}_{L^2(\rho_T^{EA,1})},\\
\mathbb{E}_{\mu_0^{\bY}} \|\rhsfo_{\intkernelvar^{\xi}} (\bX(0), \bXi(0))\|^2_{\mathcal{S}} &\geq c_{1,N, \hypspace^{\xi}}\vertiii{\intkernelvar^{\xi}}_{L^2(\rho_T^{\xi,1})},
\end{align*} where
\begin{itemize} 

\item $c_{1,N, \hypspace^{EA}}\geq (\frac{N-1}{2N^2}+ \frac{(N-1)(N-2)}{2N^2}c ),
c=\min\bigg\{c_{\mathcal{H}^{EA}}^{E},  c_{\mathcal{H}^{EA}}^{A}c_{\mu_0^{\dot\bX}}\bigg\}$, where $c_{{\bm \mu}_0^{\dot\bX}}=1-\frac{\E \langle \dot \bx_i(0),\dot \bx_{i'}(0)\rangle}{\E \|\dot \bx_i(0)\|^2}$ ($i\neq i'$) and $c_{\mathcal{H}^{EA}}^{E}$ and $c_{\mathcal{H}^{EA}}^{A}$ are non-negative constants independent of $N$, and are strictly positive for compact  $\mathcal{H}^{EA}$ of $L^2(\rho^{EA,1}_T)$.

\item $c_{1,N, \hypspace^{\xi}}\geq (\frac{N-1}{N^2}+ \frac{(N-1)(N-2)}{N^2}c ),
c=  c_{\mathcal{H}^{\xi}}c_{\mu_0^{\bXi}} $ with $c_{{ \mu}_0^{\bXi}}=1-\frac{\E \langle  \xi_i(0), \xi_{i'}(0)\rangle}{\E \| \xi_i(0)\|^2}$ ($i\neq i'$) and  $c_{\mathcal{H}^{\xi}}$ is  a non-negative constant independent of $N$, which is strictly positive for compact  $\mathcal{H}^{\xi}$ of $L^2(\rho^{\xi,1}_T)$.

\end{itemize}
 \end{theorem}

This exhibits a particular case that even the coercivity constant is independent of the number of agents $N$. Consequently, the estimated errors of our estimators are independent of $N$, and therefore not only is the convergence rate of our estimators independent of the dimension $(2d+1)N$ of the phase space, but even the constants in front of the rate term are independent of $N$.  Our results extend those for first-order systems from \cite{lu2019nonparametric, Tang2019, Zhong20}. The empirical numerical experiments on second-order systems, already conducted in \cite{Zhong20}, support that the coercivity condition is satisfied by large classes of second-order systems, and is ``generally'' satisfied for general $L$  on relevant hypothesis spaces, with a constant independent of the number of agents $N$ thanks to the exchangeability of the distribution of the initial conditions, and of the agents at any time $t$.  The proof of the result above is given in Appendix \ref{sec:app:Coercivity}.

\section{Consistency and optimal convergence rate of estimators} \label{sec:main_results}

The final preparatory results for our main theorems combine concentration with a union bound. Here we control the probability that the supremum of the difference between the expected and empirical normalized errors over the whole hypothesis space is large. 

\subsection{Concentration}

Our first main result is a concentration estimate that relates the coercivity condition to an appropriate bias-variance tradeoff in our setting. 
Let $\mathcal{N}(\bhypspace,\delta)$ be the $\delta$-covering number, with respect to  the $\infty$-norm, of the set $\bhypspace$.
 
\begin{theorem}[Concentration]\label{2ndorder:maintheorem}
Suppose that $\bintkernel^\wildcard \in \mbf{\mathcal{K}}_{S_\wildcard}^\wildcard$. Consider a convex, compact (with respect to the $\infty$-norm) hypothesis spaces  
$$\bhypspace_M^{EA} \subset \mbf{L^\infty}(\PairHypSpace \times \FeatureHypSpaceE) \oplus \mbf{L^\infty}(\PairHypSpace \times \FeatureHypSpaceA),\quad \bhypspace_M^{\xi} \subset \mbf{L^\infty}(\PairHypSpace \times \FeatureHypSpacexi)$$ 
bounded above by $S_0 \geq \max\{S_E, S_A,S_{\xi}\}$ respectively. Additionally, assume that the coercivity condition \eqref{2ndorder:gencoer} holds on $\bhypspace_M^{EA}$ and condition \eqref{2ndorder:gencoer:xi} on $\bhypspace_M^{\xi}$.

Then for all $\epsilon >0$, with  probability (with respect to $\bmu$) at least $1-\delta$, we have the estimates 
\begin{equation}
\begin{aligned}
\label{2ndorder:mainestimate1}
c_{\bhypspace_M^{EA}}\| \bphEA_{M} - \combfbpEbpA  \|^2_{\LtwoB(\brhoEAL)} &\leq   2\inf_{\combfbvpEbvpA \in \bhypspace_M^{EA}}\|\combfbvpEbvpA-\combfbpEbpA\|^2_{\LtwoB(\brhoEAL)} +2\epsilon,\\
c_{\bhypspace_M^{\xi}}\| \widehat{\bintkernel}_{M}^{\xi} - \bpxi  \|^2_{\LtwoB(\brhoxiL)} &\leq   2\inf_{\bvxi \in \bhypspace_M^{\xi}}\|\bvxi-\bpxi\|^2_{\LtwoB(\brhoxiL)} + 2\epsilon\,,
\end{aligned}
\end{equation} 
provided that, for the first bound to hold,
$$M \geq \frac{1152S_{0}^2\max\{R, R_{\dot{x}} \}^2K^4}{\epsilon c_{\bhypspace_M^{EA}}}\bigg(\log \Big(\mathcal{N}\Big(\bhypspace^{EA}_M,\frac{\epsilon}{48S_{0}\max\{R, R_{\dot{x}} \}^2K^4} \Big)\Big)+\log\Big(\frac{1}{\delta}\Big)  \bigg)\,,$$ 
and similarly for the second inequality, using $\bhypspace_M^{\xi}$. 
\end{theorem} 

\begin{proof} [\textbf{of Theorem \rm{\ref{2ndorder:maintheorem}}}]
We start out by setting $\alpha=\frac{1}{6}$ in Proposition \ref{learning:fullconcentrationEA}, it is easy to see that this chosen value yields the tightest bound in the argument below. To ease the notation we let 
$\widehat \bintkernel_{L,M,\bhypspaceEA}^{EA} = \combf{\widehat \bintkernel_{L,M,\bhypspaceEA}^E}{\widehat \bintkernel_{L,M,\bhypspaceEA}^A}$ and similarly for $\widehat \bintkernel_{L,\infty,\bhypspaceEA}^{EA}$.  
From the Proposition, we have that 
 $$\sup_{ \combfbvpEbvpA \in \bhypspaceEA}  \frac{\mathcal{D}_{\infty}(\combfbvpEbvpA)-\mathcal{D}_{M}(\combfbvpEbvpA)}{\mathcal{D}_{\infty}(\combfbvpEbvpA)+\epsilon}  <\frac{1}{2},$$
 holds true with probability 
\begin{equation} \label{eq:usefulProbBound}
\mathcal{P}\geq 1-\mathcal{N}\bigg(\bhypspaceEA, \frac{\epsilon}{48S_{EA}\max\{R, R_{\dot{x}}\}^2K^4}\bigg) \exp \bigg(-\frac{c_{\bhypspaceEA}M\epsilon}{1152S_{EA}^2\max\{R, R_{\dot{x}}\}^2K^6}\bigg). 
\end{equation}
This immediately implies, by choosing $\combfbvpEbvpA=\widehat \bintkernel_{L,M,\bhypspaceEA}^{EA}$ and reorganizing, that with probability $\mathcal{P}$
$$  \mathcal{D}_{\infty}(\widehat \bintkernel_{L,M,\bhypspaceEA}^{EA})< 2\mathcal{D}_{M}(\widehat \bintkernel_{L,M,\bhypspaceEA}^{EA})+\epsilon\,.$$
By definition of $\widehat \bintkernel_{L,M,\bhypspaceEA}^{EA}$ as the minimizer of the empirical error functional $\bm{\mathcal{E}}_{M}^{EA}$, we see that
$$\mathcal{D}_{M}(\widehat \bintkernel_{L,M,\bhypspaceEA}^{EA})=\bm{\mathcal{E}}_{M}^{EA}(\widehat \bintkernel_{L,M,\bhypspaceEA}^{EA})-\bm{\mathcal{E}}_{M}^{EA}(\widehat \bintkernel_{L,\infty,\bhypspaceEA}^{EA}) \leq 0,$$
and combining this result with equation (\ref{eq_minH}) from Proposition \ref{2ndordersystem:convexity}, we have 
\begin{equation} \label{eq:concentrationEAintermediate}
c_{\bhypspaceEA} \|\widehat \bintkernel_{L,M,\bhypspaceEA}^{EA}-\widehat \bintkernel_{L,\infty,\bhypspaceEA}^{EA}\|_{\LtwoB( \brhoEAL)}^2 
 \leq \mathcal{D}_{\infty}(\widehat \bintkernel_{L,M,\bhypspaceEA}^{EA} )<\epsilon,
\end{equation} 
with the final inequality holding with probability $\mathcal{P}$.
Now we bound the $\LtwoB$ error of the empirical estimator to the true interaction kernel, so that with probability $\mathcal{P}$
 \begin{align*}
 \|\widehat \bintkernel_{L,M,\bhypspaceEA}^{EA} - \bintkernel^{EA}\|_{\LtwoB(\brhoEAL)}^2
 & \leq 2 \|\widehat \bintkernel_{L,M,\bhypspaceEA}^{EA}-\widehat \bintkernel_{L,\infty,\bhypspaceEA}^{EA}\|_{\LtwoB(\brhoEAL)}^2
+2\|\widehat \bintkernel_{L,\infty,\bhypspaceEA}^{EA} - \bintkernel^{EA}\|_{\LtwoB(\brhoEAL)}^2 \\
&   \leq  \frac{2}{c_{\bhypspaceEA}}\bigg( \epsilon +\inf_{\combfbvpEbvpA \in \bhypspaceEA} K^2\| \combfbvpEbvpA -\combfbpEbpA \|_{\LtwoB(\brhoEAL)}^2\bigg) \\
&  \leq  \frac{2}{c_{\bhypspaceEA}}\bigg( \epsilon +\inf_{\combfbvpEbvpA \in \bhypspaceEA} K^2\max\{R, R_{\dot{x}} \}^2\| \combfbvpEbvpA -\combfbpEbpA \|_{\infty}^2\bigg).	
\end{align*}  
The first inequality follow from the coercivity condition (\ref{2ndorder:gencoer}) and the definition of 
$\bphEA_{\infty}$. 
The second follows by the definition of the norms. 
Now for a chosen $0<\delta<1$, let $$1-\mathcal{N}\bigg(\bhypspaceEA, \frac{\epsilon}{48S_{EA}\max\{R,R_{\dot{x}} \}^2K^4 }\bigg) \exp \bigg(-\frac{c_{\bhypspaceEA}M\epsilon}{1152S_{EA}^2\max\{R,R_{\dot{x}} \}^2K^6} \bigg) \geq 1-\delta$$ and solve for $M$. The proof for the $\xi$ part of the system result is similar. 
\end{proof}

\subsection{Consistency}  \label{sec:consistency}
In the regime where $M \to \infty$, we will choose an increasing sequence of hypothesis spaces, each satisfying the conditions of Theorem \ref{2ndorder:maintheorem}. By our assumptions on the kernels, we can also choose the sequence of $\bhypspace_M^{EA}$'s such that the approximation error goes to $0$ as $M\to \infty$. 
This enables us to control the infimum on the right hand side of (\ref{2ndorder:mainestimate1}). From here we can apply Theorem \ref{2ndorder:maintheorem} on each $M$ to prove the consistency of our estimators with respect to the $\LtwoB(\brhoEAL)$ norm and derive the following consistency theorem.

\begin{theorem}[Strong Consistency]\label{main:consistency}  Suppose that 
$$\{\bhypspace_M^{EA}\}_{M=1}^{\infty} \subset \mbf{L^\infty}(\PairHypSpace \times \FeatureHypSpaceE) \oplus \mbf{L^\infty}(\PairHypSpace \times \FeatureHypSpaceA)$$ is a family of compact and convex subsets such that the approximation error goes to zero,
$$\inf_{\combfbvpEbvpA\in \bhypspace_M^{EA}}\|\combfbvpEbvpA-\combfbpEbpA\|_{\infty} \xrightarrow{M\rightarrow \infty} 0.$$
Further suppose that the coercivity condition holds on $\bigcup_{M}\bhypspace_M^{EA}$, and that $ \bigcup_{M}\bhypspace_M^{EA}$ is compact in  $\mbf{L^\infty}(\PairHypSpace \times \FeatureHypSpaceE) \oplus \mbf{L^\infty}(\PairHypSpace \times \FeatureHypSpaceA)$. Then the estimator is strongly consistent with respect to the $\LtwoB(\brhoEAL)$ norm:
$$
\lim_{M\rightarrow \infty}\|\bphEA_{M} - \combfbpEbpA\|_{\bL^2(\brhoEAL)} =0 \text{ with probability one.}
$$
An analogous consistency result holds for the estimator in the $\xi$ variable. 

\end{theorem}
These two results together provide a consistency result on the full estimation of the triple $(\widehat{\bintkernel^{\xi}}, \widehat{\bintkernel^{E}}, \widehat{\bintkernel^{A}})$ and thus consistency of our estimation procedure on the full system (\ref{eq:2ndOrder}).  \\

\begin{proof}[\textbf{of Theorem \rm{\ref{main:consistency}} }]
To simplify the notation, we use the same conventions as the proof of Theorem \ref{2ndorder:maintheorem} and let 
$\mathcal{D}_{\infty} = \mathcal{D}_{L, \infty, \bhypspaceEA_M}$. By definition of the coercivity constant in (\ref{2ndorder:gencoer}), we have the inequality $c_{\cup_{M}{\bhypspaceEA_M}} \leq c_{\bhypspaceEA_M} $. 
From an analogous argument used to arrive at equation (\ref{eq:concentrationEAintermediate}) in the proof of Theorem \ref{2ndorder:maintheorem}, we obtain that 

\begin{align}\label{importantinequality}
c_{\cup_{M}{\bhypspaceEA_M}}\|\bphEA_M - \combfbpEbpA\|_{\LtwoB(\brhoEAL)}^2
&\leq \mathcal{D}_{\infty}(\bphEA_M)+\mbf{\mathcal{E}}_{\infty}^{EA}(\bphEA_{\infty}).
\end{align} 
Let $\epsilon > 0$, the inequality \eqref{importantinequality} gives us that
\begin{align*}
P_{\bmu} \{c_{\cup_{M}{\bhypspaceEA_M}} \|\bphEA_M - \combfbpEbpA\|_{\bL^2(\brhoEAL)}^2 \geq \epsilon \}
&\leq P_{\bmu} \{\mathcal{D}_{\infty}(\bphEA_M)+\mbf{\mathcal{E}}_{\infty}^{EA}(\bphEA_{\infty}) \geq \epsilon \}\\
&\leq P_{\bmu}\bigg\{\mathcal{D}_{\infty}(\bphEA_M)\geq \frac{\epsilon}{2}\bigg\}+
P_{\bmu}\bigg\{\mbf{\mathcal{E}}_{\infty}^{EA}(\bphEA_{\infty}) \geq \frac{\epsilon}{2}\bigg\}. 
\end{align*}
We now bound the two terms in the above expression separately. For the first term, the proof of Theorem  \ref{2ndorder:maintheorem} shows that 
\begin{align*}
P_{\bmu}\{\mbf{\mathcal{D}}_{\infty}(\bphEA_M) 
\geq \frac{\epsilon}{2}\} 
&\leq \mathcal{N}\bigg(\bhypspaceEA_M,\frac{\epsilon}{C_1}\bigg)\exp\bigg(-\frac{c_{\bhypspaceEA_M}M\epsilon}{C_2}\bigg)\\
&\leq \mathcal{N}\bigg(\cup_{M}{\bhypspaceEA_M},\frac{\epsilon}{C_1}\bigg)\exp\bigg(-\frac{c_{\cup_{M}\bhypspaceEA_M}M\epsilon}{C_2}\bigg) 
%\\&\leq C(\cup_{M}\bhypspace_M, \frac{\epsilon}{2}) \exp\bigg(-\frac{c_{ \cup_{M}\bhypspaceEA_M}M\epsilon}{C_2}\bigg),
\end{align*} 
where $C_1 = 96S_{EA}^2\max\{R, R_{\dot{x}} \}^2K^4$, $C_2 = 2304S_{EA}^2\max\{R, R_{\dot{x}} \}^2K^4$, and  
$\mathcal{N}(\cup_{M}{\bhypspaceEA_M},\frac{\epsilon}{C_1})$ is finite because of the compactness assumption on $\cup_{M}\bhypspaceEA_M$.

Summing this bound in $M$ we get that, 
\begin{align*}
\sum_{M=1}^{\infty} P_{\bmu}\{\mathcal{D}_{\infty}(\bphEA_M) \geq \frac{\epsilon}{2}\}&\leq \mathcal{N}\bigg(\cup_{M}{\bhypspaceEA_M},\frac{\epsilon}{C_1}\bigg)\sum_{M=1}^{\infty}\exp\bigg(-\frac{c_{\cup_M\bhypspaceEA_M}M\epsilon}{C_2}\bigg) < \infty.
\end{align*}
For the second term, the bound  \eqref{2ndordersystem:expectationerrorfunctional} yields that 
$$\mbf{\mathcal{E}}_{\infty}^{EA}(\bphEA_{\infty}) \leq 4K^4S_{EA}\max\{R, R_{\dot{x}}\}^2 \inf_{\combfbvpEbvpA \in \bhypspaceEA_M} \|\combfbvpEbvpA-\combfbpEbpA\|_{\infty}\xrightarrow{M\rightarrow \infty} 0.$$
Since $\epsilon$ is fixed, the above result, together with our assumption on the sequence of hypothesis spaces, implies that $P_{\bmu}\Big\{\mbf{\mathcal{E}}_{\infty}^{EA}(\bphEA_{\infty}) \geq \frac{\epsilon}{2}\Big\}=0$ for $M$ sufficiently large. So we have 
$\sum_{M=1}^{\infty}P_{\bmu}\{\mbf{\mathcal{E}}_{\infty}^{EA}(\bphEA_{\infty}) \geq \frac{\epsilon}{2}\} < \infty$.
The finiteness of the two sums above implies, by the first Borel-Cantelli Lemma, that 
$$P_{\bmu}\big\{\limsup_{{M\rightarrow \infty}}\{ c_{\cup_{M}{\bhypspaceEA_M}} \|\bphEA_M-\combfbpEbpA \|_{\LtwoB(\brhoEAL)}^2  \geq \epsilon\}\big\}=0,$$

As $\epsilon$ was arbitrary, we have the desired strong consistency of the estimator. An exactly analogous argument gives the result on the $\xi$ part of the system. 
\end{proof}

\subsection{Rate of convergence} \label{sec:rateofconv}

Given data collected from $M$ trajectories, we would like to choose the best hypothesis space to maximize the accuracy of the estimators. Theorem \ref{2ndorder:maintheorem} highlights the classical bias-variance tradeoff in our setting. On the one hand, we would like the hypothesis space $\bhypspaceEA_M$ to be large so that the bias $$\inf_{\bvpEA \in \bhypspaceEA_M}\|\bvpEA -\bpEA \|^2_{\LtwoB(\brhoEAL)}\,,\text{ or }\,\,\inf_{\bvpEA \in \bhypspaceEA}\|\bvpEA-\bpEA\|^2_{\infty}\,,$$ is small. Simultaneously,  we would like $\bhypspaceEA_M$  to be small enough so that the covering number  
$\mathcal{N}(\bhypspaceEA_M,\epsilon) $
is small. Just as in nonparametric regression, our rate of convergence depends on a regularity condition of the true interaction kernels and properties of the hypothesis space, as is demonstrated in the following theorem. We establish the optimal (up to a log factor) min-max rate of convergence by choosing a hypothesis space in a sample size dependent manner.  

\subsection*{Comments}
Parts (a) and (c) of the theorem concern an approximation theory type rate of convergence where $M$ plays no role in the choice of hypothesis space, whereas parts (b) and (d) of the theorem present a minimax rate of convergence that chooses an adaptive hypothesis space depending on $M$ to achieve the optimal rate. 

The splitting of the convergence result, between $EA$ and $\xi$ parts, emphasizes a common theme of the paper: 
we leverage that the system can be decoupled for the learning process to improve the rate of convergence and performance in learning the estimators, but analytically we study it as the full coupled system, see the trajectory prediction result in Theorem \ref{secondordersystem:TrajDiff_NEW}. 

A final important comment is that even though the dimension of the space in which we measure the error may be large, namely the dimension of $\brhoEAL$ (which can be calculated as $\dim(\brhoEAL) = \sum_{kk'}p_{(k,k')}^E + \sum_{(k,k')}p_{(k,k')}^A$), we exploit the structure of the system in such a way that our convergence rate only depends on the number of unique variables across all of the estimators. This number is given by the number of variables in $\spacevariables$, denoted as $|\spacevariables|$, for the $EA$ portion and by $|\spacevariables_{\xi}|$ for the $\xi$ portion of the system. 
We note that we are not predicting the number of variables nor their form, they are assumed known, this applies to both the pairwise interaction variables and the feature maps.

\begin{theorem}[Rate of Convergence] \label{main:convrate}   
Let $\bphEA:=\combf{\widehat\bintkernel_M^E}{\widehat\bintkernel_M^A}$
denote the minimizer of the empirical error functional $\bm{\mathcal{E}}_{M}^{EA}$ (defined in
\eqref{eq:Error_Func_Empirical:EA}) over the hypothesis space $\bhypspaceEA_M$.  \\ 
(a)  Let the hypothesis space be chosen as the direct sum of the admissible spaces, namely $\bhypspaceEA = \combf{\mbf{\mathcal{K}}_{S_E}^E}{\mbf{\mathcal{K}}_{S_A}^A},$ and assume that the coercivity condition \eqref{2ndorder:gencoer} holds true on it.

Then, there exists a constant $C$ depending only on $K,S_{EA},R, R_{\dot{x}}$ such that 
 \[ 
 \mathbb{E}_{\bmu}\Big[\| \bphEA_M-\bpEA\|^2_{\LtwoB(\brhoEAL)} \Big]\leq \frac{C}{c_{\bhypspaceEA}} M^{-\frac{1}{|\spacevariables|+1}}.
 \]
(b) Assume that $\{\mbf{\mathcal{L}}_n\}_{n=1}^{\infty}$  is a sequence of finite-dimensional linear subspaces of  $\bL^{\infty}(\PairHypSpace\times \FeatureHypSpaceE) \oplus \bL^{\infty}(\PairHypSpace \times \FeatureHypSpaceA)$  satisfying the dimension and approximation constraints
 \begin{align}\label{assumptions}
\text{dim}(\mbf{\mathcal{L}}_n) \leq c_0K^2n^{|\spacevariables|}\,,\quad \inf_{\bvpEA \in \mbf{\mathcal{L}}_n}\|\bvpEA-\bpEA\|_{\infty}  \leq c_1 n^{-s},
\end{align}
for some fixed constants $c_0,c_1$ representing dimension-independent approximation characteristics of the linear subspaces, and $s>0$ related to the regularity of the kernels. The value $n$ can be thought of as the number of basis functions along each of the $|\spacevariables|$ axes. Suppose the coercivity condition holds true on the set $\cup_n\mbf{\mathcal{L}}_n$.  Define $\mbf{\mathcal{B}}_n$ to be the closed ball centered at the origin of radius $(c_1+S_{EA})$ in $\mbf{\mathcal{L}}_{n}$.  Let $c_M^{EA} := c_{\bigcup_M \bhypspaceEA_M}$. 
If we choose the hypothesis space as $\bhypspace_M=\mbf{\mathcal{B}}_{(\frac{M}{\log M})^{\frac{1}{2s+|\spacevariables|}}}$, then there exists a constant $C$ depending on $K,R, R_{\dot{x}}, S_{EA},c_0, c_1, s$ such that we achieve the convergence rate,
\begin{align}\label{rate}
 \mathbb{E}_{\bmu}\Big[\| 
\bphEA_{M} -\bpEA \|^2_{\LtwoB(\brhoEAL)} \Big] \leq \frac{C}{c_M^{EA}} \left(\frac{\log M}{M}\right)^{\frac{2s}{2s+|\spacevariables|}}\,. 
 \end{align}
(c) under the corresponding assumptions as in (a), there exists a constant $C$ depending only on $K,S_{\xi},R$ such that 
 \[ 
 \mathbb{E}_{\bmu}\Big[\|\widehat\bintkernel_M^{\xi}-\bintkernel^{\xi}\|^2_{\LtwoB(\brhoxiL)} \Big]\leq \frac{C}{c_{\bhypspacexi}} M^{-\frac{1}{|\spacevariables_{\xi}|+1}}.
 \]
(d) under the corresponding assumptions as in (b), there exists a constant $C$ depending only on $K,R, S_{\xi},c_0, c_1, s$ such that, and for $c_M^{\xi}:=c_{\bigcup_M \bhypspacexi_M}$,
\begin{align}\label{rate:xi}
 \mathbb{E}_{\bmu}\Big[\| 
\widehat\bintkernel^{\xi}_{M} 
 -\bintkernel^{\xi} \|^2_{\LtwoB(\brhoxiL)} \Big] \leq \frac{C}{c_M^{\xi}} \left(\frac{\log M}{M}\right)^{\frac{2s}{2s+|\spacevariables_{\xi}|}}\,. \end{align}
 \label{t:2ndordersystem:thm_optRate}
\end{theorem}
We in fact prove bounds not only in expectation, but also with high probability, for every fixed large-enough $M$, as the proof will show. \\

\begin{proof}[\textbf{of Theorem \rm{\ref{main:convrate}}}]
For part (a), let $\bhypspace=\combf{\mbf{\mathcal{K}}_{S_E}^E}{\mbf{\mathcal{K}}_{S_A}^A}$. 
Standard results on covering numbers of function spaces (see 
theorem 2.7.1 of \cite{Wellner}) give us that the covering number of $\bm{\mathcal{H}}$ satisfies
$$ 
\mathcal{N}(\bhypspace, \epsilon, \|\cdot \|_{\infty}) \leq C_{\bhypspace}\exp\bigg(K^2\bigg(\frac{1}{\epsilon}\bigg)^{|\spacevariables|}\bigg)
$$
for some absolute constant $C_{\bhypspace}$ depending only on $\bhypspace$ and $|\spacevariables|$. By assumption on the hypothesis space, we have that 
$$ 
\inf_{\bvpEA \in \bhypspace}\|\bvpEA-\bpEA\|^2_{\infty}=0.
$$ 
From this, the concentration estimate \eqref{2ndorder:mainestimate1} together with the covering number bound imply that, 
\begin{align}
P_{\bmu}\{\| 
\bphEA_{L,M,\bhypspace}
- \bpEA \|_{\LtwoB(\brhoEAL)}^2
> \epsilon \} 
&\leq  \mathcal{N}(\bhypspace, C_1 \epsilon, \|\cdot \|_{\infty})\exp(-C_2M\epsilon) \nonumber \\
& \leq C_{\bhypspace} \exp( K^2(C_1\epsilon)^{-|\spacevariables|}  - C_2M\epsilon) \label{eq:maintheorem_intResult}
\end{align}
where $C_1 = \frac{c_{\bhypspace}}{48S_{EA}\max\{R, R_{\dot{x}} \}^2K^4}$ and $C_2  = \frac{c_{\bhypspace}}{1152S_{EA}^2\max\{R, R_{\dot{x}} \}^2K^4}$.
Next, define the function  
$$g(\epsilon) := K^2(C_1\epsilon)^{-|\spacevariables|} - \frac{C_2M\epsilon}{2},$$ which we will minimize to achieve the desired probability bound. 

By direct calculation, $g(\epsilon)=0$ if we choose $\epsilon=\epsilon_{M}=(\frac{C_3}{M})^{\frac{1}{|\spacevariables|+1}}$, where $C_3=\Big(\frac{2K^2}{C_2C_1^{|\spacevariables|}}\Big)^{\frac{1}{|\spacevariables|+1}}$.  It is then an easy computation to see that the derivative of $g(\epsilon)$ is $\leq 0$ for all $\epsilon \geq \epsilon_M$. Thus, we can put this result into the bound \eqref{eq:maintheorem_intResult} to arrive at the probability bound, 
\begin{equation} \label{eq:maintheorem_intResult2}
P_{\bmu}\{\| 
\bphEA_{L,M,\bhypspace} -
\bpEA \|_{\LtwoB(\brhoEAL)}^2
> \epsilon \} \leq  \begin{cases} \exp\Big(\frac{-C_2M\epsilon}{2}\Big), &  \epsilon \geq \epsilon_M \\ 1,  &  \epsilon \leq \epsilon_M \end{cases}
\end{equation}
Integrating this bound over $\epsilon\in(0,+\infty)$ and using the elementary inequality $e^{-x}\leq x+1$ for all $x\geq 0$, we get that  
$$ 
\int_0^{\infty} P_{\bmu}\{\| 
\bphEA_{L,M,\bhypspace} -
\bpEA \|_{\LtwoB(\brhoEAL)}^2
> \epsilon \}  d\epsilon \leq \Big(\frac{C_4}{M}\Big)^{\frac{1}{|\spacevariables|+1}} + O\Big(\frac{1}{M}\Big)
$$ 
Now, bringing the coercivity part from \eqref{2ndorder:mainestimate1} back in, we achieve the rate, 
$$ 
\mathbb{E}_{\bmu}[\| 
\bphEA_{L,M,\bhypspace}- 
\bpEA\|_{\LtwoB(\brhoEAL)}^2]  \le \frac{C_4}{c_{\bhypspace}} M^{-\frac{1}{|\spacevariables|+1}}, $$ where $C_4$ is an absolute constant that only depends on $K, S_{EA}, R, R_{\dot{x}}$. 

For part (b), we note the following basic result on the covering number of $\mbf{\mathcal{B}}_n$  by $\epsilon$-balls (see \cite[Proposition 5]{CS02}),
$$
\mathcal{N}(\bm{\mathcal{B}}_n, \epsilon, \|\cdot \|_{\infty}) \leq \bigg(\frac{4(c_1+S_{EA})}{\epsilon}\bigg)^{c_0K^2n^{|\spacevariables|}}.
$$
Using \eqref{2ndorder:mainestimate1}, and the approximation assumption, we bound the probability as
\begin{equation}
 \begin{aligned}
P_{\bmu}\{\| 
\bphEA_{L,M,\bm{\mathcal{B}}_n}
- \bpEA\|_{\LtwoB(\brhoEAL)}^2
> \epsilon + c_2n^{-2s} \} \\
& \hspace{-7cm} = P_{\bmu}\{\| 
\bphEA_{L,M,\bm{\mathcal{B}}_n} - \bpEA\|_{\LtwoB(\brhoEAL)}^2
> t'n^{-2s} + c_2n^{-2s} \} \\
& \hspace{-7cm} = P_{\bmu}\{\| 
\bphEA_{L,M,\bm{\mathcal{B}}_n} - \bpEA \|_{\LtwoB(\brhoEAL)}^2
> tn^{-2s} \} \\
& \hspace{-7cm} \leq \mathcal{N}\Big(\bm{\mathcal{B}}_n, c_3'tn^{-2s} , \|\cdot \|_{\infty}\Big)\exp(-c_4Mtn^{-2s}) \\
& \hspace{-7cm} \leq \Big(\frac{c_3}{tn^{-2s}}\Big)^{c_0K^2n^{|\spacevariables|}} \exp(-c_4Mtn^{-2s})\\
& \hspace{-7cm} \leq \exp(c_0K^2n^{|\spacevariables|}\log(c_3)+c_0K^2n^{|\spacevariables|}|\log(tn^{-2s})|-c_4Mtn^{-2s}), 
\label{est2}
\end{aligned}
\end{equation}
where $c_2=\frac{1}{c_{\cup_n\mbf{\mathcal{L}}_n}}c_1$, $ c_3' = \frac{c_{\cup_n\mbf{\mathcal{L}}_n}}{48(S_{EA}+c_1)\max\{R, R_{\dot{x}}\}^2K^4}$, $c_3=\frac{192 (S_{EA}+c_1)^2\max\{R, R_{\dot{x}}\}^2K^4}{c_{\cup_n\mbf{\mathcal{L}}_n}}$, and $c_4=\frac{c_{\cup_n\mbf{\mathcal{L}}_n}}{1152(S_{EA}+c_1)^2\max\{R, R_{\dot{x}}\}^2K^4}$ are absolute constants independent of $M$.  
Define 
$$
g(n) :=c_0n^{|\spacevariables|}K^2\log(c_3)+c_0n^{|\spacevariables|}K^2|\log(tn^{-2s})|-\frac{c_4}{2}Mtn^{-2s}.
$$ 
To find the optimal $n$ in terms of $M$, we minimize $g$ in $n$. By taking a derivative, and solving the corresponding equation, we see that the optimal $n$ is 
$$n_* = O\bigg(\Big(\frac{M}{\log M}\Big)^{\frac{1}{2s+|\spacevariables|}}\bigg),$$ with constant independent of $M$ and only depending on $c_3, c_4, c_2$.
For convenience we will choose $n_*= \lfloor(\frac{M}{\log M})^{\frac{1}{2s+|\spacevariables|}}\rfloor$. Now let $\epsilon_M = (\frac{M}{\log M})^{\frac{2s}{2s+|\spacevariables|}}$ and  consider 
$$h(\epsilon) = c_0n_*K^2\log(c_3) + c_0n_*K^2|\log(\epsilon)| - \frac{c_4}{2}M\epsilon.$$ 
As before, let $\epsilon = tn_*^{-2s} = t\epsilon_M$ and consider $h(t\epsilon_M)$. It is easy to see that $\lim_{t \rightarrow 0^+} h(t\epsilon_M) = \infty$ and $\lim_{t \rightarrow \infty}h(t \epsilon_M)=-\infty$. Together with the continuity of $h$, these facts imply that there exists a constant $c_5$, depending on $K,c_0, c_2,c_3,c_4$ such that $h(c_5\epsilon_M) = 0$. 
We further need that $h'(\epsilon) \leq 0$ for all $\epsilon \geq c_5\epsilon_M$. By taking the derivative of $h$, setting it $\leq 0$, we find that this condition eventually holds by basic calculus on $h$. 
Therefore, if needed to satisfy the derivative condition, we can enlarge the constant $c_5$ to a constant $c_6$ (independent of $M$) such that $h(\epsilon) \leq 0$ and $ h'(\epsilon) \leq 0$ for all $\epsilon \geq c_6\epsilon_M$. 
These results imply the probability bound,
\begin{equation} \label{eq:maintheorem_intResult2}
P_{\bmu}\{\| 
\bphEA_{L,M,\mbf{\mathcal{B}}_{n_*}}
- \bpEA \|_{\LtwoB(\brhoEAL)}^2
> \epsilon \} \leq  \begin{cases} \exp\Big(\frac{-c_4M\epsilon}{2}\Big), &  \epsilon \geq c_6\epsilon_M \\ 1,  &  \epsilon \leq c_6\epsilon_M \end{cases}
\end{equation}
Integrating this bound over $\epsilon\in(0,+\infty)$ and using the elementary inequality $e^{-x}\leq x+1$ for all $x\geq 0$, we get that  
$$ 
\int_0^{\infty} P_{\bmu}\{\| 
\bphEA_{L,M,\mbf{\mathcal{B}}_{n_*}} - \bpEA\|_{\LtwoB(\brhoEAL)}^2
> \epsilon \}  d\epsilon \leq C_1\Big(\frac{\log M}{M}\Big)^{\frac{2s}{2s+|\spacevariables|}},
$$ 
where $C_1$ is a constant depending on $c_0,c_1,s, K, S_{EA}, R, R_{\dot{x}}$. 
Now with $\bhypspaceEA_M = \mbf{\mathcal{B}}_{n_*}$ and using \eqref{2ndorder:mainestimate1}, we have shown the convergence rate, 
$$ 
\mathbb{E}_{\bmu}[\| 
\bphEA_{L,M,\bhypspaceEA_M}
- \bpEA \|_{\LtwoB(\brhoEAL)}^2]  \le \frac{c_7}{c_{\bigcup_M \bhypspaceEA_M}} \Big(\frac{M}{\log M}\Big)^{-\frac{2s}{2s+|\spacevariables|}}, 
$$ where $c_7$ is an absolute constant that only depends on $s,K,c_0, c_1, S_{EA}, R, R_{\dot{x}}$. 

\end{proof}

In both theorems, the convergence rates $\frac{2s}{2s+|\spacevariables|}$ and $\frac{2s}{2s+|\spacevariables_{\xi}|}$ coincide with the minimax rate of convergence for nonparametric regression in the corresponding dimension -- up to the logarithmic factor. It is possible that this logarithmic factor could be removed (see techniques in Chapter 11-15 of \cite{Gyorfi06}), but with considerable additional complexity of the proof. 
Achieving the same rate of convergence as if we had observed the noisy values of the interaction kernels directly, rather than through the dynamics, is a major strength of our approach. The strong consistency results show the asymptotic optimality of our method, and for wide classes of systems the assumptions of the theorems apply.  
Specifically, for part (b) of the theorems, the dimension and approximation conditions can be explicitly achieved by piecewise polynomials or splines appropriately adapted to the regularity of the kernel. In the conditions of theorem \ref{t:2ndordersystem:thm_optRate}, $n$ can be the number of partitions along each axis of the variables in $\spacevariables$. Then, using multivariate splines or piecewise polynomials we will have a fixed constant $c_0$ (corresponding to the number of parameters to estimate for each function) times $Kn^{\spacevariables}$ as the dimension of the linear space. Furthermore, by standard approximation theory results, see \cite{Schumaker} (Chapters 12,13), \cite{DeVore1980},\cite{DeBoor1983}, for $s$ the regularity of the interaction kernels we achieve the desired approximation condition with piecewise polynomials of degree $\lfloor s\rfloor$. 
In our admissible spaces we have $s=1$, note that the theorems are stronger if we have a kernel of higher regularity.  

We next briefly examine the convergence rate on a few systems of fundamental interest.
Recall that in table \ref{tab:2ndOrder_Examples} we have as the final two columns the values $|\spacevariables|, |\spacevariables_{\xi}|$. These correspond directly to the rate of convergence of each of the system under our learning approach.  

Some specific highlights:
\begin{itemize}
\item For Anticipation Dynamics (AD), even though we are learning both an energy and alignment kernel, because there are only $2$ unique variables shared across both of them we learn at the $2$-dimensional rate.
\item For the Synchronized Oscillator we achieve the $2$-dimensional optimal learning rate on each of the $EA$ and $\xi$ portions (rather than a $4$-dimensional rate) due to the decoupled nature of the system, similarly we only pay the $1$-dimensional rate twice for the Phototaxis system. This is a key reason for splitting our learning theory between $EA$- and $\xi$-interaction kernels and accounting for shared and non-shared variables: this enables us to substantially improve the performance guarantees in actual applications.
\item Due to the design of the measures, norms and the associated learning algorithm, even in the heterogeneous case for celestial mechanics and predator-swarm, we only pay the $1$-dimensional learning rate, although the constants are of course affected by the heterogeneity and the algorithm requires a larger learning matrix.
\item The rates of convergence of our estimators for all previously-studied first-order systems (see \cite{lu2019nonparametric,Tang2019,Zhong20}) can be derived from Theorem \ref{main:convrate}.
\end{itemize}

One downside of the results above is the lack of dependence on $L$ as it seems natural that finer time samples in each trajectory should improve the results. Indeed, the numerical experiments of \cite{Zhong20, lu2019nonparametric, Tang2019} demonstrate that more data in $L$ may indeed be helpful to improve the performance. One technique used in \cite{Zhong20} for very long trajectory data (large $L$, medium to small $M$) is to split each trajectory into larger $M$ with smaller $L$ in each. 
In this way, one can explicitly show the desired convergence rate in a form agreeing with the above theorems, we do not believe that it leads to significantly different performance compared to using the original data, only that we can easily transform the data into a form that the theorems apply to and with no loss of performance.  Explicit dependence on $N$ is not the objective of this work, see \cite{BFHM17}, but further study of the mean-field regime is of interest to the authors and work is ongoing.

\section{Performance of trajectory prediction} \label{sec:trajectory}
Once estimators $\bvphEA,\bvphxi$ are obtained, a natural question is the accuracy of the evolved trajectory based on these estimated kernels. The next theorem shows that the error in prediction is (i) bounded trajectory-wise by a continuous-time version of the error functional, and (ii) bounded on average by the $\LtwoB(\brhoEA), \LtwoB(\brhoxi)$, respectively, error of the estimator. This further validates the effectiveness of our error functional and $\LtwoB(\brhoT)$-metric to assess the quality of the estimator. 
In particular, this emphasizes that although the system is a coupled system of ODE's, our decoupled learning procedure with our choice of norm will lead to control of the expected supremum error as long as we minimize the $\LtwoB(\brhoEA), \LtwoB(\brhoxi)$ norms in obtaining our estimators. 

\begin{theorem}\label{secondordersystem:TrajDiff_NEW}
Suppose that $\widehat\bintkernel^E \in \boldsymbol{\mathcal{K}}_{S_E}^E$, $\widehat\bintkernel^A \in \boldsymbol{\mathcal{K}}_{S_A}^A$ and $\widehat\bintkernel^{\xi} \in \boldsymbol{\mathcal{K}}_{S_{\xi}}^{\xi}$. 
Denote by $\widehat{\bY}(t)$ and $\bY(t)$ 
the solutions of the systems with kernels $\widehat\bintkernel^E=(\widehat\intkernel_{kk'}^{E})_{k,k'=1}^{K,K}, \widehat\bintkernel^A =(\widehat\intkernel_{kk'}^{A})_{k,k'=1}^{K,K}$, and $\widehat\bintkernel^{\xi} =(\widehat\intkernel_{kk'}^{\xi})_{k,k'=1}^{K,K}$ and $\bintkernel^E, \bintkernel^A,\bintkernel^{\xi}$ respectively, both with the same initial condition. Then 
\begin{align*}
\sup_{t\in[0,T]}\Vert \widehat{\bY}(t)- \bY(t)\Vert_{\mathcal{Y}}^2 &\leq g(T) \bigg[2T^2 \int_{p=0}^t \int_{s=0}^p \| \ddot{\bX} - \rhsfvnc(\coords) - \rhsf^{\bphEA}(\coords) \|_{\mathcal{S}}^2 \text{ ds dp} \\
&+ 2T  \int_{s=0}^t  \| \ddot{\bX} - \rhsfvnc(\coords) - \rhsf^{\bphEA}(\coords) \|_{\mathcal{S}}^2 ds \\
&+ 2T\int_{s=0}^t \| \dot{\bXi} - \rhsfxinc(\coords) - \rhsf^{\widehat{\bintkernel}_{\xi}}(\coords) \|_{\mathcal{S}}^2 ds \bigg] 
\end{align*}
where $g(T) =1+(1+B_1T)T\exp(A_1T+T^2/2)$. The constants are $A_1 = 2T(8KP + \mathcal{L} +8QK + \mathcal{L}^{\xi})$ and $B_1 = 2T^2(8KP + \mathcal{L})$, with any unspecified constants made precise in the proof and only depending on the Lipschitz constants of the noncollective forces and the feature maps, as well as the values $S^E, S^A, S^{\xi}$ coming from the admissible spaces.
It is bounded on average, with respect to the initial distribution $\bmu$, by
\begin{align} \label{trajPred:KeyExpectation}
\E_{\bmu}[\sup_{t\in[0,T]} \|\widehat{\bY}(t)- \bY(t)\|_{\mathcal{Y}}^2] &\leq g(T)\bigg((T^2K^2 + TK^2)\Vert \bphEA - \bpEA \Vert_{\LtwoB(\brhoEA)}^2 \nonumber \\
&+ TK^2\Vert \widehat{\bm{\phi}}^{\xi} - \bm{\phi}^{\xi} \Vert_{\LtwoB(\bm{\rho}_T^{\xi})}^2 \bigg) 
\end{align}
with the measures $\brhoEA,\bm{\rho}_T^{\xi}$ defined in (\ref{meas:fullEA}, \ref{meas:rhoxibar}). Expression (\ref{trajPred:KeyExpectation}) shows that by minimizing the right hand side, we can control the expected $\mathcal{Y}$-supremum error of the estimated trajectories. 

\end{theorem}
We postpone the somewhat lengthy proof to Appendix \ref{s:trajectoryerrorproof}.

\section{Applications} \label{sec:numerics}
Our learning theory, as well as measures, norms, functionals etc. can be applied to study all the examples considered in the works \cite{lu2019nonparametric, Tang2019, Zhong20}. These examples, particularly those of \cite{Zhong20}, can thus be considered as applications of the theoretical results as well as of the algorithm in section \ref{sec:algorithm_numeric}. 

We choose to study two new dynamics, which are not considered in \cite{lu2019nonparametric, Zhong20} since they exhibit some unique features of our generalized model.  In particular, we choose them due to their special form of having both energy-based and alignment-based interactions.  These are the flocking with external potential (FwEP) model in \cite{ST2020} and the anticipation dynamics (AD) model in \cite{shu2019anticipation}.  

Table \ref{tab:app_params} shows the value of learning parameters for these dynamics.
\begin{table}[H]
\centering
\begin{tabular}{ c | c | c | c | c | c | c }
\hline
$M_{\rho}$ & $L$   & $T_f$ & $T$ & $\mu^{\bx}$             & $\mu^{\dot\bx}$          & Num. of learning trials\\
\hline
$2000$     & $500$ & $10$  & $5$ & $\text{Unif.}([0, 5]^2)$ & $\text{Unif.}([0, 5]^2)$ & $10$\\
\hline
\end{tabular}
\caption{Values of parameters for the learning.}
\label{tab:app_params} 
\end{table}

The setup of the learning experiment is as follows.  We use $M_{\rho}$ different initial conditions to evolve the dynamics\footnote{We use the built-in MATLAB integrating routine, $ode15s$, with relative tolerance at $10^{-8}$ and absolute tolerance at $10^{-11}$.} from $0$ to $T$ to generate a good approximation to $\rho_T^{L, EA}, \rho_T^{L, E}$ and $\rho_T^{L, A}$.  Then we use another set of $M$ ($M = 500$ for FwEP and $M = 750$ for AD) initial conditions to generate training data to learn the corresponding $\intkernele$ and $\intkernela$ from the empirical distributions, $\rho_T^{L, M, EA}$'s, etc.  We report the relative learning errors calculated via \eqref{e:L2NormDefs} for $\lintkernele \oplus \lintkernela$, \eqref{e:L2NormDefs} for $\lintkernele$, and \eqref{e:L2NormDefs} for $\lintkernela$, along with pictorial comparison of those interaction kernels as well as a visualization on the pairwise data which is used to learn the estimated kernels.  Then we evolve the dynamics either from the training set of $M$ initial conditions or another set of $M$ randomly chosen initial conditions with $\intkernele \oplus \intkernela$ and $\lintkernele \oplus \lintkernela$ from $0$ to $T_f > T$, and report the trajectory errors calculated using \eqref{eq:traj_norm} on $\by$ (the whole system), and for $\bx$ (the position) and $\bv$ (the velocity).  Again, pictorial comparison of the trajectories are also shown.  We report the trajectory errors over $[0, T]$ and $[T, T_f]$.  The learning results are shown in the following sections.

\subsection{Learning results for flocking with external potential}
We consider the FwEP model for its simplicity and clustering behavior in both position and velocity (hence flocking occurs).  The dynamics of the FwEP model is given as follows,
\begin{equation}
\ddot\bx_i = \frac{1}{N}\sum_{i' = 1, i' \neq i}^N a(\dot\bx_{i'} - \dot\bx_i) \nonumber + \frac{1}{N}\sum_{i' = 1, i' \neq i}^N\intkernel(\norm{\bx_{i'} - \bx_i})(\dot\bx_{i'} - \dot\bx_i).\label{eq:FwEP_model}
\end{equation}
Here $a > 0$ is a constant representing an attraction force, and $\intkernel = \frac{1}{(1 + r^2)^\beta}$\footnote{This choice of interaction for alignment is not mandatory.  It is a good comparison between this FwEP model with the Cucker-Smale mode with this choice of interaction functions.} with $\beta = \frac{1}{2}$.  To fit into our learning regime, we take, $m_i = 1$, $\numcl = 1$, no $\xi_i$, no non-collective force, and 
\[
\intkernele = a \quad \text{and} \quad \intkernela = \frac{1}{(1 + r^2)^\beta}.
\]
For the FwEP model, we use the space of $1^{st}$ degree piece-wise polynomials with dimension $n^E = 122$ for learning $\intkernele$; and for $\intkernela$, we use the same space.  First, consider the comparison of energy-based interactions shown in Fig. \ref{fig:FwEP_phiE}.
\begin{figure}[H]
  \centering
  \includegraphics[width=.8\linewidth]{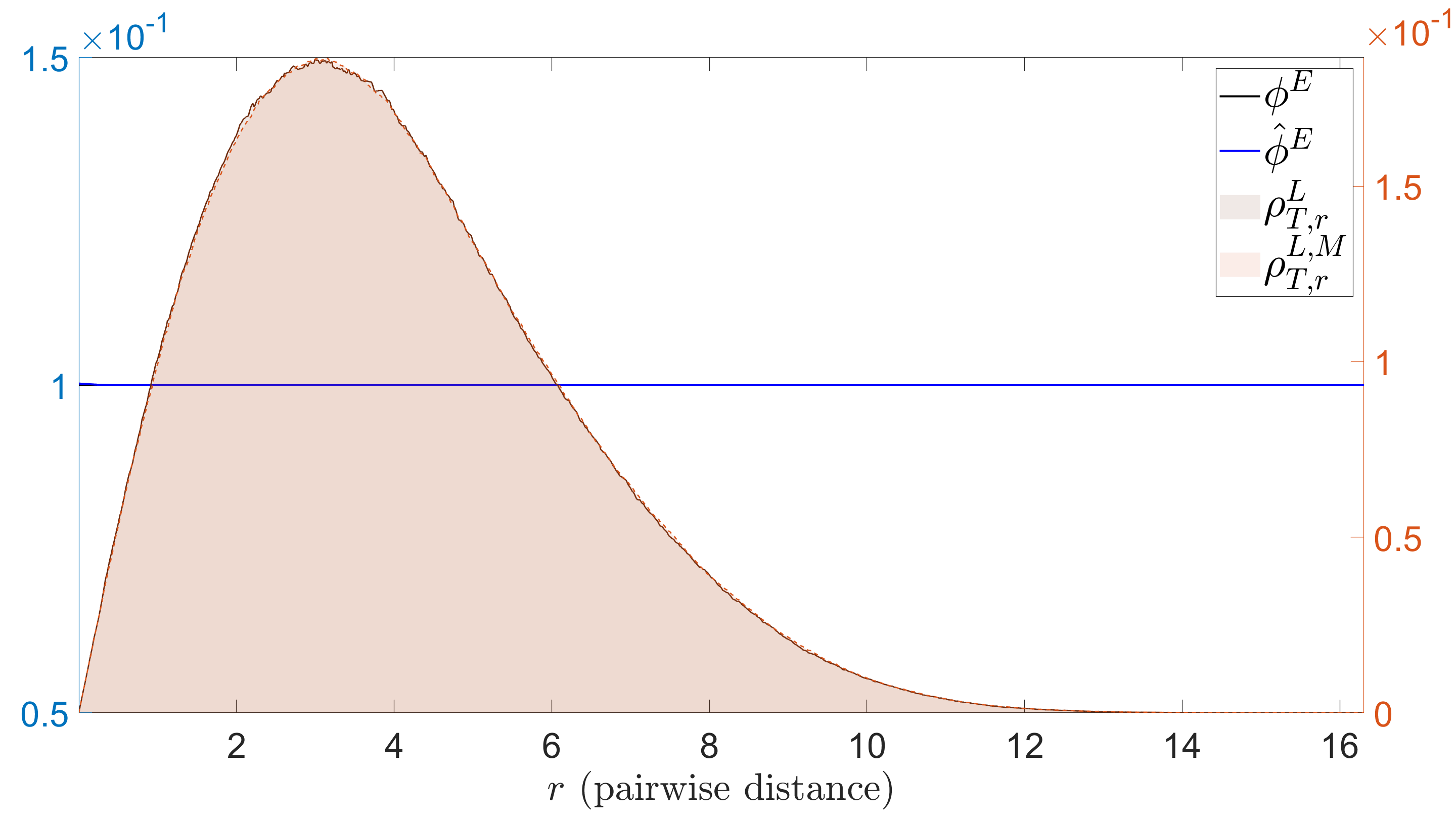}
  \caption{$\intkernele$ vs. $\lintkernele$, Err: $3.9 \cdot 10^{-6} \pm 6.6 \cdot 10^{-7}$.  The lines shown in blue are the estimated interaction kernels, and the lines shown in black are the true interaction kernels. The colored areas shown in the background are the learned distributions of pairwise distance data.}
  \label{fig:FwEP_phiE}
\end{figure}
Fig. \ref{fig:FwEP_phiE} shows that our learning performance on constant functions using piecewise linear polynomials shows promising results.  However, we still have trouble learning the behavior of the interaction at $r = 0$, part of it due to the weight of $\vec{0}$ in the model, and the other part of it being lack of available data towards $r = 0$.  Next, we show the comparison of alignment-base interactions shown in Fig. \ref{fig:FwEP_phiA_results} with distribution of the pairwise data.
\begin{figure}[H]
\centering
\begin{subfigure}{.5\textwidth}
  \centering
  \includegraphics[width=.9\linewidth]{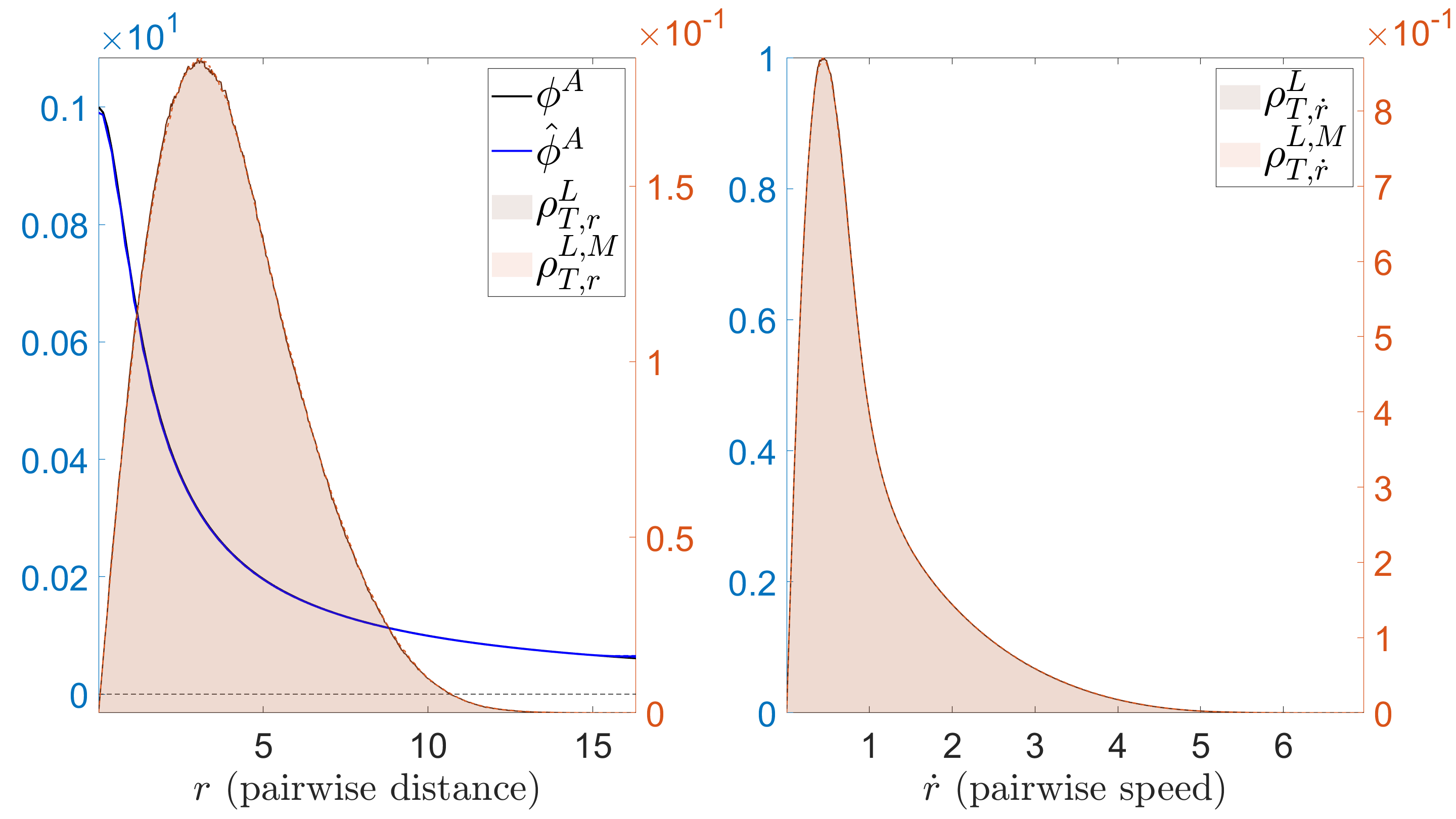}
  \caption{$\intkernela$ vs. $\lintkernela$, Err: $9.1 \cdot 10^{-3} \pm 2.5 \cdot 10^{-4}$.}
  \label{fig:FwEP_phiA}
\end{subfigure}%
\begin{subfigure}{.5\textwidth}
  \centering
  \includegraphics[width=.9\linewidth]{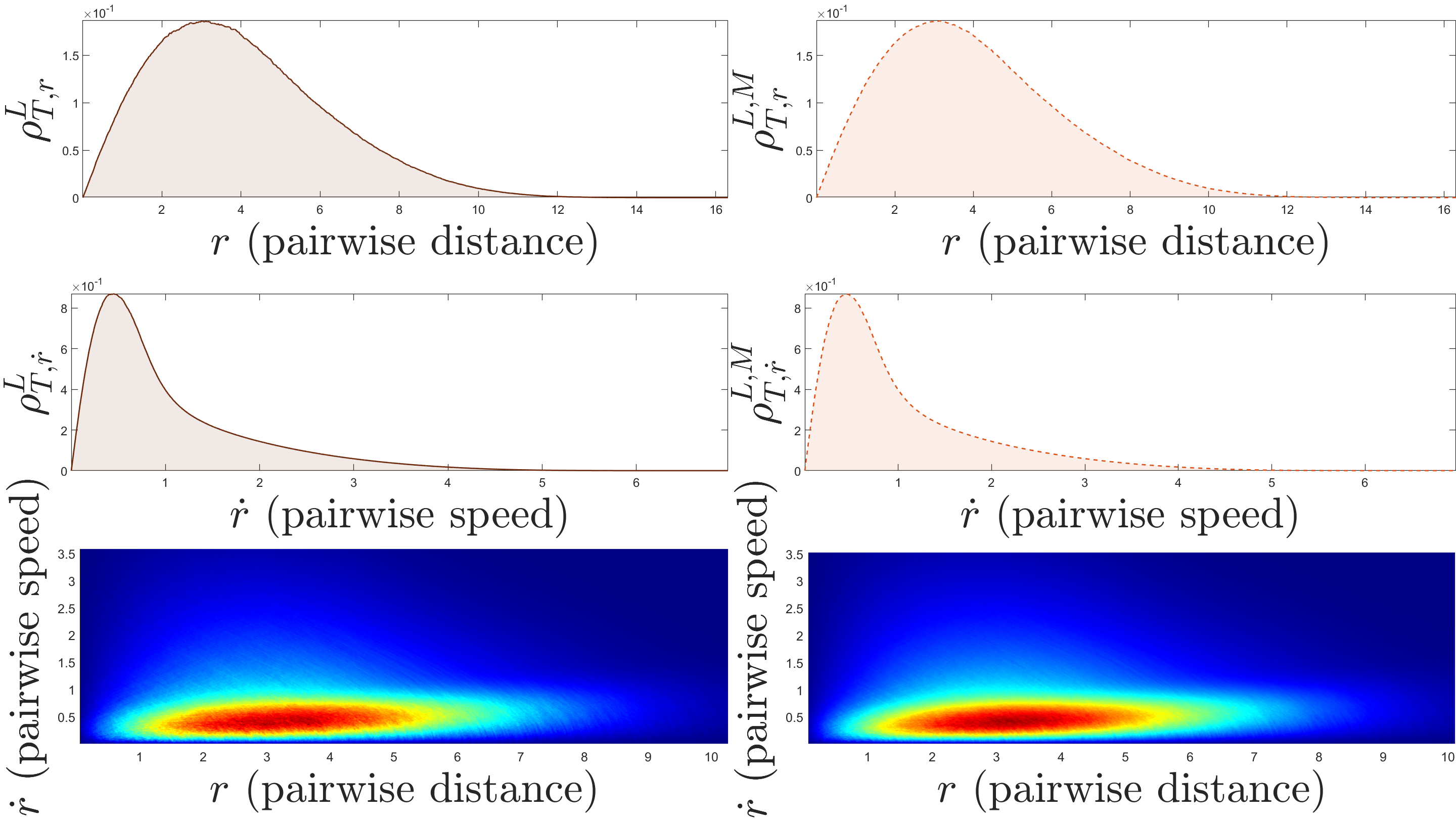}
  \caption{$\rho_{T}^{A,L}$ vs. $\rho_{T}^{A,L, M}$.}
  \label{fig:FwEP_rhoA}
\end{subfigure}
\caption{The lines shown in blue are the estimated interaction kernels, and the lines shown in black are the true interaction kernels. 
The colored areas shown in the background are the learned distributions of pairwise distance data.}
\label{fig:FwEP_phiA_results}
\end{figure}
Again, in Fig. \ref{fig:FwEP_phiA}, it shows a faithful approximation from our estimated kernels.  The $\lintkernele \oplus \lintkernela$ error is: $5.8 \cdot 10^{-3} \pm 1.6 \cdot 10^{-4}$.  The comparison of trajectories are shown in Fig. \ref{fig:FwEP_trajs}.
\begin{figure}[H]
  \centering
  \includegraphics[width=.8\linewidth]{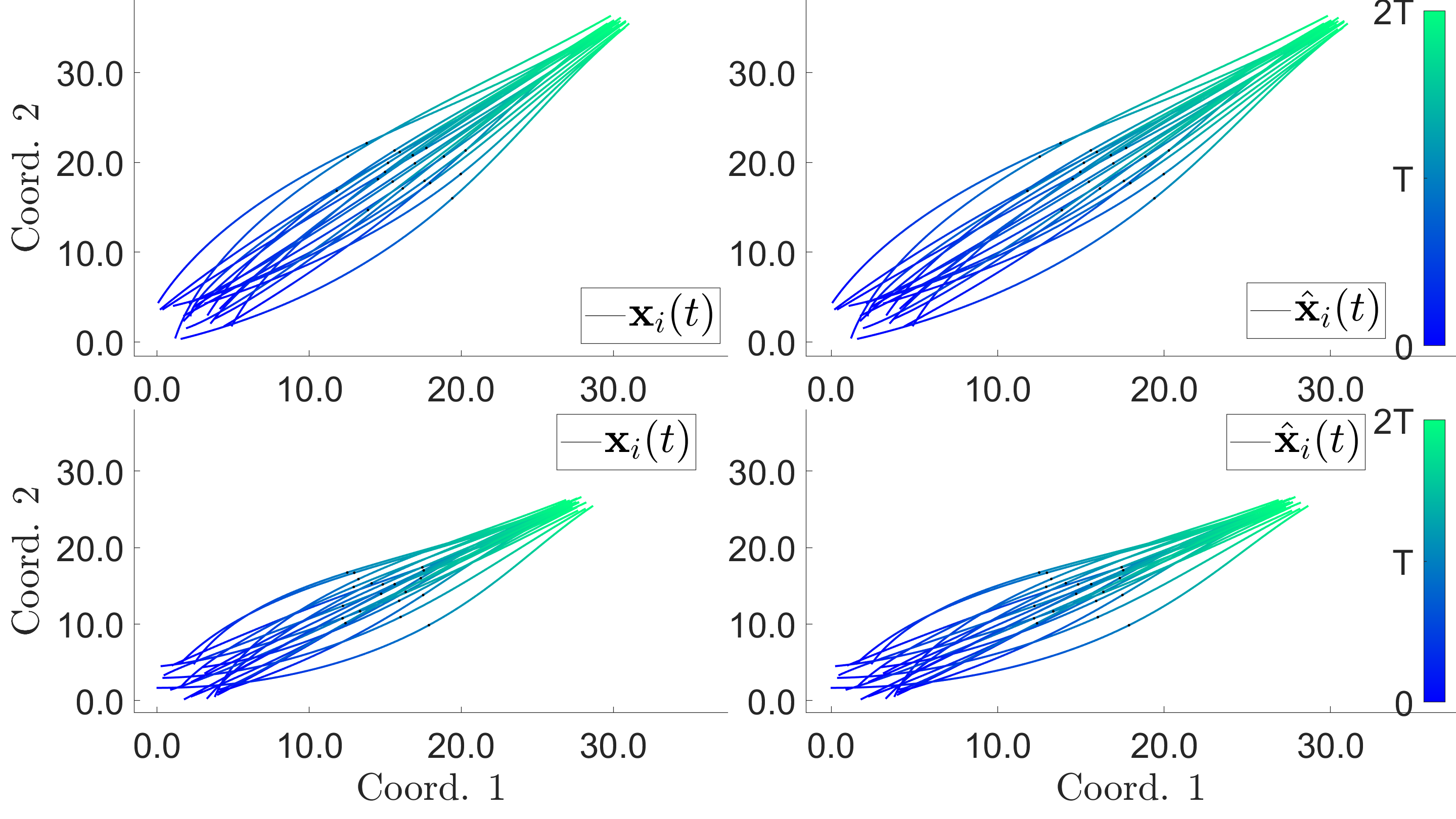}
  \caption{Trajectory Comparison.}
  \label{fig:FwEP_trajs}
\end{figure}
Fig. \ref{fig:FwEP_trajs} shows little visual difference between the learned and observed trajectories.  A more quantitative description of the trajectory errors are shown in table \ref{tab:FwEP_traj_err}.
\begin{table}[H]
\centering
\begin{tabular}{ c | c | c }
\hline
                                   & $[0, T]$                                   & $[T, T_f]$ \\
\hline
$\text{mean}_{\text{IC}}$ on $\bx$ & $7.2 \cdot 10^{-4} \pm 1.9 \cdot 10^{-5}$  & $6.7 \cdot 10^{-4} \pm 1.8 \cdot 10^{-5}$ \\
$\text{mean}_{\text{IC}}$ on $\bv$ & $1.15 \cdot 10^{-3} \pm 3.1 \cdot 10^{-5}$ & $1.5 \cdot 10^{-3} \pm 4.1 \cdot 10^{-3}$ \\
$\text{mean}_{\text{IC}}$ on $\by$ & $6.2 \cdot 10^{-6} \pm 1.7 \cdot 10^{-7}$  & $2.22 \cdot 10^{-6} \pm 6.8 \cdot 10^{-8}$ \\
\hline
$\text{std}_{\text{IC}}$ on $\bx$  & $1.28 \cdot 10^{-4} \pm 4.2 \cdot 10^{-6}$ & $1.22 \cdot 10^{-4} \pm 4.0 \cdot 10^{-6}$ \\
$\text{std}_{\text{IC}}$ on $\bv$  & $2.20 \cdot 10^{-4} \pm 7.0 \cdot 10^{-6}$ & $2.5 \cdot 10^{-4} \pm 1.0 \cdot 10^{-5}$ \\
$\text{std}_{\text{IC}}$ on $\by$  & $1.52 \cdot 10^{-6} \pm 5.9 \cdot 10^{-8}$ & $6.0 \cdot 10^{-7} \pm 2.6 \cdot 10^{-8}$ \\
\hline
\hline
$\text{mean}_{\text{IC}}$ on $\bx$ & $7.2 \cdot 10^{-4} \pm 1.7 \cdot 10^{-5}$  & $6.7 \cdot 10^{-4} \pm 1.6 \cdot 10^{-5}$ \\
$\text{mean}_{\text{IC}}$ on $\bv$ & $1.15 \cdot 10^{-3} \pm 2.7 \cdot 10^{-5}$ & $1.46 \cdot 10^{-3} \pm 3.3 \cdot 10^{-5}$ \\
$\text{mean}_{\text{IC}}$ on $\by$ & $6.2 \cdot 10^{-6} \pm 1.6 \cdot 10^{-7}$  & $2.22 \cdot 10^{-6} \pm 5.4 \cdot 10^{-8}$ \\
\hline
$\text{std}_{\text{IC}}$ on $\bx$  & $1.30 \cdot 10^{-4} \pm 5.8 \cdot 10^{-6}$ & $1.24 \cdot 10^{-4} \pm 5.3 \cdot 10^{-6}$ \\
$\text{std}_{\text{IC}}$ on $\bv$  & $2.25 \cdot 10^{-4} \pm 9.9 \cdot 10^{-6}$ & $2.56 \cdot 10^{-4} \pm 9.1 \cdot 10^{-6}$ \\
$\text{std}_{\text{IC}}$ on $\by$  & $1.6 \cdot 10^{-6} \pm 7.6 \cdot 10^{-8}$  & $6.2 \cdot 10^{-7} \pm 2.4 \cdot 10^{-8}$ \\
\hline
\end{tabular}
\caption{Trajectory Errors.  The first three rows of mean trajectory errors are from the training set of initial conditions.  The next three rows are standard deviation of the trajectory errors from the training set of initial conditions.  The following three rows are mean trajectory errors from a new set of initial conditions.  Finally, the last three rows report the standard deviation of the trajectory errors from a new set of initial conditions.}
\label{tab:FwEP_traj_err} 
\end{table}
We are maintaining a relative four-digit accuracy in estimating the position, and a relative three-digit accuracy in estimating the velocity of the agents in the system.  Although we are able to reconstruct $\intkernele$ with a $6$-digit accuracy, we are not able to do the same for $\intkernela$.  The error in $\lintkernele \oplus \lintkernela$ reflects this discrepancy by considering the two functions together.
\subsection{Learning results for anticipation dynamics with $U(r) = \frac{r^p}{p}$}\label{sec:example_AD}
The energy-based interactions are constants in the FwEP models, if we want to consider more complicated models, i.e., interactions depending on pairwise distance and more, the AD models are suitable candidates.  The dynamics of the AD model is given as follows,
\begin{align}
\ddot\bx_i &= \frac{1}{N}\sum_{i' = 1, i' \neq i}^N \frac{\tau U'(\norm{\bx_{i'} - \bx_i})}{\norm{\bx_{i'} - \bx_i}}(\dot\bx_{i'} - \dot\bx_i) \nonumber \\
& \quad + \frac{1}{N}\sum_{i' = 1, i' \neq i}^N\Big\{\frac{-\tau U'(\norm{\bx_{i'} - \bx_i})(\bx_{i'} - \bx_i)\cdot(\dot\bx_{i'} - \dot\bx_i)}{\norm{\bx_{i'} - \bx_i}^3} \nonumber \\
& \quad + \frac{\tau U''(\norm{\bx_{i'} - \bx_i})(\bx_{i'} - \bx_i)\cdot(\dot\bx_{i'} - \dot\bx_i)}{\norm{\bx_{i'} - \bx_i}^2} + \frac{U'(\norm{\bx_{i'} - \bx_i})}{\norm{\bx_{i'} - \bx_i}}\Big\}(\bx_{i'} - \bx_i). \label{eq:AD_model}
\end{align}
Compared to the original model in \cite{shu2019anticipation}, we take $\tau_{i, i'} = 0$.  In order to fit the model into our learning regime, we take
\[
\intkernela(r) = \frac{\tau U'(r)}{r} \mand \intkernele(r, s) = \frac{-\tau U'(r)s}{r^3} + \frac{\tau U''(r)s}{r^2} + \frac{U'(r)}{r}.
\]
Here we have no $\xi_i$, $\numcl = 1$, $m_i = 1$, and
\[
s^E_{i, i'} = s^A_{i, i'} = (\bx_{i'} - \bx_i)\cdot(\dot\bx_{i'} - \dot\bx_i).
\]
We also use $\tau = 0.1$.

It is shown in \cite{shu2019anticipation} that if $U''$ is bounded when $r \rightarrow \infty$ with $U(0) = U'(0) = 0$, then unconditionally flocking would occur.  We take $U(r) = \frac{r^p}{p}$ for $1 < p \le 2$, then the system would show unconditional flocking.  We choose $p = 1.5$ for our learning trials\footnote{$p = 2$ induces constant forces on the dynamics.}.  We use a tensor grid of $1^{st}$ degree piece-wise standard polynomials with $n^E = 28^2$ for learning $\intkernele(r, s)$, then a set of $1^{st}$ degree piecewise standard polynomials with $n^A = 138$ for learning $\intkernela(r)$.  For the energy-based interactions we have the following results. 
\begin{figure}[H]
\centering
\begin{subfigure}{.5\textwidth}
  \centering
  \includegraphics[width=.9\linewidth]{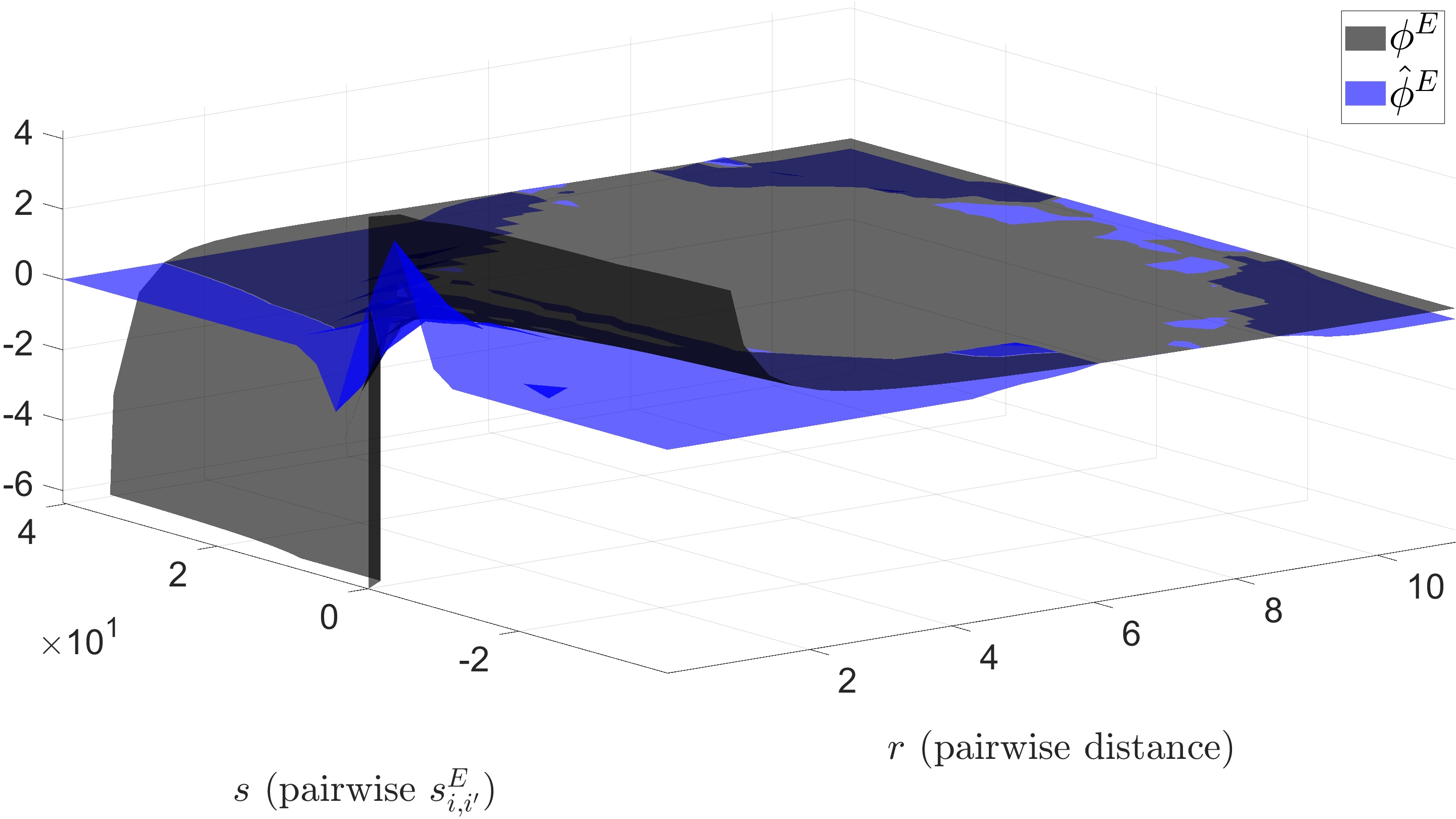}
  \caption{$U(r) = \frac{r^{1.5}}{1.5}$: $\intkernele$ vs. $\lintkernele$, Err: $6 \cdot 10^{-1} \pm 2.6 \cdot 10^{-1}$.}
  \label{fig:AD01_phiE}
\end{subfigure}%
\begin{subfigure}{.5\textwidth}
  \centering
  \includegraphics[width=.9\linewidth]{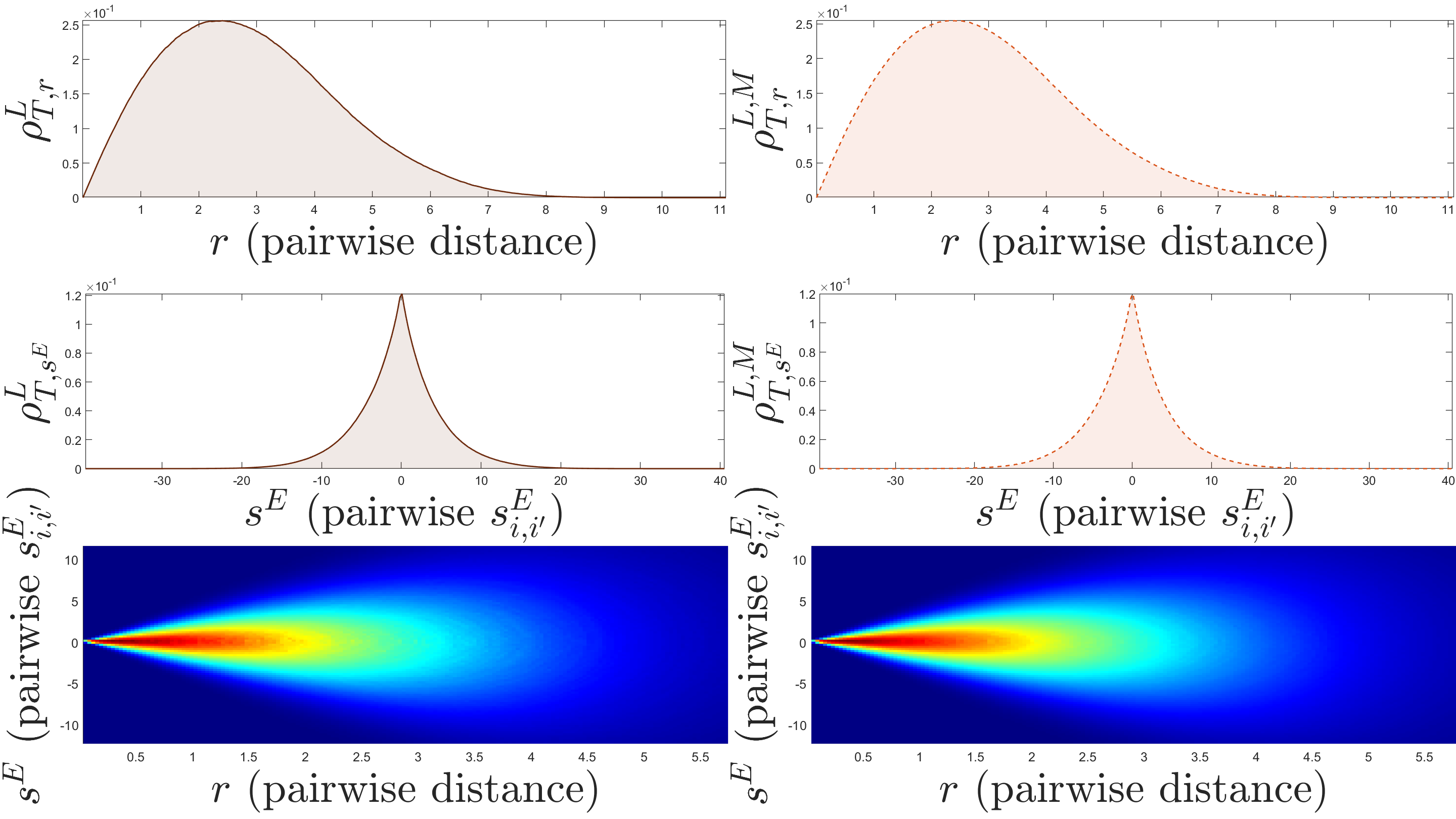}
  \caption{$U(r) = \frac{r^{1.5}}{1.5}$: $\rho_T^{E,L}$ vs. $\rho_T^{E,L, M}$.}
  \label{fig:AD01_rhoE}
\end{subfigure}
\caption{The lines shown in blue are the estimated interaction kernels, and the lines shown in black are the true interaction kernels. 
The colored areas shown in the background are the learned distributions of pairwise distance data.}
\label{fig:AD01_phiE_results}
\end{figure}

As is shown in Fig. \ref{fig:AD01_rhoE}, the concentration of pairwise distance data is away from $0$, making the estimation of the behavior of $\intkernele(r, s)$ at $r$ close to $0$ extremely difficult, meanwhile, since $\intkernele$ is also weighted by the pairwise difference, $\bx_{i'} - \bx_i$, and at $r_{i, i'}$ close to $0$, the information is also lost.  Next, we present the alignment-based interaction kernels.

\begin{figure}[H]
\centering
\begin{subfigure}{.5\textwidth}
  \centering
  \includegraphics[width=.9\linewidth]{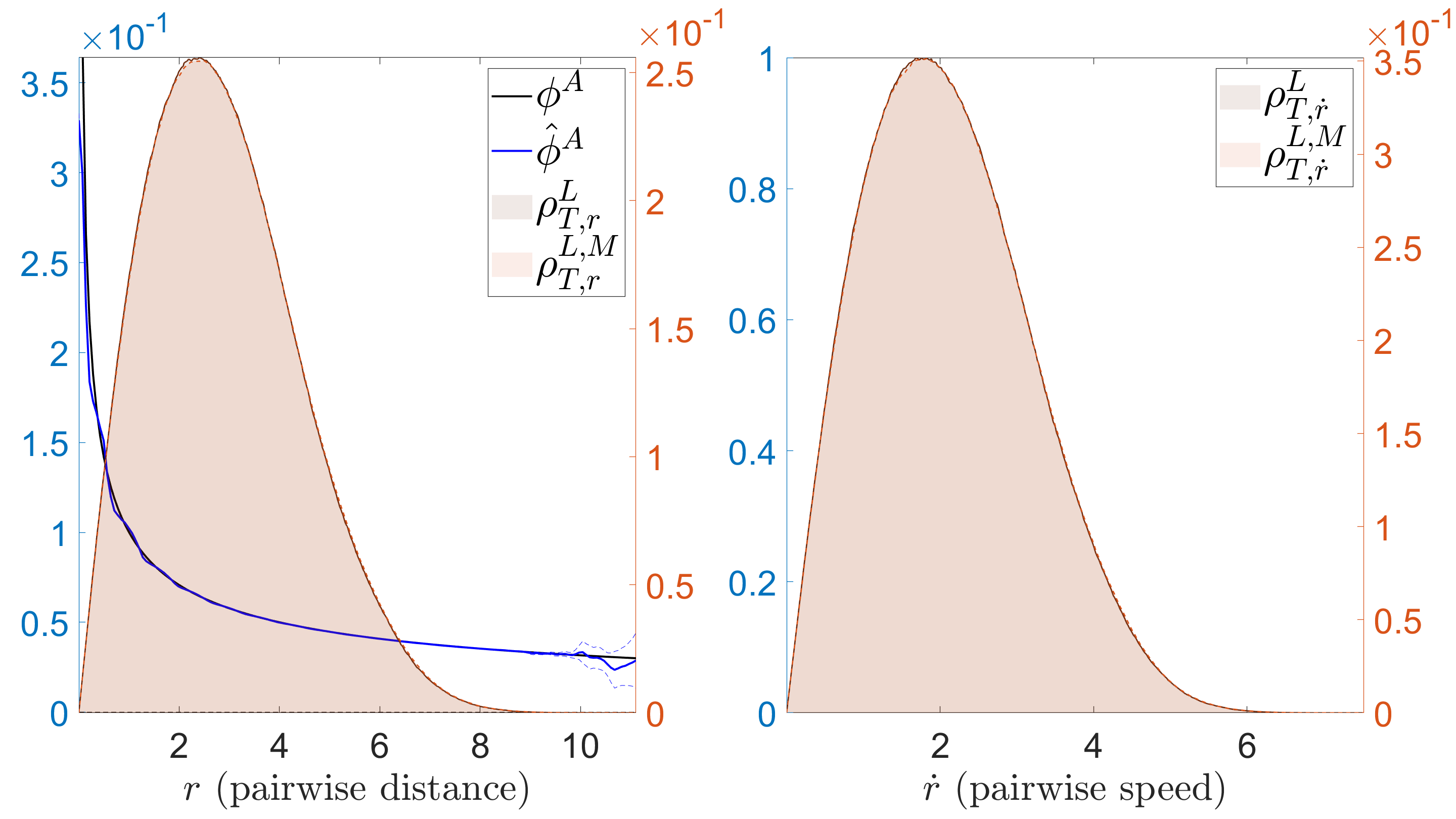}
  \caption{$U(r) = \frac{r^{1.5}}{1.5}$: $\intkernela$ vs. $\lintkernela$. Err: $1.7 \cdot 10^{-1} \cdot 3.9 \cdot 10^{-2}$.}
  \label{fig:AD01_phiA}
\end{subfigure}%
\begin{subfigure}{.5\textwidth}
  \centering
  \includegraphics[width=.9\linewidth]{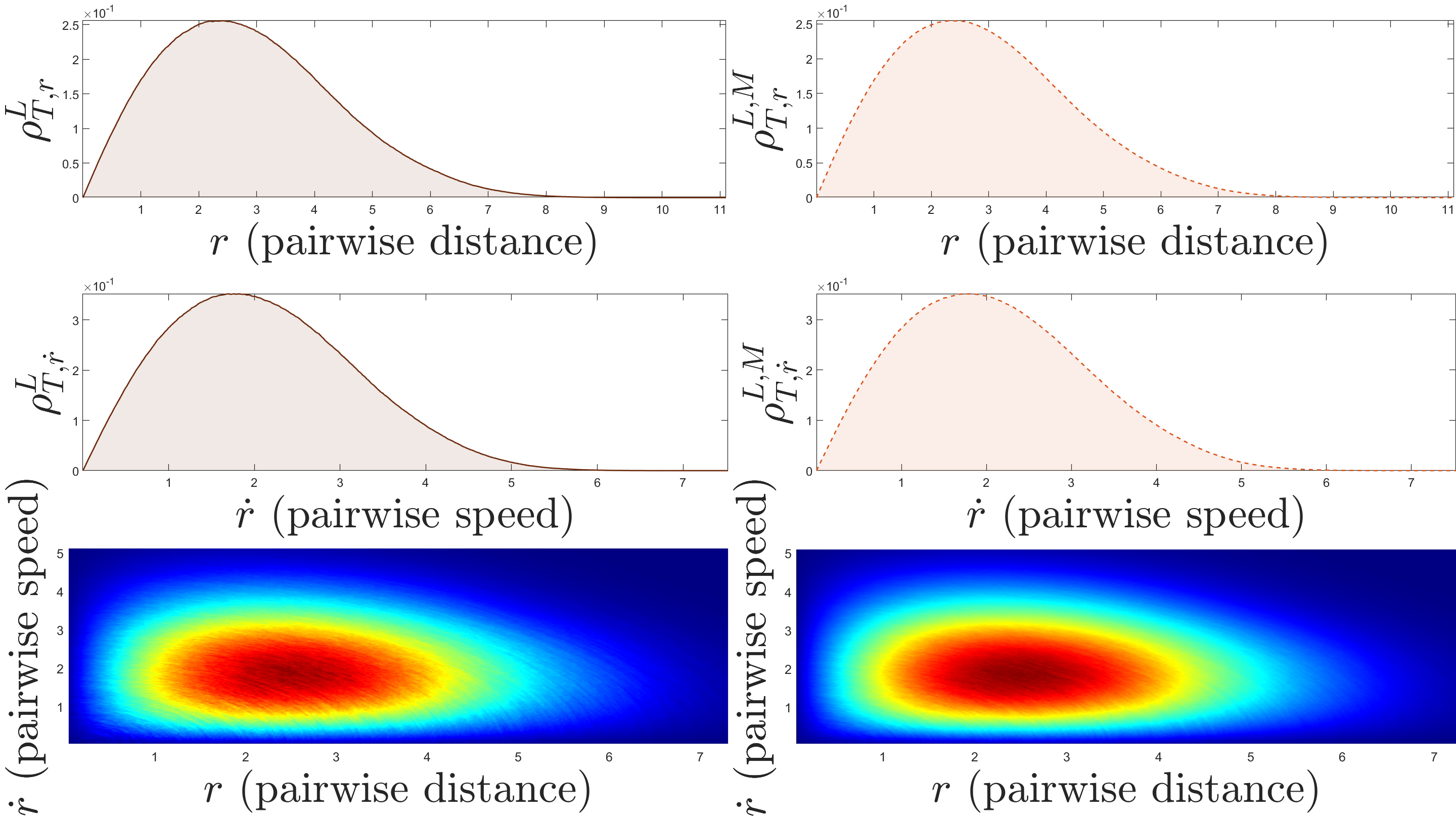}
  \caption{$U(r) = \frac{r^{1.5}}{1.5}$: $\rho_T^{A,L}$ vs. $\rho_T^{A,L, M}$.}
  \label{fig:AD01_rhoA}
\end{subfigure}
\caption{The lines shown in blue are the estimated interaction kernels, and the lines shown in black are the true interaction kernels. 
The colored areas shown in the background are the learned distributions of pairwise distance data.}
\label{fig:AD01_phiA_results}
\end{figure}
We have less trouble estimating the behavior of $\intkernela$ at $r = 0$.  We have trouble estimating $\intkernela$ at the other end of the spectrum of $r$, since the agents have aligned their velocities, hence the weight $\bv_{i'} -  \bv_i$ is close to a zero vector.  The overall learning performance for estimating $\intkernela$ is better compared to estimating $\intkernele$.  The $\lintkernele \oplus \lintkernela$ error is: $6 \cdot 10^{-1} \pm 3.0 \cdot 10^{-1}$.  The comparison of trajectories between the true kernels (LHS) and the estimators (RHS) is shown below.
\begin{figure}[H]
  \centering
  \includegraphics[width=.8\linewidth]{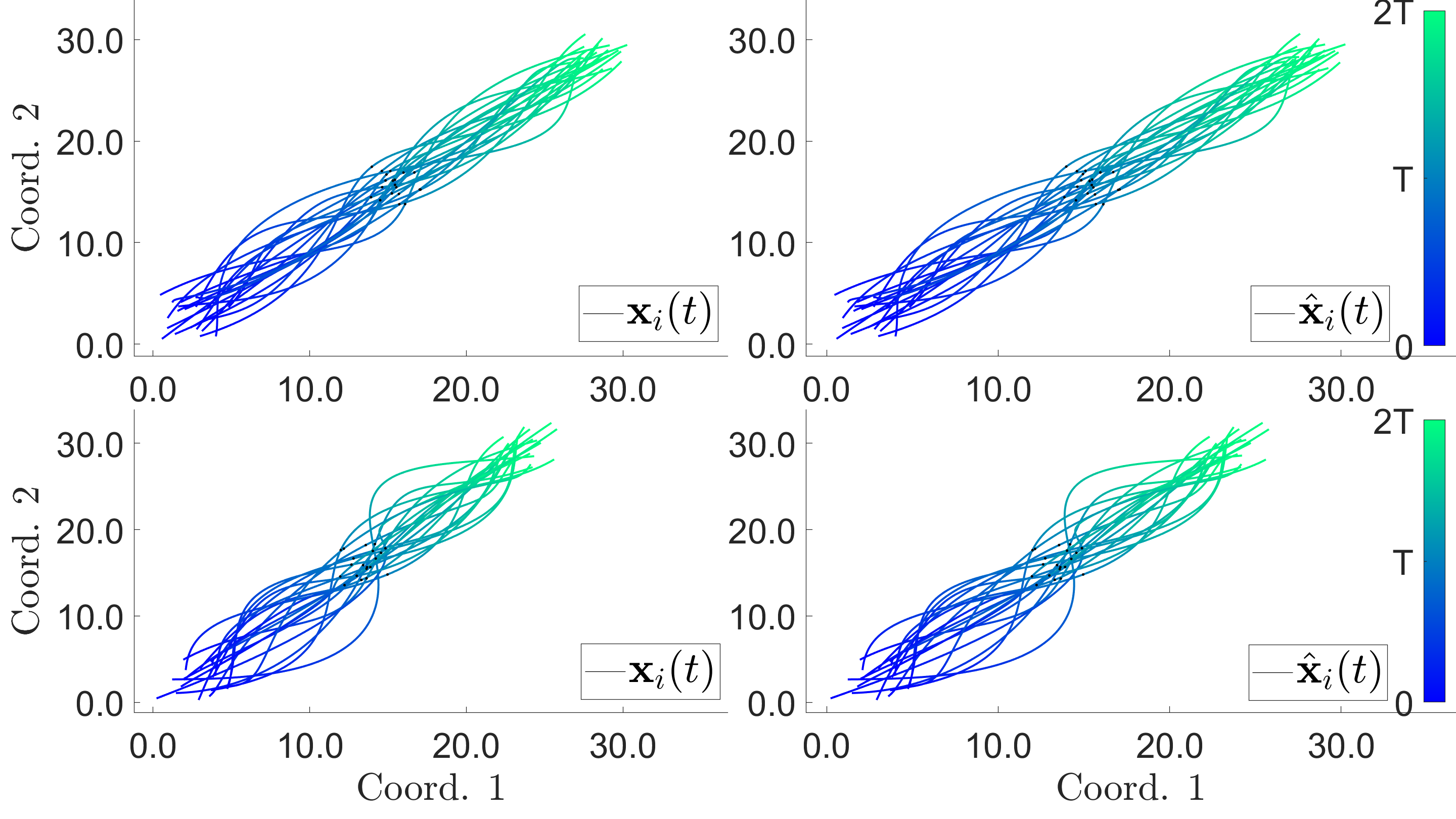}
  \caption{$U(r) = \frac{r^{1.5}}{1.5}$: Trajectory Comparison.}
  \label{fig:AD01_traj}
\end{figure}
Visually, there is no difference between the true dynamics and the estimated dynamics.  We offer more quantitative insight into the difference between the two in table \ref{tab:AD01_traj_err}.
\begin{table}[H]
\centering
\begin{tabular}{ c | c | c }
\hline
                                   & $[0, T]$                                   & $[T, T_f]$ \\
\hline
$\text{mean}_{\text{IC}}$ on $\bx$ & $2.22 \cdot 10^{-3} \pm 8.5 \cdot 10^{-5}$ & $2.4 \cdot 10^{-3} \pm 1.0 \cdot 10^{-4}$ \\
$\text{mean}_{\text{IC}}$ on $\bv$ & $7.8 \cdot 10^{-3} \pm 3.9 \cdot 10^{-4}$  & $1.78 \cdot 10^{-2} \pm 8.9 \cdot 10^{-4}$ \\
$\text{mean}_{\text{IC}}$ on $\by$ & $1.91 \cdot 10^{-5} \pm 9.0 \cdot 10^{-7}$ & $6.2 \cdot 10^{-6} \pm 3.6 \cdot 10^{-7}$ \\
\hline
$\text{std}_{\text{IC}}$ on $\bx$  & $2.8 \cdot 10^{-4} \pm 1.2 \cdot 10^{-5}$  & $3.4 \cdot 10^{-4} \pm 1.5 \cdot 10^{-5}$ \\
$\text{std}_{\text{IC}}$ on $\bv$  & $1.14 \cdot 10^{-3} \pm 7.8 \cdot 10^{-5}$ & $2.7 \cdot 10^{-3} \pm 1.5 \cdot 10^{-4}$ \\
$\text{std}_{\text{IC}}$ on $\by$  & $4.7 \cdot 10^{-6} \pm 3.0 \cdot 10^{-7}$  & $6.2 \cdot 10^{-6} \pm 3.6 \cdot 10^{-7}$ \\
\hline
\hline
$\text{mean}_{\text{IC}}$ on $\bx$ & $2.22 \cdot 10^{-3} \pm 8.4 \cdot 10^{-5}$ & $2.4 \cdot 10^{-3} \pm 1.0 \cdot 10^{-4}$ \\
$\text{mean}_{\text{IC}}$ on $\bv$ & $7.8 \cdot 10^{-3} \pm 3.8 \cdot 10^{-4}$  & $1.78 \cdot 10^{-2} \pm 8.7 \cdot 10^{-4}$ \\
$\text{mean}_{\text{IC}}$ on $\by$ & $1.91 \cdot 10^{-5} \pm 8.6 \cdot 10^{-7}$ & $2.4 \cdot 10^{-5} \pm 1.1 \cdot 10^{-6}$ \\
\hline
$\text{std}_{\text{IC}}$ on $\bx$  & $3.0 \cdot 10^{-4} \pm 2.4 \cdot 10^{-5}$  & $3.4 \cdot 10^{-4} \pm 1.3 \cdot 10^{-5}$ \\
$\text{std}_{\text{IC}}$ on $\bv$  & $1.15 \cdot 10^{-3} \pm 6.8 \cdot 10^{-5}$ & $2.7 \cdot 10^{-3} \pm 1.5 \cdot 10^{-4}$ \\
$\text{std}_{\text{IC}}$ on $\by$  & $4.7 \cdot 10^{-6} \pm 2.6 \cdot 10^{-7}$  & $6.2 \cdot 10^{-6} \pm 3.1 \cdot 10^{-7}$ \\
\hline
\end{tabular}
\caption{$U(r) = \frac{r^{1.5}}{1.5}$: Trajectory Errors.} 

\label{tab:AD01_traj_err} 
\end{table}
We maintain a $3$-digit relative accuracy in estimating the position/velocity of the agents, even though for the interaction kernels, we are only able to maintain a $1$-digit relative accuracy. 
\section{Conclusion and further directions} \label{sec:conclusion}
We have described a second-order model of interacting agents that incorporates multiple agent types, an environment, external forces, and multivariable interaction kernels. The inference procedure described exploits the structure of the system to achieve a learning rate that only depends on the dimension of the interaction kernels, which is much smaller than the full ambient dimension $(2d+1)N$. %
Our estimators are strongly consistent, and in fact have learning rates that are min-max optimal within the nonparametric class, under mild assumptions on the interaction kernels and the system. We described how one can relate the expected supremum error of the trajectories for the system driven by the estimated interaction kernels to the difference between the true interaction kernels and the estimated ones -- this result gives strong support to the use of our weighted $L^2$ norms as the correct way to measure performance and derive estimators. A detailed discussion of the full numerical algorithm, including the inverse problem derived from data and a coercivity condition to ensure learnability, along with complex examples, were presented and we showed how the formulation presented covers a very wide range of systems coming from many disciplines. 

There are various ways that one could build on this work to handle different systems and for many of these further directions, the theoretical framework, techniques, and theorems presented here would be directly useful. In particular, one could consider second-order stochastic systems or a similar system but on a manifold, more complex environments, having more unknowns within the model beyond just the interaction kernels (say estimating the non-collective forces as well), 
identifying the best feature maps to model the data, and considering semiparametric problems where there are hidden parameters within the interaction kernels or other parts of the model that we wish to estimate along with the interaction kernels. 
The generality of the model and its broad coverage of models across the sciences, together with the scalability and performance of the algorithm, could inspire new models -- both explicit equations and nonparametric estimators learned from data -- which are theoretically justified and highly practical.

\acks{MM is grateful for discussions with Fei Lu and Yannis Kevrekidis, and for partial support from NSF-1837991, NSF-1913243, NSF-1934979, NSF-Simons-2031985, AFOSR-FA9550-17-1-0280 and FA9550-20-1-0288, ARO W911NF-18-C-0082, and to the Simons Foundation for the Simons Fellowship for the year '20-'21; ST for support from an AMS Simons travel grant; JM for support from NIH - T32GM11999. Please direct correspondence to any of the first three authors.

All authors jointly designed research and wrote the manuscript; JM and ST derived theoretical results; MZ developed algorithms and applications; JM and MZ analyzed data.}

\appendix

\section{Control of trajectory error}\label{s:trajectoryerrorproof}
\begin{proof}[of Theorem \ref{secondordersystem:TrajDiff_NEW}]
We introduce the function $$F{[\varphi^{EA}]}(\bx,\dot{\bx},\topwde,\topwda):= \varphi^{E}(||\bx||, \sx)\bx + \varphi^{A}(||\bx||,\sxdot)\dot{\bx}$$ defined on $\mathbb{R}^{2d+p^E+p^A}$ for functions $\varphi^E \in L^{\infty}([0,R]\times \mathbb{S}^E), \varphi^A \in L^{\infty}([0,R]\times \mathbb{S}^A)$. Similarly, let $F[\varphi^{\xi}](\bx, \xi,  s^{\xi}):= \varphi^{\xi}(||\bx||,  \topwdxi)\xi$.
%$\varphi^{E}(||\bx||, \sx)\bx + \varphi^{A}(||\bx||,\sxdot)\dot{\bx} \in L^{\infty}([0,\Rx],[0,\Rxdot],[-S_x,S_x],[-S_{\dot{x}},S_{\dot{x}}])$. 
Now we have by assumption that all trajectories 

start from the same initial conditions on both the position and velocity, which implies that $\widehat{\bY}(0) = \bY(0)$ and $ \dot{\widehat{\bY}}(0) = \dot{\bY}(0)$. 
For every $t\in[0,T]$ we have that, by the fundamental theorem of calculus and the triangle inequality,  
\begin{align}  %\bphEA used to be the argument to the \rhsf
\| \bX(t) - \widehat{\bX}(t)\|_{\mathcal{S}}^2 &= \sum_{j=1}^k \sum_{i\in C_j}\frac{1}{N_j}\bigg\Vert\int_0^t \int_0^p (\ddot{\bx_i}-\ddot{\widehat{\bx}}_i) ds dp  \bigg\Vert^2 \nonumber\\
 &\leq tp \sum_{j=1}^K \sum_{i\in C_j}\frac{1}{N_j} \int_{p=0}^t \int_{s=0}^p \| \ddot{\bx_i}-\ddot{\widehat{\bx}}_i \|^2 ds dp \nonumber \\
&= tp \int_{p=0}^t \int_{s=0}^p \| \ddot{\bX} - \rhsfvnc(\coords) - \rhsf^{\bphEA}(\coords) +\rhsfvnc(\coords) \nonumber\\ & \hspace{1cm}+ \rhsf^{\bphEA}(\coords) - \rhsfvnc(\estcoords) - \rhsf^{\bphEA}(\estcoords) \|_{\mathcal{S}}^2 ds dp \nonumber\\
& \leq 2T^2 \int_{p=0}^t \int_{s=0}^p \| \ddot{\bX} - \rhsfvnc(\coords) - \rhsf^{\bphEA}(\coords) \|_{\mathcal{S}}^2 ds dp \label{trajpred:mainbound} \\
& + 2T^2 \int_{p=0}^t \int_{s=0}^p I ds dp + 2T^2 \int_{p=0}^t \int_{s=0}^p \| \rhsfvnc(\coords) - \rhsfvnc(\estcoords) \|_{\mathcal{S}}^2 ds dp. \nonumber
\end{align}
Here we have introduced the term, 
\begin{equation*}
I = \bigg \Vert \rhsf^{\bphEA}(\coords)- \rhsf^{\bphEA}(\estcoords) \bigg \Vert_{\mathcal{S}},
\end{equation*}
which can be expressed explicitly as,
\begin{equation}
I = \Bigg\Vert \bigg( \sum_{j'=1}^K \sum_{i'\in C_{j'}}\frac{1}{N_{j'}}(F{[\phijhat]}(\br_{ii'},\dot{\br}_{ii'},\sxi,\sxidot) - F{[\phijhat]}(\widehat{\br}_{ii'},\dot{\widehat{\br}}_{ii'},\widehat{\sxi},\widehat{\sxidot}) ) \bigg)_{i,j} \Bigg\Vert_{\mathcal{S}}^2.
\end{equation}
Note that in $I$, $j$ is the index of the type among the $\{1,\ldots K\}$ and $i$ indexes within each type $C_j$. This holds similarly in later expressions $I_1,I_2$. 
For the third term of (\ref{trajpred:mainbound}), we exploit the Lipschitz property of the non-collective force:
\begin{align*}
 \| \rhsfvnc(\coords) - \rhsfvnc(\estcoords) \|_{\mathcal{S}}^2 &= \sum_{j=1}^K \sum_{i \in C_j}\frac{1}{N_j} \|\forcev_i(\bx_i, \dot{\bx}_i, \xi_i) - \forcev_i(\widehat{\bx}_i, \dot{\widehat{\bx}}, \widehat{\xi}_i) \|^2 \\
 &\leq \sum_{j=1}^K \sum_{i \in C_j}\frac{1}{N_j} \text{Lip}^2[\forcev_i](\| \bx_i - \widehat{\bx}_i \|^2 + \|\dot{\bx}_i - \dot{\widehat{\bx}}_i \|^2 + \|\xi_i - \widehat{\xi}_i \|^2) \\
 & \leq \max_i \text{Lip}^2[\forcev_i] \|\bY - \widehat{\bY} \|_{\mathcal{Y}}^2
\end{align*}
So that we have the bound
\begin{equation} \label{trajpred:bound1}
2T^2 \int_{p=0}^t \int_{s=0}^p \| \rhsfvnc(\coords) - \rhsfvnc(\estcoords) \|_{\mathcal{S}}^2 ds dp \leq 2T^2 \int_{p=0}^t \int_{s=0}^p \max_i \text{Lip}^2[\forcev_i]  \|\bY - \widehat{\bY} \|_{\mathcal{Y}}^2 ds dp
\end{equation}
First we introduce the convenient notations of 
\begin{equation} \label{trajpred:s-notation}
 \bm s_{i \widehat{i'}}^{E} = \topwde_{(\clof_i,\clof_{i'})}(\bx_i, \dot{\bx}_i, \xi_i, \widehat{\bx}_{i'}, \dot{\widehat{\bx}}_{i'}, \widehat{\xi}_{i'}), \hspace{0.75cm} 
 \widehat{\bm s}_{ii'}^{A} = \topwda_{(\clof_i,\clof_{i'})}(\widehat{\bx}_i, \dot{\widehat{\bx}}_i, \widehat{\xi}_i, \widehat{\bx}_{i'}, \dot{\widehat{\bx}}_{i'}, \widehat{\xi}_{i'})
\end{equation}

with analogous formulae for 
$\bm s_{\widehat{i} i'}^{E},\bm s_{\widehat{i} i'}^{A},\bm s_{i \widehat{i'}}^{A},\bm s_{i \widehat{i'}}^{\xi}, \bm s_{\widehat{i} i'}^{\xi},\widehat{\bm s}_{ii'}^{E}, \widehat{\bm s}_{ii'}^{A}, \widehat{\bm s}_{ii'}^{\xi} $. 
Now we break up $I$ using the triangle inequality and get that $I \leq I_1 + I_2$ where
$$
I_1 =\Bigg\Vert  \bigg( \sum_{j'=1}^K \sum_{i'\in C_{j'}}\frac{1}{N_{j'}} (F{[\phijhat]}(\br_{ii'},\dot{\br}_{ii'},\sxi,\sxidot) - F{[\phijhat]}(\bx_i - \widehat{\bx_{i'}}, \dot{\bx_i} - \dot{\widehat{\bx_i}},\bm{s}_{i\widehat{i'}}^{E},\bm{s}_{i\widehat{i'}}^{A} )) \bigg )_{i,j}
 \Bigg\Vert_{\mathcal{S}}^2
$$ 
$$
I_2 = \Bigg\Vert \bigg( \sum_{j'=1}^K \sum_{i'\in C_{j'}}\frac{1}{N_{j'}} (F[\phijhat](\bx_i - \widehat{\bx_{i'}}, \dot{\bx_i} - \dot{\widehat{\bx_i}},\bm{s}_{i\widehat{i'}}^{E},\bm{s}_{i\widehat{i'}}^{A}) - F_{[\phijhat]}(\widehat{\br_{ii'}}, \dot{\widehat{\br_{ii'}}}, \widehat{\sxi},\widehat{\sxidot}))) \bigg )_{i,j}
\Bigg\Vert_{\mathcal{S}}^2
$$
So using the Lipschitz property of $F[\phijhat]$ we get that, since
$$
I_1 = \sum_{j=1}^K \sum_{i\in C_{j}}\frac{1}{N_{j}}  \Bigg| \sum_{j'=1}^K \sum_{i'\in C_{j'}}\frac{1}{N_{j'}} (F{[\phijhat]}(\br_{ii'},\dot{\br}_{ii'},\sxi,\sxidot) - F{[\phijhat]}(\bx_i - \widehat{\bx_{i'}}, \dot{\bx_i} - \dot{\widehat{\bx_i}},\bm{s}_{i\widehat{i'}}^{E},\bm{s}_{i\widehat{i'}}^{A} )) 	
 \Bigg|^2,
$$ 
then, 
\begin{align*}
I_1 \leq K \sum_{j=1}^K\sum_{i\in C_j}\frac{1}{N_j}\sum_{j'=1}^K\sum_{i'\in C_{j'}}\frac{1}{N_{j'}}&\bigg|\bigg(\text{Lip}[F[\phijhat]]\|(\bx_{i'} - \widehat{\bx_{i'}}, \dot{\bx}_{i'} - \dot{\widehat{\bx}}_{i'}\bm{s}_{ii'}^{E} - \bm{s}_{i\widehat{i'}}^{E} , \bm{s}_{ii'}^{A} - \bm{s}_{i\widehat{i'}}^{A} )\|  \bigg) \bigg|^2.
\end{align*}
By the assumptions on the feature maps, we have that
\begin{align*}
\| \bm{s}_{ii'}^{E} - \bm{s}_{i\widehat{i'}}^{E} \| \leq \text{Lip}[\topwde_{(\clof_i,\clof_{i'})}]\|(\bx_{i'} - \widehat{\bx_{i'}}, \dot{\bx}_{i'} - \dot{\widehat{\bx}}_{i'},\xi_{i'} - \widehat{\xi}_{i'}) \| \\
\| \sxidot - \bm{s}_{i\widehat{i'}}^{A}\| \leq \text{Lip}[\topwda_{(\clof_i,\clof_{i'})}]\|(\bx_{i'} - \widehat{\bx_{i'}},\dot{\bx}_{i'} - \dot{\widehat{\bx}}_{i'},\xi_{i'} - \widehat{\xi}_{i'}) \| 
\end{align*}
Combining these bounds we see that, 
\begin{align} 
I_1 \leq K \sum_{j=1}^K\sum_{i\in C_j} \frac{1}{N_j} \sum_{j'=1}^K\sum_{i'\in C_{j'}}\frac{1}{N_{j'}} &\Big(\max_{j,j'}\Big(\text{Lip}[F[\phijhat](\text{Lip}[\topwde_{(j,j')}]+1),\text{Lip}[F[\phijhat](\text{Lip}[\topwda_{(j,j')}]+1) \Big)\Big)^2 \nonumber \\ 
&(\|\bx_{i'} - \widehat{\bx_{i'}} \| +  \|\dot{\bx}_{i'} - \dot{\widehat{\bx}}_{i'}\| + \|\xi_{i'} - \widehat{\xi}_{i'} \|)^2. \label{trajpred:I_1Bound}
\end{align}

Let $\tilde{S}= \max(S_E, S_A)^2$, 
$J=(\max_{j,j'}\text{Lip}[\topwde_{(j,j')},\topwda_{(j,j')}]+1)^2$, 
and then let $P=\tilde{S}J$ and we get by Young's inequality that,
\begin{equation}
I_1 \leq 4KP\|\bY - \widehat{\bY} \|_{\mathcal{Y}}^2,  \label{trajPred:FinalI1bound}
\end{equation}
and performing a similar analysis we get that 
$$I_2 \leq 4KP \|\bY - \widehat{\bY} \|_{\mathcal{Y}}^2. $$
So gathering terms, we can reexpress (\ref{trajpred:mainbound}) as 
\begin{align}
\| \bX(t) - \widehat{\bX}(t)\|_{\mathcal{S}}^2 &\leq
2T^2 \int_{p=0}^t \int_{s=0}^p \| \ddot{\bX} - \rhsfvnc(\coords) - \rhsf^{\bphEA}(\coords) \|_{\mathcal{S}}^2 ds dp \nonumber \\
&+ 2T^2 (F + 8KP) \int_{p=0}^t \int_{s=0}^p  \|\bY - \widehat{\bY} \|_{\mathcal{Y}}^2 ds dp \label{trajpred:Xbound}
\end{align}
where $F = \max_{i}\text{Lip}[\forcev_i]$.
Performing an analogous analysis on $\| \bV(t) - \widehat{\bV}(t)\|_{\mathcal{S}}^2 ,\| \bXi(t) - \widehat{\bXi}(t)\|_{\mathcal{S}}^2$ , 
with some additional effort, one can get the following result on the phase variable 
\begin{align}
\| \bXi(t) - \widehat{\bXi}(t)\|_{\mathcal{S}}^2 & \leq 2T(8QK + F^{\xi}) \int_{s=0}^t \|\bY - \widehat{\bY} \|_{\mathcal{Y}}^2 ds \nonumber \\
&+ 2T\int_{s=0}^t \|\dot\bXi - \rhsfxinc(\bX, \bV, \bXi) +  \rhsf^{\bintkernelxi}(\bX, \bV, \bXi) \|_{\mathcal{S}}^2 ds \label{trajpred:Xibound}
\end{align}
where $F^{\xi} = \max_{i}\text{Lip}[\forcexi_i]$ and $Q = \max(H, S^{\xi})$ where 
$H = \max_{j,j'}\text{Lip}[\topwde_{(j,j')},\topwda_{(j,j')}]$. 
Similarly, we have that, 
\begin{align}
\| \bV(t) - \widehat{\bV}(t)\|_{\mathcal{S}}^2 &\leq 2T \int_{s=0}^t  \| \ddot{\bX} - \rhsfvnc(\coords) - \rhsf^{\bphEA}(\coords) \|_{\mathcal{S}}^2 ds \nonumber \\
& +2T(F + 8KP)\int_{s=0}^t \|\bY - \widehat{\bY} \|_{\mathcal{Y}}^2 ds \label{trajpred:Vbound}
\end{align}

Gathering the bounds (\ref{trajpred:Xbound}, \ref{trajpred:Xibound}, \ref{trajpred:Vbound}), we have that

\begin{align*}
\Vert \widehat{\bY}(t)- \bY(t)\Vert_{\mathcal{Y}}^2 &\leq
2T(8KP + F + 8QK + F^{\xi})\int_{s=0}^t \Vert \widehat{\bY}- \bY\Vert_{\mathcal{Y}}^2 ds \\
& + 2T^2(8KP + F)\int_{p=0}^t \int_{s=0}^p \Vert \widehat{\bY}- \bY\Vert_{\mathcal{Y}}^2 ds dp \\
a(t) & \begin{dcases} &+ 2T^2 \int_{p=0}^t \int_{s=0}^p \| \ddot{\bX} - \rhsfvnc(\coords) - \rhsf^{\bphEA}(\coords) \|_{\mathcal{S}}^2 ds dp\\
&+ 2T  \int_{s=0}^t  \| \ddot{\bX} - \rhsfvnc(\coords) - \rhsf^{\bphEA}(\coords) \|_{\mathcal{S}}^2 ds \\
&+ 2T\int_{s=0}^t \| \dot{\bXi} - \rhsfxinc(\coords) - \rhsf^{\widehat{\bintkernel}^{\xi}}(\coords) \|_{\mathcal{S}}^2 ds \bigg]
\end{dcases}
\end{align*}

where we denote the last three lines by $a(t)$ and notice that this is a nondecreasing function in $t$. We also denote $A_1 = 2T(8KP + F + 8QK + F^{\xi})$ and $B_1 = 2T^2(8KP + F)$. Now use theorem \ref{IteratedGronwall}, which is in \cite{ChoYeol2007} and is originally in Bainov and Simeonov. With this notation, we can rewrite the above bound as
\begin{equation}
\Vert \widehat{\bY}(t)- \bY(t)\Vert_{\mathcal{Y}}^2 \leq A_1 \int_{s=0}^t \Vert \widehat{\bY}- \bY\Vert_{\mathcal{Y}}^2 ds  + B_1 \int_{p=0}^t \int_{s=0}^p \Vert \widehat{\bY}- \bY\Vert_{\mathcal{Y}}^2 ds dp +a(t) 
\label{trajpred:Bainov}
\end{equation}
And so in the notation of Theorem \ref{IteratedGronwall} we have $u(t) = \|\widehat{\bY}(t) - \bY(t) \|_{\mathcal{Y}}^2$,  $b(t) = 1$, $k_1(t,t_1) = A_1$ and $k_2(t,t_1,t_2) = B_1$, so that for all $t$ we have 
\begin{align*}
\Vert \widehat{\bY}(t)- \bY(t)\Vert_{\mathcal{Y}}^2 \leq a(t) + \int_0^t \widehat{R}[a](t,s)\exp\bigg(\int_s^t \widehat{R}[b](t,\tau) d\tau\bigg)ds
\end{align*}
and we have the simple bounds
\begin{align}
\widehat{R}[a](t,s) = a(t) + \int_0^s B_1 a(t_2)dt_2 \leq a(T) + B_1Ta(T) \\
\widehat{R}[b](t,\tau) = A_1+ \int_0^{\tau}1dy = A_1 + \tau
\end{align}
So that, 
\begin{align*}
\Vert \widehat{\bY}(t)- \bY(t)\Vert_{\mathcal{Y}}^2 &\leq 
a(T) + [a(T) + B_1Ta(T)]\int_{s=0}^t \exp\bigg(\int_s^t (A_1 + \tau)d\tau\bigg)ds \\
& \leq a(T) + [a(T) + B_1Ta(T)] \int_{s=0}^T\exp\bigg(\int_0^T A_1+\tau d\tau\bigg)ds \\
&= a(T) + [a(T) + B_1Ta(T)]T\exp(A_1T + T^2/2) \\
&= a(T)(1+(1+B_1T)T\exp(A_1T + T^2/2))
\end{align*}
So that we can immediately conclude the first assertion of the theorem, 
$$\sup_{t\in [0,T]}\Vert \widehat{\bY}(t)- \bY(t)\Vert_{\mathcal{Y}}^2 \leq a(T)(1+(T+B_1T^2)\exp(A_1T + T^2/2))$$
Lastly, we can use the results of section \ref{app:continuity} to get the key result on the expected supremum error. We take expectation on each of the three terms of $a(T)$ and normalize them so they are in the form of the results of \ref{app:continuity}. 
\begin{align}
\frac{1}{T^2}\int_{p=0}^T \int_{s=0}^T \mathbb{E}_{\bmu} \| \ddot{\bX} - \rhsfvnc(\coords) - \rhsf^{\bphEA}(\coords) \|_{\mathcal{S}}^2 ds dp \\
< K^2 \Vert \bphEA - \bpEA \Vert_{\LtwoB(\brhoEA)}^2 
\end{align}
We similarly get that,
\begin{equation}
\frac{1}{T}\int_{s=0}^T \mathbb{E}_{\bmu} \| \dot{\bXi} - \rhsfxinc(\coords) - \rhsf^{\widehat{\bintkernel}^{\xi}}(\coords) \|_{\mathcal{S}}^2 ds < K^2 \Vert \bphxi - \bpxi \Vert_{\LtwoB(\brhoxi)}^2,
\end{equation}
and can get an analogous bound for the remaining term of $a(T)$. 
These bounds together lead to 
$$a(T) \leq (2T^4K^2 +2T^2K^2)\Vert \bphEA - \bpEA \Vert_{\LtwoB(\brhoEA)}^2  + 2T^2K^2\Vert \bphxi - \bpxi \Vert_{\LtwoB(\brhoxi)}^2 $$
which implies the desired result. 
\end{proof}

\section{Learning theory - technical tools} \label{sec:app:learntechnical}

\subsection{Continuity of the error functionals} \label{app:continuity}
For any $t\in [0,T]$, consider the two random variables,
\begin{align} \label{eq:Error_Func_RV:EA}
\bm{\mathcal{E}}_{\bX(t)}^{EA}(\bvpEA) &= \Big\Vert\ddot{\bX}(t) -  \rhsfvnc(\coordsT) \nonumber \\ 
&- \rhsf^{\bvpE}(\coordsT) - \rhsf^{\bvpA}(\coordsT)\Big\Vert^2_{\mathcal{S}} 
\end{align}
\begin{align} \label{eq:Error_Func_RV:xi}
\bm{\mathcal{E}}_{\bXi(t)}^{\xi}(\bvpxi) = \Big\Vert\dot{\bm{\Xi}}(t) - \rhsfxinc(\coordsT) - \rhsf^{\bvpxi}(\coordsT)\Big\Vert^2_{\mathcal{S}}
\end{align}
These will be used in various places throughout the technical proofs and easily relate to the natural error functionals defined in \eqref{eq:Error_Func:EA}, \eqref{eq:Error_Func:xi} as, 
$$ \bm{\mathcal{E}}_{\infty}^{EA}(\bvpEA)=\frac{1}{L} \sum_{l=1}^{L} \mathbb{E}_{\bmu}\left[\bm{\mathcal{E}}_{\bX(t_{l})}^{EA}(\bvpEA)\right], \qquad \bm{\mathcal{E}}_{\infty}^{\xi}(\bvpxi)=\frac{1}{L} \sum_{l=1}^{L} \mathbb{E}_{\bmu}\left[\bm{\mathcal{E}}_{\bXi(t_l)}^{\xi}(\bvpxi)\right]. $$

We begin by establishing basic continuity results for our error functionals over the hypothesis space. 
The specific structure of the governing equations plays a critical role in the analysis.
\subsection*{Alignment and energy based kernels}
\begin{proposition}  \label{proposition:EAbound}
For $\combfbvphEbvphA, \combfbphEbphA \in \bhypspaceEA $ the true and empirical error functionals are bounded as follows,   
\begin{align}
|\mbf{\mathcal{E}}_{\infty}^{EA}(\combfbvphEbvphA) - \mbf{\mathcal{E}}_{\infty}^{EA}(\combfbphEbphA)| &\leq  K^2 \|\combfbvphEbvphA -  \combfbphEbphA\|_{\LtwoB(\bm{\rho}_T^{EA,L})} \|2\combfbpEbpA  - \combfbvphEbvphA -  \combfbphEbphA \|_{\LtwoB(\bm{\rho}_T^{EA,L})}\label{2ndordersystem:expectationerrorfunctional}\\
|\mbf{\mathcal{E}}_{M}^{EA}(\combfbvphEbvphA) - \mbf{\mathcal{E}}_{M}^{EA}(\combfbphEbphA)| &\leq  K^4  \max \{R, R_{\dot{x}} \}^2 \|\combfbvphEbvphA -  \combfbphEbphA  \|_{\infty}  \|2
\combfbpEbpA  - \combfbvphEbvphA -  \combfbphEbphA \|_{\infty}\label{2ndordersystem:empiricalerrorfunctional}
\end{align}
Recall the definitions of $R,R_{\dot{x}}$ in equations \eqref{eq:Rdef}, and \eqref{eq:Rdyndef}. 
\end{proposition}
\begin{AdditionalTxt}
\begin{proof}
Using Jensen's inequality,
\begin{align}
&|\bm{\mE}_{\bX(t)}^{EA}(\combfbvphEbvphA)-\bm{\mE}_{\bX(t)}^{EA}(\combfbphEbphA)| \nonumber \\ 
&= \bigg| \sum_{k=1}^{K}\frac{1}{N_k} \sum_{i \in C_k} \big\langle \sum_{k'=1}^{K}\frac{1}{N_{k'}}\sum_{i'\in C_{k'}}( \widehat \intkernelvar_{kk'}^{E}-\widehat \intkernel_{kk'}^{E})(r_{ii'},\topwde_{ii'}){\br}_{ii'}+ (\widehat \intkernelvar_{kk'}^{A}-\widehat \intkernel_{kk'}^{A})(r_{ii'},\topwda_{ii'}){\dot\br}_{ii'}, \nonumber 
\\& \sum_{k''=1}^{K}\frac{1}{N_{k''}} \sum_{i'\in C_{k''}}(2\intkernel_{kk'}^{E}-\widehat \intkernelvar_{kk'}^{E}-\widehat \intkernel_{kk'}^{E})(r_{ii'},\topwde_{ii'}){\br}_{ii'}+(2\intkernel_{kk'}^{A}-\widehat \intkernelvar_{kk'}^{A}-\widehat \intkernel_{kk'}^{A})(r_{ii'},\topwda_{ii'})\dot{\br}_{ii'}
\big\rangle \bigg| \nonumber 
\\&\leq \sum_{k=1}^{K}
\sum_{k'=1}^{K} \sum_{k''=1}^{K}\frac{1}{N_k}\sum_{i \in C_k} \| \frac{1}{N_{k'}}\sum_{i' \in C_{k'}}( \widehat \intkernelvar_{kk'}^{E}-\widehat \intkernel_{kk'}^{E})(r_{ii'},\topwde_{ii'}){\br}_{ii'}+ (\widehat \intkernelvar_{kk'}^{A}-\widehat \intkernel_{kk'}^{A})(r_{ii'},\topwda_{ii'}){\dot\br}_{ii'} \|\\& \| \frac{1}{N_{k''}}\sum_{i' \in C_{k''}}(2\intkernel_{kk'}^{E}-\widehat \intkernelvar_{kk'}^{E}-\widehat \intkernel_{kk'}^{E})(r_{ii'},\topwde_{ii'}){\br}_{ii'}+(2\intkernel_{kk'}^{A}-\widehat \intkernelvar_{kk'}^{A}-\widehat \intkernel_{kk'}^{A})(r_{ii'},\topwda_{ii'})\dot{\br}_{ii'}
 \| \nonumber 
 \end{align}
\begin{align}
&< \sum_{k=1}^{K}\sum_{k'=1}^{K} \sum_{k''=1}^{K}\sqrt{\frac{1}{N_kN_{k'}} \sum_{i \in C_k, i'\in C_{k'}} \|(\widehat \intkernelvar_{kk'}^{E}-\widehat \intkernel_{kk'}^{E})(r_{ii'},\topwde_{ii'}){\br}_{ii'}+ (\widehat \intkernelvar_{kk'}^{A}-\widehat \intkernel_{kk'}^{A})(r_{ii'},\topwda_{ii'}){\dot\br}_{ii'}}\|^2 \times \nonumber\\ 
&\qquad\sqrt{\frac{1}{N_kN_{k''}}\sum_{i\in C_{k},i' \in C_{k''}}\|(2\intkernel_{kk'}^{E}-\widehat \intkernelvar_{kk'}^{E}-\widehat \intkernel_{kk'}^{E})(r_{ii'},\topwde_{ii'}){\br}_{ii'}+(2\intkernel_{kk'}^{A}-\widehat \intkernelvar_{kk'}^{A}-\widehat \intkernel_{kk'}^{A})(r_{ii'},\topwda_{ii'})\dot{\br}_{ii'}\|^2}\nonumber\\ 
&< \sum_{k=1}^{K}\sum_{k'=1}^{K} \sum_{k''=1}^{K}  \| ( \widehat{\varphi}_{kk'}^{EA} - \widehat{\phi}_{kk'}^{EA}) \|_{\Ltwo(\hat\rho_{T}^{t, kk'})} 
\| 2(\phi_{kk'}^{EA} - \widehat{\varphi}_{kk'}^{EA} - \widehat{\varphi}_{kk'}^{EA})\|_{\Ltwo(\hat\rho_{T}^{t, kk''})} \nonumber \\
&\leq K^2\| \combfbvphEbvphA -  \combfbphEbphA \|_{\LtwoB(\hat\brho_T^t)} \|2\combfbpEbpA - \combfbvphEbvphA -  \combfbphEbphA\|_{\LtwoB(\hat\brho_T^t)},
\end{align}
where 
\begin{equation*}
\smash{\widehat\rho}_{T}^{ t, kk'}(r,\topwde,\dot{r}, \topwda)=\displaystyle \frac{1}{LN_{kk'}}\sum_{l= 1}^L\! \sum_{\substack{i \in C_{k}, i' \in C_{k'} \\ i\neq i'}} \delta_{r_{ii'}(t), \topwde_{ii'}(t_l), \dot{r}_{ii'}(t_l), \topwda_{ii'}(t_l)}(r,\topwde,\dot{r}, \topwda)
\end{equation*}
and $\widehat\brho_{T}^{t}=\bigoplus_{k,k'=1,1}^{K,K}\widehat\rho_{T}^{t, kk'}$.
Therefore, we have that 
\begin{align}%\label{deviation}
&\big |\frac{1}{L}\sum_{l=1}^{L}\bm{\mathcal{E}}_{\bX(t_l)}^{EA}(\combfbvphEbvphA)-\frac{1}{L}\sum_{l=1}^{L}\bm{\mathcal{E}}_{\bX(t_l)}^{EA}(\combfbphEbphA) \big| \leq \frac{1}{L}\sum_{l=1}^{L} \big |\bm{\mathcal{E}}_{\bX(t_l)}^{EA}(\combfbphEbphA)- \bm{\mathcal{E}}_{\bX(t_l)}^{EA}(\combfbvphEbvphA) \big|   \nonumber \\
&<\frac{K^2}{L}\sum_{l=1}^{L}\| \combfbvphEbvphA -  \combfbphEbphA \|_{\LtwoB(\widehat\brho_T^{t_l})} \|2\combfbpEbpA - \combfbvphEbvphA -  \combfbphEbphA \|_{\LtwoB(\widehat\brho_T^{t_l})}\nonumber 
\\&\leq K^2\sqrt{\frac{1}{L}\sum_{l=1}^{L}\|\combfbvphEbvphA -  \combfbphEbphA \|^2_{\LtwoB(\widehat\brho_T^{t_l})}}\sqrt{\frac{1}{L}\sum_{l=1}^{L}\|2\combfbpEbpA - \combfbvphEbvphA -  \combfbphEbphA\|^2_{\LtwoB(\widehat\brho_T^{t_l})}} \nonumber 
\\&=K^2 \|\combfbvphEbvphA -  \combfbphEbphA \|_{\bL^2(\widehat{\mbf{\rho}}_T^L)} \|2\combfbpEbpA - \combfbvphEbvphA -  \combfbphEbphA\|_{\bL^2(\widehat{\mbf{\rho}}_T^L)} \label{usefuleq1}\\
%&\leq K^4 \|\combfbvphEbvphA -  \combfbphEbphA \|_{\infty} \|2\combfbpEbpA - \combfbvphEbvphA -  \combfbphEbphA\|_{\infty}\label{usefuleq2}
%\\
&\leq K^4  \max \{R, R_{\dot{x}} \}^2 \| \combfbvphEbvphA -  \combfbphEbphA \|_{\infty} \|2\combfbpEbpA - \combfbvphEbvphA -  \combfbphEbphA \|_{\infty}\label{usefuleq3}
\end{align}
Taking the expectation with respect to $\bmu$ on each side of (\ref{usefuleq1}) we get the first inequality. The second inequality follows by noticing that,  
$$ |\mbf{\mathcal{E}}_{M}^{EA}(\combfbvphEbvphA)-\mbf{\mathcal{E}}_{M}^{EA}(\combfbphEbphA)| \leq \frac{1}{M}\sum_{m=1}^{M} \big|\frac{1}{L}\sum_{l=1}^{L}\mbf{\mathcal{E}}_{\bX^{(m)}(t_l)}(\combfbvphEbvphA)-\frac{1}{L}\sum_{l=1}^{L}\mbf{\mathcal{E}}_{\bX^{(m)}(t_l)}(\combfbphEbphA)\big|.$$

\end{proof}

\end{AdditionalTxt}

\subsection*{Environment interaction kernels }
Here we show an analogous result to the alignment and energy result above. The techniques are similar and the result serves an identical purpose in the theory. Recall the definition of $R_{\xi}$ in \eqref{eq:Rdyndef:xi}.
\begin{proposition}
For  $\widehat \bintkernelvar,\widehat \bintkernel  \in \bhypspace^{\xi}$, we have
\begin{align}
|\mbf{\mathcal{E}}_{\infty}^{\xi}(\widehat\bintkernelvar)-\mbf{\mathcal{E}}_{\infty}^{\xi}(\widehat \bintkernel)| &\leq  K^2 \|\widehat \bintkernelvar-\widehat \bintkernel \|_{\LtwoB(\bm\rho_T^{\xi,L})} \|2\bintkernel^{\xi}-\widehat \bintkernelvar-\widehat \bintkernel \|_{\LtwoB(\bm\rho_T^{\xi,L})}\label{2ndordersystem:expectationerrorfunctionalxi}\\
 |\mbf{\mathcal{E}}_{M}^{\xi}(\widehat \bintkernelvar)-\mbf{\mathcal{E}}_{M}^{\xi}(\widehat \bintkernel)|&\leq K^4 R_{\xi}^2 \| \widehat \bintkernelvar-\widehat \bintkernel \|_{\infty}\| 2\bintkernel -\widehat \bintkernelvar-\widehat \bintkernel \|_{\infty}\label{2ndordersystem:empiricalerrorfunctionalxi}
\end{align}
\end{proposition}

The following lemma can be immediately deduced using \eqref{2ndordersystem:expectationerrorfunctional}, \eqref{2ndordersystem:empiricalerrorfunctional}
\begin{AdditionalTxt}
, and \eqref{usefuleq3} 
\end{AdditionalTxt}
.

\begin{lemma} \label{berstein} For all $\combfbvpEbvpA \in \bhypspaceEA$, define the defect function $L_M^{EA}(\combfbvpEbvpA)$ as 
\begin{equation}
L_{M}^{EA}(\combfbvpEbvpA)=\mbf{\mathcal{E}}_{\infty}^{EA}(\combfbvpEbvpA)-\mbf{\mathcal{E}}_{M}^{EA}(\combfbvpEbvpA).
\label{e:defectFunction}
\end{equation}
Then, given two functions $\bvpEA_1,  \bvpEA_2\in \bhypspaceEA$, the defect function is bounded by   
\begin{align*}
|L_{M}^{EA}(\bvpEA_1)-L_{M}^{EA}(\bvpEA_2)| &\leq  2K^4 \max \{R, R_{\dot{x}} \}^2 \|\bvpEA_1 - \bvpEA_2 \|_{\infty} \| \bvpEA_1 + \bvpEA_2 -2\combfbpEbpA\|_{\infty}
\end{align*}
almost surely with respect to $\bmu$. 
\end{lemma}

A similar lemma can be immediately deduced on the $\xi$ variable .
\begin{lemma}
For all $\bvxi \in \bhypspacexi$, define the defect function $L_M^{\xi}(\bvxi)$ as 
\begin{equation}
L_{M}^{\xi}(\bvxi)=\mbf{\mathcal{E}}_{\infty}^{\xi}(\bvxi)-\mbf{\mathcal{E}}_{M}^{\xi}(\bvxi).
\label{e:defectFunctionxi}
\end{equation}
Then, given two functions $\bvxi_1, \bvxi_2 \in \bhypspacexi$, the defect function is bounded by   
\begin{align*}
|L_{M}^{\xi}(\bvxi_1)-L_{M}^{\xi}(\bvxi_2)| &\leq  2K^4 R_{\xi}^2 \|\bvxi_1 - \bvxi_2 \|_{\infty} \| \bvxi_1 + \bvxi_2 -2\bpxi\|_{\infty}
\end{align*}
almost surely with respect to $\bmu$. 
\end{lemma}

 \subsection{Uniqueness of minimizers over a compact convex space}
Recall the energy and alignment bilinear functional $\dbinnerp {\cdot, \cdot}_{EA}$, previously defined in equation \eqref{eq:bilinearFn} 
\begin{align*}
\dbinnerp {\bvpEA_1, \bvpEA_2}_{EA}:=\frac{1}{L}\sum_{l=1}^{L}\mathbb{E}_{\bmu}  \bigg[\left\langle \rhsfo_{\bvpEA_1}(\coordsTL) , \rhsfo_{\bvpEA_2}(\coordsTL) \right\rangle_{\mathcal{S}} \bigg],
\end{align*} 
for any $\bvpEA_1, \bvpEA_2 \in \bhypspaceEA$. The $\mathcal{S}$-inner product is the inner product induced by the $\| \cdot \|_{\mathcal{S}}$ norm by the polarization identity, which holds as we are working in an $\Ltwo$ space, so the parallelogram law holds. 
Then our coercivity condition  \eqref{2ndorder:gencoer}.
can be written in terms of this bilinear functional as: for all $\combfbvpEbvpA \in \bhypspaceEA$
\[
c_{\bhypspaceEA} \norm{\combfbvpEbvpA}^2_{\bm{L}^2(\bm{\rho}_T^{EA,L})}\leq  \dbinnerp {\combfbvpEbvpA, \combfbvpEbvpA} 
\]

\begin{proposition}\label{2ndordersystem:convexity}
Let the minimizer of the error functional be denoted
 $$\widehat{\bintkernel}_{L,\infty, \bhypspaceEA}^{EA}:=\combf{\widehat{\bintkernel}_{L,\infty, \bhypspaceEA}^E}{\widehat{\bintkernel}_{L,\infty, \bhypspaceEA}^A} :=\argmin{\combfbvpEbvpA \in \bhypspaceEA} \mbf{\mathcal{E}}_{\infty}^{EA}(\combfbvpEbvpA);$$ 
  then for all $\combfbvpEbvpA \in \bhypspaceEA$, the difference of the error functional at this element of $ \bhypspaceEA$ and the minimizer is lower bounded as, 
 \begin{equation}\label{eq_minH}
 \mbf{\mathcal{E}}_{\infty}^{EA}(\combfbvpEbvpA)-\mbf{ \mathcal{E}}_{\infty}^{EA}(\widehat{\bintkernel}_{L,\infty, \bhypspaceEA}^{EA}) \geq c_{\bhypspaceEA} \| \combfbvpEbvpA - \widehat{\bintkernel}_{L,\infty, \bhypspaceEA}^{EA} \|_{\LtwoB(\brhoEAL) }^2\,.
 \end{equation}
Thus, the minimizer of $\mbf{\mathcal{E}}_{\infty}^{EA}$ over $\bhypspaceEA$ is unique in $\bm{L}^2(\bm{\rho}_T^{EA,L})$. 
\end{proposition} 
\begin{AdditionalTxt}
\begin{proof} For $\combfbvpEbvpA \in \bhypspaceEA$, and to ease the notation let $\widehat{\bintkernel}^{EA} := \widehat{\bintkernel}_{L,\infty, \bhypspaceEA}^{EA}$,  we have 
\begin{align}
\mbf{\mathcal{E}}_{\infty}^{EA}(\combfbvpEbvpA)- \mbf{\mathcal{E}}_{\infty}^{EA}(\widehat{\bintkernel}^{EA})
&=\dbinnerp{(\bvpEA - \bpEA), (\bvpEA - \bpEA) } \nonumber \\ &- \dbinnerp{(\bphEA - \bpEA), (\bphEA - \bpEA)} \label{interEq1}
\end{align}
Using that $\dbinnerp{X,X} - \dbinnerp{Y,Y} = \dbinnerp{X-Y, X+Y}$, which holds by bilinearity and the definition of the form, we get that  
\begin{align}
(\ref{interEq1}) &= \dbinnerp{ (\bvpEA - \bphEA), (\bvpEA + \bphEA - 2\bpEA } \nonumber \\
&= \dbinnerp{(\bvpEA - \bphEA), (\bvpEA - \bphEA + 2(\bphEA - \bpEA))} \nonumber \\
&= \dbinnerp{(\bvpEA - \bphEA), (\bvpEA - \bphEA)} + 2\dbinnerp{(\bvpEA - \bphEA), (\bphEA - \bpEA)} \label{intereq2}
\end{align}
By the coercivity condition, the first term in (\ref{intereq2}) is at least as large as $$ c_{\bhypspaceEA}\|\combfbvpEbvpA - \combfbphEbphA \|_{\bm{L}^2(\bm{\rho}_T^{EA,L})} \geq 0 $$
We are left to show the second term in (\ref{intereq2}) is nonnegative. Since $\bhypspaceEA$ is convex, for all $t \in [0,1]$, $t (\combfbvpEbvpA) + (1-t) (\combfbphEbphA) \in \bhypspaceEA$. By definition of $\combfbphEbphA$ as an argmin, 
$$\mbf{\mathcal{E}}_{\infty}^{EA}(t\combfbvpEbvpA + (1-t)\combfbphEbphA) - \mbf{\mathcal{E}}_{\infty}^{EA}(\combfbphEbphA) \geq 0 $$
which means, using a decomposition analogous to the one above in (\ref{intereq2}), that
\begin{align*}
\langle\hspace{-1mm}\langle (t\bvpEA + (1-t)\bphEA - \bphEA), (t\bvpEA + (1-t)\bphEA - \bphEA + 2(\bphEA - \bpEA) \rangle\hspace{-1mm}\rangle \\
= \langle\hspace{-1mm}\langle t((\bvpEA - \bphEA)), (t\bvpEA + (2-t)\bphEA - 2\bpEA) \rangle\hspace{-1mm}\rangle \geq 0
\end{align*}
So that, 
\begin{align}
t\langle\hspace{-1mm}\langle (\bvpEA - \bphEA), (t\bvpEA + (2-t)\bphEA - 2\bpEA) \rangle\hspace{-1mm}\rangle \geq 0 \label{finalintereq} \\
\Leftrightarrow \langle\hspace{-1mm}\langle (\bvpEA - \bphEA), (t\bvpEA + (2-t)\bphEA - 2\bpEA) \rangle\hspace{-1mm}\rangle \geq 0
\end{align}
By the results of section \ref{app:continuity}, we have (Lipschitz) continuity of the bilinear functional $\dbinnerp{\cdot, \cdot}$ over $\bhypspaceEA \times \bhypspaceEA$. Next, take the $\lim_{t\to 0^+}$ of (\ref{finalintereq}) and by the dominated convergence theorem (which holds due to the boundedness and continuity assumptions on the kernels) we pass the limit through the expectations in $\dbinnerp{\cdot, \cdot}$. This gives that (\ref{intereq2}) is greater than $0$, giving the desired result on the uniqueness of the minimizer. 
\end{proof}
\end{AdditionalTxt}

\begin{proposition}\label{2ndordersystem:convexity:xi}
Let the minimizer of the error functional be denoted
 $$\widehat{\bintkernel}_{L,\infty, \bhypspacexi} :=\argmin{\bvxi \in \bhypspacexi} \mbf{\mathcal{E}}_{\infty}^{\xi}(\bvxi);$$ 
  then for all $\bvxi \in \bhypspacexi$, the difference of the error functional at this element of $ \bhypspacexi$ and the minimizer is lower bounded as, 
 \begin{equation}\label{eq_minH:xi}
 \mbf{\mathcal{E}}_{\infty}^{\xi}(\bvxi)-\mbf{ \mathcal{E}}_{\infty}^{\xi}(\widehat{\bintkernel}_{L,\infty, \bhypspacexi}^{\xi}) \geq c_{\bhypspacexi} \| \bvxi - \widehat{\bintkernel}_{L,\infty, \bhypspacexi}^{\xi} \|_{\LtwoB(\brhoxiL) }^2\,.
 \end{equation}
Thus, the minimizer of $\mbf{\mathcal{E}}_{\infty}^{\xi}$ over $\bhypspacexi$ is unique in $\LtwoB(\brhoxiL)$. 
\end{proposition}

\subsection{Uniform estimates on defect functions}
We start this section by introducing normalized errors of the estimators. 
Denote the minimizer of  $\bmE_{\infty}^{EA}(\cdot)$ over $\bhypspaceEA$ by  
\begin{align}
\widehat\bintkernel_{L,\infty, \bhypspaceEA}^{EA}:=\combf{\widehat\bintkernel_{L,\infty, \bhypspaceEA}^{E}}{\widehat\bintkernel_{L,\infty, \bhypspaceEA}^{A}} =\argmin{\combfbvpEbvpA \in \bhypspaceEA} \bmE_{\infty}^{EA}(\combfbvpEbvpA). 
\end{align} 
For any $ \combfbvpEbvpA \in \bhypspaceEA$, define the \textit{normalized errors} as
\begin{align}
&\mathcal{D}_{\infty}(\combfbvpEbvpA) :=\bmE_{\infty}^{EA}(\combfbvpEbvpA)-\bmE_{\infty}^{EA}(\widehat\bintkernel_{L,\infty, \bhypspaceEA}^{EA} ) \,,\label{2ndorder:difference1}\\
&\mathcal{D}_{M}(\combfbvpEbvpA) :=\bmE_{M}^{EA}(\combfbvpEbvpA)-\bmE_{M}^{EA}(\widehat\bintkernel_{L,\infty, \bhypspaceEA}^{EA} )\,. \label{2ndorder:difference2}
\end{align}
These quantities capture the difference between the expected/empirical errors of the estimator and the function in the hypothesis space minimizing the expected error functional.
We begin by proving a lemma that assumes the distance between the expected and empirical normalized errors are small for a given estimator. We then show that we have similar control on these distances for all points in a neighborhood of this particular one. 
This control enables us to apply a covering argument in the main proposition of this section due to the compactness of the hypothesis space. 
\begin{remark}
Exactly analogous definitions hold for the $\xi$ variable and we will simply state the results in that case. 
\end{remark}

\begin{lemma} \label{2ndorder:covering}
For all $\epsilon>0$ and $0<\alpha<1$,  if the function $\bvpEA_1 \in \bhypspaceEA$ satisfies $$ \frac{\mathcal{D}_{\infty}(\bvpEA_1 )-\mathcal{D}_{M}(\bvpEA_1 )}{\mathcal{D}_{\infty}(\bvpEA_1 )+\epsilon}  < \alpha\,,$$
then for all $ \bvpEA_2  \in \bhypspaceEA$ such that $\| \bvpEA_1  - \bvpEA_2 \|_{\infty}\leq \frac{\alpha \epsilon}{8S_{EA}\max{\{R, R_{\dot{x}}\}}^2K^4}$, where $S_{EA} = \max\{S_E,S_A\}$ we have
 $$  \frac{\mathcal{D}_{\infty}(\bvpEA_2 )-\mathcal{D}_{M}(\bvpEA_2 )}{\mathcal{D}_{\infty}(\bvpEA_2 )+\epsilon} <3 \alpha.  $$
\end{lemma}

\begin{AdditionalTxt}

\begin{proof}
To ease the notation, write $\widehat{\bintkernel}^{EA}:=\widehat\bintkernel_{L,\infty, \bhypspaceEA}^{EA}$, and using definition \eqref{e:defectFunction}, we have that
\begin{align*}
\frac{\mathcal{D}_{\infty}(\bvpEA_2 )-\mathcal{D}_{M}( \bvpEA_2 )}{\mathcal{D}_{\infty}(\bvpEA_2 )+\epsilon} & =\frac{\bmE_{\infty}^{EA}(\bvpEA_2 )-\bmE_{\infty}^{EA}(\widehat{\bintkernel}^{EA}) - ( \bmE_{M}^{EA}( \bvpEA_2 )-\bmE_{M}^{EA}( \widehat{\bintkernel}^{EA} ))}{\mathcal{D}_{\infty}( \bvpEA_2 )+\epsilon} \\ 
& =\frac{ L_M^{EA}(\bvpEA_2 )-L_M^{EA}(\bvpEA_1 )}{\mathcal{D}_{\infty}( \bvpEA_2 )+\epsilon}+\frac{ L_M^{EA}( \bvpEA_1 )-L_M^{EA}(\widehat{\bintkernel}^{EA} )}{\mathcal{D}_{\infty}(\bvpEA_2 )+\epsilon}  \hspace{0.5cm}
 \end{align*} 
 By Lemma \ref{berstein}, we have 
 $$ L_M^{EA}( \bvpEA_2 )-L_M^{EA}(\bvpEA_1 ) \leq  8S_{EA}\max\{R,R_{\dot{x}} \}^2K^4 \|\bvpEA_2 -\bvpEA_1  \|_{\infty} \leq \alpha \epsilon.$$ 
By definition we have that $\mathcal{D}_{\infty}( \bvpEA_2 ) \geq 0$ implying that, 
 $$ \frac{ L_M( \bvpEA_1 )-L_M(\bvpEA_2 ) }{\mathcal{D}_{\infty}( \bvpEA_2 )+\epsilon} \leq \alpha. $$
For the second term, by equation  (\ref{2ndordersystem:empiricalerrorfunctional}) and the assumption that $\alpha<1$, we obtain that
$$ \bm{\mathcal{E}}_{\infty}^{EA}(\bvpEA_1 )-\bm{\mathcal{E}}_{\infty}^{EA}(\bvpEA_2 )  < 4S_{EA}\max\{R, R_{\dot{x}}\}^2K^4\| \bvpEA_1 -  \bvpEA_2 \|_{\infty}  < \epsilon.$$ 
Therefore
$$\mathcal{D}_{\infty}(\bvpEA_1 )-\mathcal{D}_{\infty}( \bvpEA_2 )= \bm{\mathcal{E}}_{\infty}^{EA}(\bvpEA_1)-\bm{\mathcal{E}}_{\infty}^{EA}( \bvpEA_2 )  < \epsilon \leq \epsilon +\mathcal{D}_{\infty}(\bvpEA_2 ),$$ and thus 
$$\frac{\mathcal{D}_{\infty}(\bvpEA_1 )+\epsilon}{\mathcal{D}_{\infty}(\bvpEA_2 )+\epsilon} \leq 2.$$ 
We conclude that 
$$
\frac{ L_M^{EA}(\bvpEA_1 )-L_M^{EA}(\widehat{\bintkernel}^{EA})}{\mathcal{D}_{\infty}( \bvpEA_2 )+\epsilon}= \frac{\mathcal{D}_{\infty}(\bvpEA_1 )-\mathcal{D}_{M}(\bvpEA_1 )}{\mathcal{D}_{\infty}(\bvpEA_1 )+\epsilon} \frac{\mathcal{D}_{\infty}(\bvpEA_1 )+\epsilon }{\mathcal{D}_{\infty}(\bvpEA_2 )+\epsilon} < 2 \alpha,
$$ and the result follows by summing the two estimates. 
\end{proof}

\end{AdditionalTxt}

Arguing in the same way as above, we can derive the lemma below using equation  \ref{2ndordersystem:empiricalerrorfunctionalxi}. We define $\mathcal{D}_{\infty}^{\xi}, \mathcal{D}_M^{\xi}$ similarly to  (\ref{2ndorder:difference1}), (\ref{2ndorder:difference2}) using $\bm{\mathcal{E}}_{\infty}^{\xi},\bm{\mathcal{E}}_{M}^{\xi}$ in the obvious way. 
\begin{lemma} \label{2ndorder:covering:xi}
For all $\epsilon>0$ and $0<\alpha<1$,  if the function $\bpxi_1 \in \bhypspacexi$ satisfies $$ \frac{\mathcal{D}_{\infty}^{\xi}(\bpxi_1  )-\mathcal{D}_{M}^{\xi}(\bpxi_1)}{\mathcal{D}_{\infty}^{\xi}(\bpxi_1)+\epsilon}  < \alpha\,,$$
then for all $ \bpxi_2  \in \bhypspacexi$ such that $\| \bpxi_1   - \bpxi_2 \|_{\infty}\leq \frac{\alpha \epsilon}{8S_{0}R_{\xi}^2K^4}$, we have, for $S_0\geq S_{\xi}$,
 $$  \frac{\mathcal{D}_{\infty}^{\xi}(\bpxi_2)-\mathcal{D}_{M}^{\xi}(\bpxi_2)}{\mathcal{D}_{\infty}^{\xi}(\bpxi_2)+\epsilon} <3 \alpha.  $$
\end{lemma}

\subsection{Concentration}
\begin{proposition} \label{learning:concentrationEA}
For all $\epsilon > 0$, $0<\alpha<1$, $\combfbvpEbvpA \in \bhypspaceEA$, the following concentration bound holds 
$$  
P_{\bmu}\bigg\{ \frac{\mathcal{D}_{\infty}(\combfbvpEbvpA)-\mathcal{D}_{M}(\combfbvpEbvpA)}{\mathcal{D}_{\infty}(\combfbvpEbvpA)+\epsilon}  \geq \alpha \bigg\}  \leq \exp \bigg({\frac{-c_{\bhypspaceEA}\alpha^2  M \epsilon}{32S_{EA}^2 K^4}}\bigg)
$$
\end{proposition}

\begin{AdditionalTxt}
\begin{proof}
Consider the random variable $\Theta$ (with randomness coming from the random initial condition distributed $\bmu$),  and to ease the notation let $\widehat{\bintkernel}^{EA} := \widehat{\bintkernel}_{L,\infty, \bhypspaceEA}^{EA}$, 
$$\Theta = \frac{1}{L} \sum_{l=1}^{L}\big(\bmE_{\bX(t_l)}^{EA}( \combfbvpEbvpA)- \bmE_{\bX(t_l)}^{EA}(\widehat{\bintkernel}^{EA})\big)$$
The coercivity condition given in Definition (\ref{def_coercivity_2ndorder}), Proposition \ref{2ndordersystem:convexity} and \eqref{2ndordersystem:expectationerrorfunctional} allow us to bound the variance, denoted $\sigma^2$, of $\Theta$ as follows. 

\begin{align}\label{2ndorder:variance}
\sigma^2 &\leq \E_{\bmu} \bigg[ \bigg | \frac{1}{L} \sum_{l=1}^{L}\big(\bmE_{\bX(t_l)}^{EA}(\combfbvpEbvpA )- \bmE_{\bX(t_l)}^{EA}(\widehat{\bintkernel}^{EA})\big)\bigg|^2 \bigg] \nonumber \\ &\leq   \frac{1}{L} \sum_{l=1}^{L}  \E_{\bmu} \bigg[ \bigg| \bmE_{\bX(t_l)}^{EA}(\combfbvpEbvpA)- \bmE_{\bX(t_l)}^{EA}(\widehat{\bintkernel}^{EA}) \bigg|^2  \bigg]\nonumber \\ 
&\leq   K^4 \max\{R, R_{\dot{x}} \}^2 \|\combfbvpEbvpA-\widehat{\bintkernel}^{EA}\|_{\LtwoB(\brhoL)}^2 \| \combfbvpEbvpA+\widehat{\bintkernel}^{EA}-2\bpEA\|_{\infty}^2 \nonumber \\ 
&\leq \frac{K^4 \max \{R,R_{\dot{x}} \}^2}{c_{\bhypspaceEA}}(\bmE_{\infty}^{EA}( \combfbvpEbvpA)-\bmE_{\infty}^{EA}(\widehat{\bintkernel}^{EA})) \| \combfbvpEbvpA+\widehat{\bintkernel}^{EA}-2\bpEA\|_{\infty}^2 \nonumber \\ 
&\leq \frac{16S_{EA}^2\max \{R,R_{\dot{x}} \}^2K^4}{c_{\bhypspaceEA}}(\bmE_{\infty}^{EA}( \combfbvpEbvpA)-\bmE_{\infty}^{EA}(\widehat{\bintkernel}^{EA}))\nonumber \\
&\leq   \frac{16S_{EA}^2\max \{R,R_{\dot{x}} \}^2K^4}{c_{\bhypspaceEA}}\mathcal{D}_\infty(\combfbvpEbvpA).
\end{align}

By applying equation (\ref{usefuleq3}) from the proof of Proposition \ref{proposition:EAbound},
we have that $\Theta \leq 8S_{EA}^2\max\{R, R_{\dot{x}} \}^2K^4$ almost surely. 
We then apply the one-sided Bernstein inequality to $\Theta$ and recalling the definitions (\ref{eq:Error_Func_RV:EA}) together with the definitions of the normalized errors in (\ref{2ndorder:difference1}, \ref{2ndorder:difference2}), we get that: 
\begin{equation} %\label{eq:intermediate_Bernstein}
P_{\bmu}\bigg\{ \frac{ \mathcal{D}_{\infty}(\combfbvpEbvpA)- \mathcal{D}_{M}(\combfbvpEbvpA)}{ \mathcal{D}_{\infty}(\combfbvpEbvpA)+\epsilon} \geq \alpha  \bigg\} 
\leq \exp\Bigg(-\frac{\alpha^2(\mathcal{D}_{\infty}(\combfbvpEbvpA)+\epsilon)^2 M }{2\Big(\sigma^2+\frac{8S_{EA}^2\max \{R,R_{\dot{x}} \}^2K^4\alpha(\mathcal{D}_{\infty}(\combfbvpEbvpA)+\epsilon)}{3}\Big)}\Bigg). \nonumber
\end{equation}
Now we provide a lower bound for the exponent to simplify the dependencies. We show that, 
\begin{align*}
&\frac{ \epsilon \cdot c_{\bhypspaceEA} }{32S_{EA}^2\max \{R,R_{\dot{x}} \}^2K^6}  \leq \frac{(\mathcal{D}_{\infty}(\combfbvpEbvpA)+\epsilon)^2  }{2\Big(\sigma^2+\frac{8S_{EA}^2\max \{R,R_{\dot{x}} \}^2K^4\alpha(\mathcal{D}_{\infty}(\combfbvpEbvpA)+\epsilon)}{3}\Big)}\,,
\end{align*}
or equivalently,
\begin{align*} 
&\frac{\epsilon \cdot c_{\bhypspaceEA}}{16S_{EA}^2\max \{R,R_{\dot{x}} \}^2K^6}\bigg(\sigma^2+\frac{8S_{EA}^2\max \{R,R_{\dot{x}} \}^2K^4\alpha(\mathcal{D}_{\infty}(\combfbvpEbvpA)+\epsilon)}{3} \bigg) \\
& \hspace{4cm} \leq (\mathcal{D}_{\infty}(\combfbvpEbvpA)+\epsilon)^2\,. 
\end{align*} 
By the estimate \eqref{2ndorder:variance}, since $0< \alpha \leq 1$, and $0<c_{L, N, \bhypspaceEA}<K^2$ it is sufficient to show that
$$\mathcal{D}_{\infty}(\combfbvpEbvpA)\epsilon+\frac{\epsilon(\mathcal{D}_{\infty}(\combfbvpEbvpA)+\epsilon)}{6} \leq (\mathcal{D}_{\infty}(\combfbvpEbvpA)+\epsilon)^2. $$ This follows from Young's inequality as
$2\mathcal{D}_{\infty}(\combfbvpEbvpA)\epsilon+\epsilon^2 \leq (\mathcal{D}_{\infty}(\combfbvpEbvpA)+\epsilon)^2$, and together these results give the desired bound of the proposition.	
\end{proof}

\end{AdditionalTxt}

We can easily derive the desired supremum bound by a covering argument. The estimation of the covering numbers involved will play a critical role in the main theorems and will be done in a dimension dependent way in order to get optimal minimax rates. 

\begin{proposition} \label{learning:fullconcentrationEA}
In the notation of Proposition \ref{learning:concentrationEA}, 
\begin{align*}
 P_{\bmu}\bigg\{ \sup_{\combfbvpEbvpA \in \bhypspaceEA} \frac{\mathcal{D}_{\infty}(\combfbvpEbvpA)-\mathcal{D}_{M}(\combfbvpEbvpA)}{\mathcal{D}_{\infty}(\combfbvpEbvpA)+\epsilon}  \geq 3\alpha \bigg\} \leq  \mathcal{N}\bigg(\bhypspaceEA, \frac{\alpha\epsilon}{8S_{EA}\max\{R, R_{\dot{x}}\}^2K^4}\bigg) e^{-\frac{c_{\bhypspaceEA}\alpha^2 M \epsilon}{32S_{EA}K^4}}
\end{align*}
where $\mathcal{N}\bigg(\bhypspaceEA, \frac{\alpha\epsilon}{8S_{EA}\max\{R, R_{\dot{x}}\}^2K^4}\bigg)$ denotes the covering number of $\bhypspaceEA$ with radius \\$\frac{\alpha\epsilon}{8S_{EA}\max\{R, R_{\dot{x}}\}^2K^4}$.   
\end{proposition}
\begin{AdditionalTxt}
\begin{proof}
Let $\bvpEA_i = \combf{\bvpE_i}{\bvpA_i} \in \bhypspaceEA$, for $i=1, \ldots, \mathcal{N}\bigg(\bhypspaceEA, \frac{\alpha\epsilon}{8S_{EA}\max\{R, R_{\dot{x}}\}^2K^4}\bigg)$, denote the center of disks $D_i$ of radius $\frac{\alpha\epsilon}{8S_{EA}\max\{R, R_{\dot{x}}\}^2K^4}$ covering $\bhypspaceEA$. The covering number is finite by the compactness assumption on the hypothesis space.  
By Lemma \ref{2ndorder:covering},  
 $$\sup_{\combfbvpEbvpA \in D_i} \frac{\mathcal{D}_{\infty}(\combfbvpEbvpA)-\mathcal{D}_{M}(\combfbvpEbvpA) }{ \mathcal{D}_{\infty}(\combfbvpEbvpA)+\epsilon} \geq 3 \alpha \Rightarrow  \frac{\mathcal{D}_{\infty}(\bvpEA_i)-\mathcal{D}_{M}(\bvpEA_i) }{ \mathcal{D}_{\infty}(\bvpEA_i)+\epsilon}\geq \alpha.$$
Now, by Proposition \ref{learning:concentrationEA}, for each $i$, 
\begin{align*}
P_{\bmu}\bigg\{ \sup_{\combfbvpEbvpA \in D_i} \frac{\mathcal{D}_{\infty}(\combfbvpEbvpA)-\mathcal{D}_{M}(\combfbvpEbvpA) }{ \mathcal{D}_{\infty}(\combfbvpEbvpA)+\epsilon} \geq 3 \alpha  \bigg\} 
&    \leq P_{\bmu}\bigg\{ \frac{\mathcal{D}_{\infty}(\bvpEA_i)-\mathcal{D}_{M}(\bvpEA_i) }{ \mathcal{D}_{\infty}(\bvpEA_i)+\epsilon} \geq \alpha \bigg\}
\\ & \leq e^{-\frac{c_{\bhypspaceEA}\alpha^2 M \epsilon}{32S_{EA}^2\max\{R,R_{\dot{x}}\}^2K^6}}.
\end{align*}
By definition, $\bhypspaceEA\subseteq \bigcup_{i}D_i$, so that
 \begin{align*}
 P_{\bmu}\bigg\{ \sup_{\combfbvpEbvpA \in \bhypspaceEA} \frac{\mathcal{D}_{\infty}(\combfbvpEbvpA)-\mathcal{D}_{M}(\combfbvpEbvpA) }{ \mathcal{D}_{\infty}(\combfbvpEbvpA)+\epsilon} \geq 3 \alpha  \bigg\} \\ & \hspace{-7cm} \leq \sum_{i} P_{\bmu}\bigg\{\sup_{\combfbvpEbvpA \in D_i} \frac{\mathcal{D}_{\infty}(\combfbvpEbvpA)-\mathcal{D}_{M}(\combfbvpEbvpA) }{ \mathcal{D}_{\infty}(\combfbvpEbvpA)+\epsilon} \geq 3 \alpha  \bigg\} 
 \\ & \hspace{-7cm}  \leq  \mathcal{N}\bigg(\bhypspaceEA, \frac{\alpha\epsilon}{8S_{EA}\max\{R, R_{\dot{x}} \}^2K^4}\bigg) e^{-\frac{c_{\bhypspaceEA}\alpha^2 M \epsilon}{32S_{EA}^2\max\{R, R_{\dot{x}} \}^2K^6}}.
 \end{align*}
\end{proof}
\end{AdditionalTxt}

Finally, we state the results for the $\xi$ variable, the proofs are analogous. The advantage of splitting the theorems will become apparent. Specifically, it allows us to control the covering numbers on the $EA$ and $\xi$ hypothesis spacees separately, enabling us to get faster rates than if we viewed the task as estimating all functions simultaneously. This is possible due to the fundamentally decoupled nature of the dynamical system. 

\begin{proposition} \label{learning:concentrationxi}
For all $\epsilon > 0$, $0<\alpha<1$, $\bvxi \in \bhypspacexi$, the following concentration bound holds 
$$  
P_{\bmu}\bigg\{ \frac{\mathcal{D}_{\infty}^{\xi}(\bvxi)-\mathcal{D}_{M}^{\xi}(\bvxi)}{\mathcal{D}_{\infty}^{\xi}(\bvxi)+\epsilon}  \geq \alpha \bigg\}  \leq e^{-\frac{c_{\bhypspacexi}\alpha^2  M \epsilon}{32S_{0}^2 K^4}}\,.
$$
\end{proposition}

\begin{proposition}
In the notation of Proposition \ref{learning:concentrationxi}, 
\begin{align*}
 P_{\bmu}\bigg\{ \sup_{\bvxi \in \bhypspacexi} \frac{\mathcal{D}_{\infty}^{\xi}(\bvxi)-\mathcal{D}_{M}^{\xi}(\bvxi)}{\mathcal{D}_{\infty}^{\xi}(\bvxi)+\epsilon}  \geq 3\alpha \bigg\}  \leq  \mathcal{N}\bigg(\bhypspacexi, \frac{\alpha\epsilon}{8S_{0}R_{\xi}^2K^4}\bigg) e^{-\frac{c_{\bhypspacexi}\alpha^2 M \epsilon}{32S_{0}K^4}}\,, 
\end{align*}
where $\mathcal{N}\bigg(\bhypspacexi, \frac{\alpha\epsilon}{8S_{0}R_{\xi}^2K^4}\bigg)$ denotes the covering number of $\bhypspacexi$ with radius $\frac{\alpha\epsilon}{8S_{0}R_{\xi}^2K^4}$.   
\end{proposition}

\section{Further verification of coercivity condition} \label{sec:app:Coercivity}

In this appendix, we study the coercivity condition for the second-order system of the form:
\begin{equation}
\begin{dcases}
m_i\ddot\bx_i &= \forcev(\bx_i, \dot\bx_i, \xi_i) + \sum_{i'=1}^N \frac{1}{N} (\intkernele(\norm{\bx_{i'} - \bx_i})(\bx_{i'} - \bx_i) + \intkernela(\norm{\bx_{i'} - \bx_i})(\dot\bx_{i'} - \dot\bx_i) \\
\dot\xi_i &= \forcexi(\bx_i, \dot\bx_i, \xi_i) + \sum_{i'=1}^N \frac{1}{N} \intkernelxi(\norm{\bx_{i'} - \bx_i})(\xi_{i'} - \xi_i)
\end{dcases}\,
\label{eq:2ndOrderCoercivity}
\end{equation}
We prove the coercivity condition for the system \eqref{eq:2ndOrderCoercivity}
 in the case of $L=1$.

When the system does not have a $\xi$ variable, the system \eqref{eq:2ndOrderCoercivity}
is related to the anticipation dynamics studied in \cite{shu2019anticipation}. When $\intkernel^{A}\equiv0$ or $\intkernel^{E}\equiv0$, the system is called energy-based or alignment-based respectively, and has found  application in various disciplines including opinion dynamics, particle dynamics, fish-milling dynamics, Cucker-Smale flocking dynamics and phototaxix dynamics. We refer the reader to \cite{lu2019nonparametric, Tang2019, Zhong20} where extensive numerical experiments were conducted to demonstrate the effectiveness of the proposed learning approach on the aforementioned dynamics.  

For conciseness, we only present the proof of coercivity  for learning of $\intkernele$ and $\intkernela$. A similar argument can be conducted to prove the coercivity for learning $\intkernelxi$.  Our aguments also work for special cases when $\intkernel^{A}\equiv0$ or $\intkernel^{E}\equiv0$. Therefore we obtain strict generalization  of the coercivity results in \cite{lu2019nonparametric,Tang2019} which are only for first-order energy-based systems.  We may go further to analyze the coercivity condition for heterogeneous systems. Compared to the homogeneous system, the coercivity condition would impose constraints on the angle between components of  subspaces defined on the directed sum of measures, suggesting that the learning task is more difficult in the case $K\geq 2$.

\begin{theorem}\label{2ndordersingle:coercivity}
Consider the system \eqref{eq:2ndOrderCoercivity} at time $t_1=0$ with the initial distribution $\mu_0^{\bY}=\begin{bmatrix}\mu_0^{\bX}\\ \mu_0^{\dot\bX} \\ \mu_0^{\bXi} \end{bmatrix}$, where ${ \mu}_0^{\bX}$  is exchangeable Gaussian with $\mathrm{cov}(\bx_i(t_1))-\mathrm{cov}(\bx_i(t_1),\bx_j(t_1))=\lambda I_d$\,  for a constant $\lambda>0$,  ${ \mu}_0^{\dot\bX}, { \mu}_0^{\bXi}$  are exchangeable with finite second moment, and they are independent of ${ \mu}_0^{\bX}$.   Then 

\begin{align*}
\mathbb{E}_{\mu_0^{\bY}} \|\rhsfo_{\combf{\varphi^E}{\varphi^A}} (\bX(0), \bV(0))\|^2_{\mathcal{S}} &\geq c_{1,N, \hypspace^{EA}}\vertiii{\combf{\varphi^E}{\varphi^A}}_{L^2(\rho_T^{EA,1})},\\
\mathbb{E}_{\mu_0^{\bY}} \|\rhsfo_{\intkernelvar^{\xi}} (\bX(0), \bXi(0))\|^2_{\mathcal{S}} &\geq c_{1,N, \hypspace^{\xi}}\vertiii{\intkernelvar^{\xi}}_{L^2(\rho_T^{\xi,1})},
\end{align*} 
\begin{itemize}
\item[(1)] where we have  \begin{align*}& \rho_T^{EA,1}(r,\dot r)=   \mathbb{E}_{\mathbf{\mu_0^{\bY}}} \big[\delta_{r_{12}(t_1),\dot{r}_{12}(t_1)}(r,\dot{r})\big]=\mathbb{E}_{\mu_0^{\bX}}\delta_{r_{12}}(t_1)\mathbb{E}_{\mu_0^{\dot\bX}}\delta_{\dot r_{12}}(t_1), \\ &\rho_T^{\xi,1}(r,\xi)=   \mathbb{E}_{\mathbf{\mu_0^{\bY}}} \big[\delta_{r_{12},\xi_{12}}(t_1)\big]
\\ &\vertiii{\combf{\varphi^E}{\varphi^A}}^2_{L^2(\rho_T^{EA,1})} = \|\varphi^{E}(r)r+\varphi^{A}(r)\dot r\|_{L^2(\rho_T^{EA,1}(r,\dot r))}^2, \\ &\vertiii{\intkernelvar^{\xi}}^2_{L^2(\rho_T^{\xi,1})} = \|\varphi^{\xi}(r)\xi\|_{L^2(\rho_T^{EA,1}(r,\xi))}^2,
\end{align*}

\item[(2)] $c_{1,N, \hypspace^{EA}}\geq \frac{N-1}{2N^2}+ \frac{(N-1)(N-2)}{2N^2}c$, $c=\min\bigg\{c_{\mathcal{H}^{EA}}^{E},  c_{\mathcal{H}^{EA}}^{A}c_{\mu_0^{\dot\bX}}\bigg\}$ with $c_{{\bm \mu}_0^{\dot\bX}}=1-\frac{\E \langle \dot \bx_i(0),\dot \bx_{i'}(0)\rangle}{\E \|\dot \bx_i(0)\|^2}$ ($i\neq i'$) and $c_{\mathcal{H}^{EA}}^{E}$ and $c_{\mathcal{H}^{EA}}^{A}$ are  non-negative constants and are positive for compact  $\mathcal{H}^{EA}$ of $L^2(\rho^{EA,1}_T)$ and independent of $N$.

\item[(3)] $c_{1,N, \hypspace^{\xi}}\geq (\frac{N-1}{N^2}+ \frac{(N-1)(N-2)}{N^2}c ),
c=  c_{\mathcal{H}^{\xi}}c_{\mu_0^{\bXi}} $ with $c_{{ \mu}_0^{\bXi}}=1-\frac{\E \langle  \xi_i(0), \xi_{i'}(0)\rangle}{\E \| \xi_i(0)\|^2}$ ($i\neq i'$) and  $c_{\mathcal{H}^{\xi}}$ is  a non-negative constant and is positive for compact  $\mathcal{H}^{\xi}$ of $L^2(\rho^{\xi,1}_T)$ and independent of $N$.

\end{itemize}
 \end{theorem}
 
\begin{proof} 

The proof of part (1) follows from the definition of measures, norms and the properties of the initial distributions. For part (2), we have 
\begin{align}
\mathbb{E}_{\mu_0^{\bY}} \|\rhsfo_{\combf{\varphi^E}{\varphi^A}} (\bX(0), \bV(0))\|^2_{\mathcal{S}}  &=\frac{1}{N^3}\sum_{i=1}^{N} \left( \left(\sum_{j=k=1}^N + \sum_{j\neq k=1}^N\right) C_{i,j,k}^{E}+C_{i,j,k}^{A}+D_{i,j,k}\right)    \nonumber \\
&=\frac{N-1}{N^2}( \|\intkernelvar^{E}(r)  r\|_{L^2(\rho_T^{EA,1})}^2+ \|\intkernelvar^{A}(r) \dot r\|_{L^2(\rho_T^{EA,1})}^2) + \mathcal{R}\nonumber\\&\geq \frac{N-1}{2N^2} \|  \combf{\varphi^E}{\varphi^A}\|^2_{L^2(\rho_T^{EA,1})}+\mathcal{R}
\end{align}where
\begin{align*}
C_{i,j,k}^{E}&= \mathbb{E}_{\mu_0^{\bY}}\intkernelvar^{E}(\|\br_{ji}(0)\|) \intkernelvar^E(\|\br_{ki}(0)\|)\langle \br_{ji}(0), \br_{ki}(0)\big\rangle, \\
C_{i,j,k}^{A}&= \mathbb{E}_{\mu_0^{\bY}}\intkernelvar^{A}(\|\br_{ji}(0)\|) \intkernelvar^A(\|\br_{ki}(0)\|)\langle \dot\br_{ji}(0), \dot\br_{ki}(0)\big\rangle,\\
D_{i,j,k}&= \mathbb{E}_{\mu_0^{\bY}}\big(\intkernelvar^{E}(\|\br_{ji}(0)\|) \intkernelvar^A(\|\br_{ki}(0)\|)\langle \br_{ji}(0), \dot\br_{ki}(0)\big\rangle\\&\quad\quad+\intkernelvar^{A}(\|\br_{ji}(0)\|) \intkernelvar^E(\|\br_{ki}(0)\|)\langle \dot\br_{ji}(0), \br_{ki}(0)\big\rangle\big)=0,\\
 \mathcal{R} &= \frac{1}{N^3} \sum_{i=1}^N\sum_{j\neq k, j\neq i, k\neq i} (C_{ijk}^{A}+C_{ijk}^{E}).
 \end{align*}
 
By the property of $\mu_0^{\bY}$, we have 
 \begin{align*}
 C_{ijk}^{E}&=\mathbb{E}\big[ \intkernelvar^{E}(\|X_1-X_2\|) \intkernelvar^{E}(\|X_1-X_3\|) \left\langle X_2-X_1,  X_3-X_1 \right\rangle\big]\\
C_{ijk}^{A}&=\mathbb{E}\big[ \intkernelvar^{A}(\|X_1-X_2\|) \intkernelvar^{A}(\|X_1-X_3\|)\big]\mathbb{E}\big[\left\langle Y_2-Y_1,  Y_3-Y_1 \right\rangle\big],
\end{align*} for all $(i,j,k)$, where $X_i$s are exchangeable Gaussian random vectors with $\mathrm{cov}(X_1)-\mathrm{cov}(X_1,X_2)=\lambda I_d$ and $Y_i$s are exchangeable random vectors who have the same distribution with the initial velocities of agents $\dot\bx_i$s.  From the Lemma 10 in \cite{Tang2019} and Lemma \ref{coerlemma} below,
\begin{align*} 
C_{ijk}^{E}&\geq c_{\mathcal{H}^{EA}}^{E} \|\combf{\intkernelvar^{E}}{0}\|^2_{L^2(\rho_T^{EA,1})}\\
C_{ijk}^{A}&\geq c_{\mathcal{H}^{EA}}^{A}c_{\mu_0^{\dot\bX}} \|\combf{0}{\intkernelvar^{A}}\|^2_{L^2(\rho_T^{EA,1})}\,,\qquad c_{\mu_0^{\dot\bX}}= (1-\frac{\mathbb{E}\langle Y_1,  Y_2 \rangle}{\E \|Y_1\|^2}), 
\end{align*}
 where the constants $c_{\mathcal{H}^{EA}}^{E}, c_{\mathcal{H}^{EA}}^{A}\geq 0$ are always positive and independent of $N$ for compact $\mathcal{H}^{EA}$, and we used the fact
$$\mathbb{E}\big[\left\langle Y_2-Y_1,  Y_3-Y_1 \right\rangle\big]=\E {\|Y_1\|^2} (1-\frac{\mathbb{E}\langle Y_1,  Y_2 \rangle}{\E \|Y_1\|^2})\geq 0.$$  
Therefore,  
\begin{align*}
&\mathbb{E}_{\mu_0}[ \big\|  \rhsfo_{\intkernelvar}(\bX(0), \dot\bX(0)) \big\|_{\mathcal{S}}^2] \geq   c_{1,N, \hypspace^{EA}}\|\combf{\intkernelvar^{E}}{\intkernelvar^{A}}\|^2_{L^2(\rho_T^{1}(r,\dot r)},\,\\
 &c_{1,N, \hypspace^{EA}}\geq \frac{N-1}{2N^2}+ \frac{(N-1)(N-2)}{2N^2}\min\bigg\{c_{\mathcal{H}^{EA}}^{E},  c_{\mathcal{H}^{EA}}^{A}c_{\mu_0^{\dot\bX}}\bigg\}.
\end{align*} 
For part (3), the proof follows a similar path as for part (2). 
 \end{proof}

 From Theorem \ref{2ndordersingle:coercivity}, we see a particular case when  the coercivity constant $c_{1,N,\mathcal{H}^{EA}}$ is positive uniformly in $N$ if  $c_{\mu_0^{\dot\bX}}>0$. In fact, many distributions  on $\mathbb{R}^{dN}$ with non-i.i.d  $\R^d$ components make the constant $c_{\mu_0^{\dot\bX}}$ positive. 
For example,  the components of $\dot\bX$ are exchangeable Gaussian but not i.i.d, and $d\geq 2$. In this particular case,  coercivity is a property also of the system in the limit as $N\rightarrow\infty$, satisfying the mean-field equations.  As a result, the estimation error of our estimators is independent of $N$.

 The proof of Theorem \ref{2ndordersingle:coercivity} used the following lemma, whose proof is  the same with the proof of Lemma 10 in \cite{Tang2019}. To be self-contained, we listed the statement here.  

\begin{lemma}\label{coerlemma}
Let $X_1, X_2, X_3$ be exchangeable Gaussian random vectors in $\R^d$ with $\mathrm{cov}(X_1)-\mathrm{cov}(X_1,X_2)=\lambda I_d$ for a constant $\lambda>0$.  
\begin{itemize}
\item The marginal distribution of $\rho_{T}^{EA,1}(r,\dot r)$ with respect to $r$, denoted by $\rho(r)$, is a probability measure over $\R^+$ with density function $C_\lambda^{-1} r^{d-1}e^{-\frac{1}{4\lambda}r^2}$ where $C_\lambda =  \frac{1}{2} (4\lambda)^{\frac{d}{2}}\Gamma (\frac{d}{2})$.

\item We have \begin{equation} \label{ineq_gauss}
 \E\left[\intkernelvar(|X_1-X_2|)\intkernelvar(|X_1-X_3|)  \right] \geq c_\mathcal{\mathcal{X}}\|\intkernelvar\|^2_{L^2(\rho)}
 \end{equation}  for all $\intkernelvar \in \mathcal{X} \subset L^2(\rho)$, with $c_{\mathcal{X}}>0$ if $\mathcal{X}$ is compact  and $c_{\mathcal{X}}=0$ if $\mathcal{X}=L^2(\rho)$.

\end{itemize}
  \end{lemma}

\section{Existence, uniqueness and properties of the measures} \label{sec:app:analytical}
In this section, we provide technical details of the analytic properties of the collective system under consideration as well as of the measures that we defined in section \ref{sec:PerfMeasures}. We emphasize that for the analytic portion of the theory,  as we saw with the trajectory prediction result, we view the system \eqref{eq:2ndOrder} as coupled (whereas for the learning theory we leverage that they can be decoupled to make the estimation have better performance).

We begin by showing that under the assumption that the interaction kernels lie in the corresponding admissible spaces, then the system is well-posed.

\subsection{Well-posedness of second-order heterogeneous systems}
\begin{proposition} \label{2ndordersystem:wellposedness}Suppose the kernels $\bintkernel^E=(\intkernel_{kk'}^{E})_{k,k'=1}^{K,K},\bintkernel^A=(\intkernel_{kk'}^A)_{k,k'=1}^{K,K},\bintkernel^{\xi}=(\intkernel_{kk'}^{\xi})_{k,k'=1}^{K,K}$ 
lie in the admissible sets $\mbf{\mathcal{K}}_{S_{E}}^E$,$\mbf{\mathcal{K}}_{S_{A}}^A$, $\mbf{\mathcal{K}}_{S_{\xi}}^{\xi}$ respectively.
Where the admissible spaces are defined in (\ref{eq:E-Admissibility}).
Then the second-order heterogenous system \eqref{eq:2ndOrder} admits a unique global solution in $[0, T]$ for every initial datum $\bX(0),\bXdot(0) \in \mathbb{R}^{dN}$, $\bm{\Xi}(0) \in \mathbb{R}^N$ 
and the solution depends continuously on the initial condition.
\end{proposition} 

The proof of Proposition \ref{2ndordersystem:wellposedness} uses Lemma \ref{lipchitz} and similar techniques used to prove the well-posedness of the first-order homogeneous system (see Section 6 in \cite{BFHM17}) by rewriting the second-order system as a first-order system and then applying standard Caratheodory ODE results.   

\begin{lemma}\label{lipchitz} For any  $\intkernelvar^E \in \mK_{S_E}^E$, $\intkernelvar^A \in \mK_{S_A}^A$, the  function
$$
F[\intkernelvar^{EA}](\bx,\dot{\bx},s^E,s^A):= \varphi^{E}(||\bx||, s^E)\bx + \varphi^{A}(||\bx||,s^A)\dot{\bx}, 
$$
for $\bx, \dot{\bx} \in \real^d$ is Lipschitz continuous on $\mathbb{R}^{2d+p^E+p^A}$ where $p^E,p^A$ are the dimensions of the range of the functions $s^E,s^A$, respectively. Additionally, for any $\intkernelvar^{\xi} \in \mK_{S_{\xi}}^{\xi}$, the function
$$
F[\intkernelvarxi](\bx, s^{\xi}, \xi) := \varphi^{\xi}(||\bx||, s^{\xi})\xi
$$
is Lipschitz continuous on $\mathbb{R}^{d+1 +p^{\xi}}$, where $p^{\xi}$ is the dimension of the range of $s^{\xi}$.
\end{lemma}

\subsection{Properties of measures}
In this section we state and prove some technical properties of the measures described in Section \ref{sec:PerfMeasures}.
\begin{lemma}\label{averagemeasure} 
Suppose each of the interaction kernels lie in the respective admissible spaces, namely, $\bintkernel^E \in \mbf{\mathcal{K}}_{S_E}^E$,$\bintkernel^A \in \mbf{\mathcal{K}}_{S_A}^A$, $\bintkernel^{\xi} \in \mbf{\mathcal{K}}_{S_{\xi}}^{\xi}$. 
Then, for each $(\idxcl,\idxcl')$,  the measures 
$\rho_T^{EA,k,k'},\rho_T^{EA,L,k,k'}$ and $\rho_T^{\xi,k,k'}, \rho_T^{\xi,L,k,k'}$
defined in section \ref{sec:PerfMeasures}, are regular Borel probability measures. 
Furthermore, if $\bmu$ is absolutely continuous with respect to the Lebesgue measure, then for each $(k,k')$ we have that $\rho_T^{EA,k,k'},\rho_T^{EA,L,k,k'},\rho_T^{\xi,k,k'}, \rho_T^{\xi,L,k,k'}$
are absolutely continuous with respect to the Lebesgue measure. 
This implies Borel regularity, and under the absolute continuity of $\bmu$, absolute continuity with respect to Lebesgue measure of the measures, $\brhoEA, \brhoxi, \brhoEAL, \brhoxiL$. 
\end{lemma}

 \begin{proposition}\label{compactmeasure} Suppose the distribution $\bmu$ of the initial condition is compactly supported. Then 
 for each $(\idxcl, \idxcl')$, the support of the measures  $\rho_T^{EA,kk'}, \rho_T^{\xi,kk'}$ (and therefore $\rho_T^{EA,L,kk'}, \rho_T^{\xi,L,kk'}$) is also compact. 
\end{proposition}
\begin{proof}
The compact support of the variables $r,\dot{r}, \xi$ and the feature maps follows by the global well-posedness of the system in finite time, together with the Lipschitz assumptions on the non-collective forces. This compact support over a fixed, finite time is what is claimed in Proposition \ref{compactmeasure}. 
\end{proof}
The main point is that by making reasonable assumptions on the non-collective forces, feature maps, interaction kernels, and the interval of time, together with the assumption that our agents' initial conditions cannot be arbitrarily far apart, we can derive that the pairwise distance, velocity and $\xi$ will be controlled. 
Thus, the measures in section \ref{sec:PerfMeasures}, if given enough trajectories, will be well approximated by the discretized version using the numerical approach described in section \ref{sec:numerics}. Meaning that if we have a reasonable number of trajectories, we can look at the set of pairwise distances, velocities, etc. that these agents explore and bin them to set the support of the interaction kernels. 
Explicit values for the constants claimed in the proposition 
depend on the properties of the non-collective forces, the support and sup-norm of the interaction kernels, the interval $T$, and the number of agents.

\section{Background results} \label{sec:appendixB}
In this section, for the convenience of the reader, we gather a few of the technical tools used in the analysis of the system. These are fundamental results necessary for developing the trajectory prediction, the measure support, and the existence and uniqueness results. We also include some of the necessary results on covering numbers of function spaces used for the learning theory. \\

The first theorem we present is an iterated Gr\"{o}nwall type result that allows us to analyze the trajectory error of the full system $\bY (t)$. 

\begin{theorem}  \label{IteratedGronwall}
Let $u(t), a(t),$ and $b(t)$ be nonnegative continuous functions in $J=[\alpha, \beta],$ and suppose that
$$
\begin{aligned}
u(t) \leq & a(t)+b(t)\left[\int_{\alpha}^{t} k_{1}\left(t, t_{1}\right) u\left(t_{1}\right) d t_{1}+\cdots\right.\\
&\left.+\int_{\alpha}^{t}\left(\int_{\alpha}^{t_{1}} \cdots\left(\int_{\alpha}^{t_{n-1}} k_{n}\left(t, t_{1}, \ldots, t_{n}\right) u\left(t_{n}\right) d t_{n}\right) \cdots\right) d t_{1}\right]
\end{aligned}
$$
for all $t \in J,$ where $k_{i}\left(t, t_{1}, \ldots, t_{i}\right)$ are nonnegative continuous functions in $J_{i+1}, i=$ $1,2, \ldots, n,$ which are nondecreasing in $t \in J$ for all fixed $\left(t_{1}, \ldots, t_{i}\right) \in J_{i}, i=$ $1,2, \ldots, n .$ Then, for all $t \in J$
$$
u(t) \leq a(t)+b(t) \int_{\alpha}^{t} \widehat{R}[a](t, s) \exp \left(\int_{s}^{t} \widehat{R}[b](t, \tau) d \tau\right) d s
$$
where, for all $(t, s) \in J_{2}$
$$
\begin{aligned}
\widehat{R}[w](t, s))=& k_{1}(t, s) w(s)+\int_{\alpha}^{s} k_{2}\left(t, s, t_{2}\right) w\left(t_{2}\right) d t_{2} \\
&+\sum_{i=3}^{n} \int_{\alpha}^{s}\left(\int_{\alpha}^{t_{2}} \cdots\left(\int_{\alpha}^{t_{t-1}} k_{i}\left(t, s, t_{2}, \ldots, t_{i}\right) w\left(t_{i}\right) d t_{i}\right) \cdots\right) d t_{2}
\end{aligned}
$$
for each continuous function $w(t)$ in $J$.
\end{theorem}
\begin{proof}
See \cite{ChoYeol2007}. 
\end{proof}
\section{Additional comments on first-order models and theory}\label{sec:1stOrder}
Our second-order model formulation covers the first-order equations of \cite{Tang2019, lu2019nonparametric} as a special case.  When $\forcev(\bx_i, \dot\bx_i, \xi_i) = -\nu_i\dot\bx_i + \forcex(\bx_i, \xi_i)$ for some constant $\nu_i > 0$, $\forcexi(\bx_i, \dot\bx_i, \xi_i) = \forcexi(\bx_i, \xi_i)$, $\intkernela_{\clof_i \clof_{i'}} \equiv 0$ for all $\idxcl, \idxcl' = 1, \ldots, \numcl$, and $m_i \ll 1$, \eqref{eq:2ndOrder} becomes,
\[
\begin{dcases}
\nu_i\dot\bx_i &= \forcex(\bx_i, \xi_i) + \sum_{i'=1}^N \frac{1}{N_{\clof_{i'}}}\intkernele_{\clof_i \clof_{i'}}(r_{ii'}, \topwde_{i i'})(\bx_{i'} - \bx_i)\\
\dot\xi_i &= \forcexi(\bx_i, \xi_i) + \sum_{i'=1}^N \frac{1}{N_{\clof_{i'}}} \intkernelxi_{\clof_i, \clof_{i'}}(r_{ii'}, \topwdxi_{i i'})(\xi_{i'} - \xi_i)
\end{dcases}
\]
It extends the first-order models considered in \cite{lu2019nonparametric, Tang2019, Zhong20} by adding non-collective forces, $\forcex, \forcexi$, multi-dimensional interaction kernels, $\intkernele_{\idxcl, \idxcl'}, \intkernelxi_{\idxcl, \idxcl'}$, and auxiliary variables, $\xi_i$.

The first-order theory considered in \cite{lu2019nonparametric, Tang2019} was focused on the learnability of functions of the form, $\intkernele(r)r$, which is a special case of our second-order theory, where we study the functions of the form $\intkernele(r)r + \intkernela(r)\dot{r}$, with $\intkernela(r) \equiv 0$.
\section{Additional performance measures}\label{sec:pm_xi}

For measures related to learning the $\xi$-based interaction kernels, we take
\[
\begin{dcases}
\delta_{i, i', t}^{\xi}(r, \topwdxi, \xi) &:= \delta_{r_{i i'}(t), \topwdxi_{i i'}(t), \xi_{i i'}(t)}(r, \topwdxi, \xi) \\
\delta_{i, i', t, m}^{\xi}(r, \topwdxi, \xi) &:= \delta_{r_{i i'}^{(m)}(t), \topwdxim_{i i'}(t), \xi_{i i'}^{(m)}(t)}(r, \topwdxi, \xi)
\end{dcases}
\]
Then, the measures are given by
\begin{equation}\label{meas:rhoxibar}
\rho_T^{\xi, \idxcl, \idxcl'}(r, \topwdxi, \xi) := \mathbb{E}_{\bY_0 \sim \bmu}\frac{1}{TN_{\idxcl \idxcl'}}\int_{t = 0}^T \sum_{\substack{i \in \cl_{\idxcl}, i' \in \cl_{\idxcl'} \\ i \neq i'}}\delta^{\xi}_{i i', t}(r, \topwdxi, \xi) \, dt 
\end{equation}
and similarly for $\rho_T^{\xi, L, \idxcl, \idxcl'}(r, \topwdxi, \xi)$, $\rho_T^{\xi, L, M, \idxcl, \idxcl'}(r, \topwdxi, \xi)$.
Similarly, $\rho_T^{\xi, \idxcl, \idxcl'}$ and its time-discretization version, $\rho_T^{\xi, L, \idxcl, \idxcl'}$, are only used in the theoretical setting, whereas the empirical $\rho_T^{\xi, L, M, \idxcl, \idxcl'}$ is used in the actual algorithm.
We consider direct sums of the measures for the phase variable for ease of notation. 
\begin{equation} \label{meas:vectorized:xi}
\bm{\rho}_{T}^{\xi,L}=\bigoplus_{k, k^{\prime}=1,1}^{K, K} \rho_{T}^{\xi,L, k k^{\prime}}, \quad
\bm{\rho}_{T}^{\xi}=\bigoplus_{k, k^{\prime}=1,1}^{K, K} \rho_{T}^{\xi,k k^{\prime}}, \quad \bm{L}^{2}\left(\bm{\rho}_{T}^{\xi,L}\right)=\bigoplus_{k, k^{\prime}=1,1}^{K, K} L^{2}\left(\rho_{T}^{\xi,L, k k^{\prime}}\right)
\end{equation} 

Lastly, for the $\xi$-based interaction kernels, i.e., $\lintkernelxi_{\idxcl\idxcl'}$ versus $\intkernelxi_{\idxcl\idxcl'}$, we consider the following norm,
\begin{equation}\label{e:L2NormDefs_xi}
\norm{\lintkernelxi_{\idxcl\idxcl'} - \intkernelxi_{\idxcl\idxcl'}}_{L^2(\rho_T^{\xi, \idxcl, \idxcl'})}^2 = \int_r\int_{\xi}\int_{\topwdxi}(\lintkernelxi_{\idxcl\idxcl'}(r, \xi, \topwdxi) - \intkernelxi_{\idxcl\idxcl'}(r, \xi, \topwdxi))^2\xi^2\, d\rho_T^{\xi, \idxcl, \idxcl'}(r, \xi, \topwdxi).
\end{equation}

\section{Numerical algorithm}\label{sec:algorithm_numeric}
In this section, we will detail the construction of the linear systems to learn $\vec{\alpha}^{EA}$ and $\vec{\alpha}^{\xi}$.  

We start from the procedure of solving for $\vec{\alpha}^{EA}$.  First, we build the basis functions for the finite dimensional subspace $\widehat\hypspace^{EA} \subset \hypspace^{EA}$.  
\begin{remark}
The support of the unknown interaction kernels is not assumed to be known.  We build our finite dimensional subspaces, $\widehat\hypspace^{EA}$ for example, based on the empirical observation data.  For the support-detection capability of our estimators, see the examples of opinion dynamics in \cite{lu2019nonparametric, Zhong20}.
\end{remark}
We utilize the tensor grid of basis functions, i.e., tensor product of basis functions in each dimension of the basis $[R_{\idxcl\idxcl'}^{\min, L, M}, R_{\idxcl\idxcl'}^{\max, L, M}] \times \mathbb{S}^{E, L, M}_{\idxcl\idxcl'}$ or $[R_{\idxcl\idxcl'}^{\min, L, M}, R_{\idxcl\idxcl'}^{\max, L, M}] \times \mathbb{S}^{A, L, M}_{\idxcl\idxcl'}$, where $[R_{\idxcl\idxcl'}^{\min, L, M}, R_{\idxcl\idxcl'}^{\max, L, M}]$ is the empirical range of $r$ given by the observation data, similarly for the empirical $\mathbb{S}^{E, L, M}_{\idxcl\idxcl'}$ and $\mathbb{S}^{A, L, M}$ being the range of $\topwde_{\idxcl\idxcl'}$ and $\topwda_{\idxcl\idxcl'}$ given by the observation, respectively.  In each dimension\footnote{Mixture of basis functions in each dimension is possible, the algorithm does not required the basis functions in each dimension to be of the same kind.  We make such assumption for simplicity sake.} of $[R_{\idxcl\idxcl'}^{\min, L, M}, R_{\idxcl\idxcl'}^{\max, L, M}] \times \mathbb{S}^{E, L, M}_{\idxcl\idxcl'}$ or $[R_{\idxcl\idxcl'}^{\min, L, M}, R_{\idxcl\idxcl'}^{\max, L, M}] \times \mathbb{S}^{A, L, M}_{\idxcl\idxcl'}$, the basis functions are built as piece-wise standard polynomials (or other functions, such as Clamped B-splines, Fourier basis, etc.) uniformly with the number of basis functions being $n^{E, j}_{\idxcl\idxcl'}$ or $n^{A, j}_{\idxcl\idxcl'}$.  Hence $n^{E}_{\idxcl\idxcl'} = \prod_{j}^{1 + p^E_{\idxcl, \idxcl'}}n^{E, j}_{\idxcl\idxcl'}$ and $n^{A}_{\idxcl\idxcl'} = \prod_{j}^{1 + p^A_{\idxcl, \idxcl'}} n^{A, j}_{\idxcl\idxcl'}$.  Then, we assemble $\vec{d}^{(m)}$ as follows,
\[
\vec{d}^{EA, (m)} = \begin{bmatrix} \frac{1}{\sqrt{N_{\clof_1}}}\ddot\bx_1(t_1) \\ \vdots \\ \frac{1}{\sqrt{N_{\clof_N}}}\ddot\bx_N(t_1) \\ \vdots \\ \frac{1}{\sqrt{N_{\clof_1}}}\ddot\bx_1(t_L) \\ \vdots \\ \frac{1}{\sqrt{N_{\clof_N}}}\ddot\bx_N(t_L)  \end{bmatrix}.
\]
If $\ddot\bx_i(t_l)$ is not given, a finite difference scheme on $\bx_i(t_l)$ or $\dot\bx_i(t_l)$ is used to approximate $\ddot\bx_i(t_l)$.   Next, we build, $\vec{f}^{(m)}$ as follows,
\[
\vec{f}^{EA, (m)} =\begin{bmatrix}\frac{1}{\sqrt{N_{\clof_1}}}\forcev(\bx_1(t_1), \dot\bx_1(t_1), \xi_1(t_1)) \\ \vdots \\ \frac{1}{\sqrt{N_{\clof_N}}}\forcev(\bx_N(t_1), \dot\bx_N(t_1), \xi_N(t_1)) \\ \vdots \\ \frac{1}{\sqrt{N_{\clof_1}}}\forcev(\bx_1(t_L), \dot\bx_1(t_L), \xi_1(t_L)) \\ \vdots \\ \frac{1}{\sqrt{N_{\clof_N}}}\forcev(\bx_N(t_L), \dot\bx_N(t_L), \xi_N(t_L))  \end{bmatrix}.
\]
Then for the learning matrix, $\Psi^{EA, (m)} \in \R^{LNd \times n}$ with $n = n^E + n^A$.  It is a concatnation of two sub-matrix, $\Psi^{E, (m)}$ and $\Psi^{A, (m)}$, i.e., 
\[
\Psi^{EA, (m)} = \begin{bmatrix}\Psi^{E, (m)} & \Psi^{A, (m)} \end{bmatrix}.
\]
For the energy-based learning matrix, $\Psi^{E, (m)}$, we use a lexicographical order on $(\idxcl, \idxcl')$ for $\idxcl, \idxcl' = 1, \ldots, \numcl$.  We define $n^{E}_{\idxcl, \idxcl', \text{prev}} = \sum_{(\idxcl_1, \idxcl_2) < (\idxcl, \idxcl')} n^{E}_{\idxcl_1, \idxcl_2}$; if $(\idxcl, \idxcl') = (1, 1)$, we take $n^{E}_{1, 1, \text{prev}} = 0$.  Then for $\eta_{\idxcl\idxcl'}^E = 1, \ldots, n^E_{\idxcl,  \idxcl'}$, $\Psi^{E, (m)}$ is given as follows,
\[
\Psi^{E, (m)}(li(1 : d), \eta_{\idxcl\idxcl'}^E) = \sum_{i' \in \cl_{\idxcl'}}\frac{1}{\sqrt{N_{\clof_i}}}\basise_{\idxcl, \idxcl', \eta_{\idxcl\idxcl'}^E}(\norm{\bx_{i'}(t_l) - \bx_i(t_l)}, \topwde_{i, i'}(t_l))(\bx_{i'}(t_l) - \bx_i(t_l)), \quad i \in \cl_{\idxcl},
\]
and for $l = 1, \ldots, L$.  Similar process of construction is done for $\Psi^{A, (m)}$.  Then we define,
\[
A^{EA, (m)} = (\Psi^{EA, (m)})^\top\Psi^{EA, (m)} \quad \text{and} \quad \vec{b}^{EA,(m)} = (\Psi^{EA, (m)})^\top(\vec{d}^{EA, (m)} - \vec{f}^{EA, (m)}).
\]
And lastly,
\[
A^{EA} =\frac{1}{LM}\sum_{m = 1}^M A^{EA,(m)} \quad \text{and} \quad \vec{b}^{EA} = \frac{1}{LM}\sum_{m = 1}^M\vec{b}^{EA,(m)}.
\]
Then, $\vec{\alpha}^{EA} = \begin{bmatrix} (\vec{\alpha}^E)^\top & (\vec{\alpha}^A)^\top \end{bmatrix}^\top$,  is obtained by solving
\[
A^{EA}\vec{\alpha}^{EA} = \vec{b}^{EA}.
\]
Then, we assemble 
\[
\lintkernele_{\idxcl\idxcl'} = \sum_{\eta_{\idxcl\idxcl'}^E = 1}^{n^E_{\idxcl,\idxcl'}} \alpha_{\idxcl,\idxcl', \eta_{\idxcl\idxcl'}^E}^E\basise_{\idxcl, \idxcl', \eta_{\idxcl\idxcl'}^E}.
\]
Similar assembly from $\alpha^A$ is done for $\lintkernela_{\idxcl\idxcl'}$.  In the case of using finite difference approximation to approximate the second derivatives of $\bx_i$, we end up with
\[
A^{EA}\vec{\alpha}^{EA} = \vec{b}^{EA} + \vec{\zeta},
\]
where $\vec{\zeta} = \bigO(\frac{T}{L})$ when a first-order finite difference scheme is used.

Next for $\vec{\alpha}^{\xi}$, we build the basis functions for the finite dimensional subspace $\hypspacexi_{\idxcl\idxcl'} \subset \hypspacexi_{\idxcl\idxcl'}$.  We utilize the tensor grid of basis functions, i.e., tensor product of basis functions in each dimension of the basis $[R_{\idxcl\idxcl'}^{\min, L, M}, R_{\idxcl\idxcl'}^{\max, L, M}] \times \mathbb{S}^{\xi, L, M}_{\idxcl\idxcl'}$, where $[R_{\idxcl\idxcl'}^{\min, L, M}, R_{\idxcl\idxcl'}^{\max, L, M}]$ is the empirical range of $r$ given by the observation data, similarly for the empirical $\mathbb{S}^{\xi, L, M}_{\idxcl\idxcl'}$ being the range of $\topwdxi_{\idxcl\idxcl'}$ given by the observation.  And in each dimension of $[R_{\idxcl\idxcl'}^{\min, L, M}, R_{\idxcl\idxcl'}^{\max, L, M}] \times \mathbb{S}^{\xi, L, M}_{\idxcl\idxcl'}$, the basis functions are built as piece-wise standard polynomials (or other functions, such as Clamped B-splines, Fourier basis, etc.) uniformly with the number of basis functions being $n^{\xi, j}_{\idxcl\idxcl'}$.  Hence $n^{\xi}_{\idxcl\idxcl'} = \prod_{j}^{1 + p^{\xi}_{\idxcl, \idxcl'}}n^{\xi, j}_{\idxcl\idxcl'}$.  We let
\[
\vec{d}^{\xi, (m)} := \begin{bmatrix} \frac{1}{\sqrt{N_{\clof_1}}}\dot\xi_1(t_1) \\ \vdots \\ \frac{1}{\sqrt{N_{\clof_N}}}\dot\xi_N(t_1) \\ \vdots \\ \frac{1}{\sqrt{N_{\clof_1}}}\dot\xi_1(t_L) \\ \vdots \\ \frac{1}{\sqrt{N_{\clof_N}}}\dot\xi_N(t_L)  \end{bmatrix}\,,\quad
\vec{f}^{\xi, (m)} := \begin{bmatrix}\frac{1}{\sqrt{N_{\clof_1}}}\forcexi(\bx_1(t_1), \dot\bx_1(t_1), \xi_1(t_1)) \\ \vdots \\ \frac{1}{\sqrt{N_{\clof_N}}}\forcexi(\bx_N(t_1), \dot\bx_N(t_1), \xi_N(t_1)) \\ \vdots \\ \frac{1}{\sqrt{N_{\clof_1}}}\forcexi(\bx_1(t_L), \dot\bx_1(t_L), \xi_1(t_L)) \\ \vdots \\ \frac{1}{\sqrt{N_{\clof_N}}}\forcexi(\bx_N(t_L), \dot\bx_N(t_L), \xi_N(t_L))  \end{bmatrix},
\]
and
\[
\Psi^{\xi, (m)}(li(1 : d), \eta_{\idxcl\idxcl'}^{\xi}) = \sum_{i' \in \cl_{\idxcl'}}\frac{1}{\sqrt{N_{\clof_i}}}\basisxi_{\idxcl, \idxcl', \eta_{\idxcl\idxcl'}^{\xi}}(\norm{\bx_{i'}(t_l) - \bx_i(t_l)}, \topwdxi_{i, i'}(t_l))(\xi_{i'}(t_l) - \xi_i(t_l)), \quad i \in \cl_{\idxcl},
\]
Finally we define,
\[
A^{\xi, (m)} = (\Psi^{\xi, (m)})^\top\Psi^{\xi, (m)} \quad \text{and} \quad \vec{b}^{\xi,(m)} = (\Psi^{\xi, (m)})^\top(\vec{d}^{\xi, (m)} - \vec{f}^{\xi, (m)}).
\]
and 
\[
A^{\xi} =\frac{1}{LM}\sum_{m = 1}^M A^{\xi,(m)} \quad \text{and} \quad \vec{b}^{\xi} = \frac{1}{LM}\sum_{m = 1}^M\vec{b}^{\xi, (m)}.
\]
Thus, $\vec{\alpha}^{\xi}$ is obtained by solving
\[
A^{\xi}\vec{\alpha}^{\xi} =\vec{b}^{\xi}.
\]
Then, we assemble 
\[
\lintkernelxi_{\idxcl\idxcl'} = \sum_{\eta_{\idxcl\idxcl'}^{\xi} = 1}^{n^{\xi}_{\idxcl,\idxcl'}} \alpha_{\idxcl,\idxcl', \eta_{\idxcl\idxcl'}^{\xi}}^{\xi}\basisxi_{\idxcl, \idxcl', \eta_{\idxcl\idxcl'}^{\xi}}.
\]
\newpage
\bibliography{learning_dynamics}
\end{document}